\documentclass{article}

    \PassOptionsToPackage{numbers, compress}{natbib}

\usepackage[final]{neurips_2022}

\usepackage[utf8]{inputenc} 
\usepackage[T1]{fontenc}    
\usepackage{url}            
\usepackage{booktabs}       
\usepackage{amsfonts}       
\usepackage{nicefrac}       
\usepackage{microtype}      
\usepackage{titletoc}
\usepackage{tocvsec2}
\usepackage[title,titletoc]{appendix}
\usepackage{bm}
\usepackage{amsmath}
\usepackage{amsthm}
\usepackage{amssymb}
\usepackage{algorithm}
\usepackage{algorithmic}
\usepackage{multirow}
\usepackage{natbib}
\usepackage{graphicx}
\usepackage{boxedminipage}
\usepackage{caption}
\usepackage{color}
\usepackage{soul}
\usepackage{makecell}
\definecolor{orgin}{RGB}{254, 138, 113}
\definecolor{second}{RGB}{99, 172, 229}
\usepackage[dvipsnames]{xcolor}
\usepackage[breakable,skins]{tcolorbox}
\usepackage[colorlinks,linkcolor=darkblue,citecolor=org,urlcolor=org, linktoc=all]{hyperref}
\newenvironment{qbox}
{\begin{tcolorbox}[enhanced jigsaw, drop shadow=black!50!white,colback=white, width=0.95\linewidth, center, left=2pt,right=2pt,top=1pt,bottom=1pt]}
{\end{tcolorbox}}
\definecolor{mod}{HTML}{b66afc}

\definecolor{ballblue}{HTML}{338EA7}
\definecolor{darkblue}{HTML}{183D5E}
\definecolor{lightseagreen}{HTML}{759D39}
\definecolor{lightred}{HTML}{DD7769}
\definecolor{org}{HTML}{F8A145}
\definecolor{blu}{HTML}{63ACE5}
\definecolor{c1}{HTML}{41B3A3}
\definecolor{c2}{HTML}{3500D3}
\definecolor{mark}{RGB}{153, 204, 255}
\newtheorem{theorem}{Theorem}
\newtheorem{lemma}{Lemma}
\newtheorem{definition}{Definition}

\newtheorem{corollary}{Corollary}
\newtheorem{proposition}{Proposition}
\newtheorem{assumption}{Assumption}
\newtheorem{remark}{Remark}

\usepackage{enumitem}

\newcommand{\ie}{\textit{i}.\textit{e}.}

\newcommand{\stt}{\textit{s}.\textit{t}.}
\newcommand{\wrt}{\textit{w}.\textit{r}.\textit{t}. }
\newcommand{\tabincell}[2]{\begin{tabular}{@{}#1@{}}#2\end{tabular}}
\newtcolorbox{modbox}{
        colframe=cyan!5!white,
        colback =cyan!5!white,
        top=0mm, bottom=0mm, left=1mm, right=1mm,
        arc=1mm,
        fontupper=\color{blue!70!black},
        fonttitle=\bfseries\color{blue!70!black},
        notitle,
        }
\title{Asymptotically Unbiased Instance-wise Regularized Partial AUC Optimization: Theory and Algorithm}
\def \np {n_+}
\def \nn {n_-}
\def \npa {n_+^\alpha}
\def \nnb {n_-^\beta}
\def \tp {\mathrm{TPAUC}}
\def \op {\mathrm{OPAUC}}
\def \ehat {\hat{\mathbb{E}}}
\def \cmin {f, (a,b) \in [0,1]^2}
\def \cmins {(a,b) \in [0,1]^2}
\def \spaces {\qquad \qquad \qquad \qquad \qquad \qquad}
\def \cminl {f, (a,b) \in [0,1]^2, s' \in \Omega_{s'}}
\def \cmax {\gamma \in \Omega_\gamma}
\def \efb {\eta_\beta(f)}
\def \efa {\eta_\alpha(f)}
\def \hefb {\hat{\eta}_\beta(f)}
\def \hefa {\hat{\eta}_\alpha(f)}
\def \exdp {\underset{\bm{x}\sim \mathcal{D}_{\mathcal{P}}}{\mathbb{E}}}
\def \exdn {\underset{\bm{x}'\sim \mathcal{D}_{\mathcal{N}}}{\mathbb{E}}}
\newcommand{\colorg}[1]{{\color{org}#1}}
\newcommand{\colsec}[1]{{\color{second}#1}}
\newcommand{\colblue}[1]{{\color{blue}#1}}
\newcommand{\colbit}[1]{{\color{Bittersweet}#1}}

\author{\parbox{13cm}
  {\centering
    {\large \quad\quad Huiyang Shao$^{1,2}$ \ \ \ \ \ \ \ \ \  Qianqian Xu$^{1}$\thanks{Corresponding authors.} \ \ \ \ \ \ \ \ \ 
    Zhiyong Yang$^{2}$ \ \ \ \ \ \ \ \ \ \ \  \\ Shilong Bao$^{3,4}$ \ \ \ \ \ \ \ \ \ \ \ \ \ Qingming Huang$^{1,2,5,6*}$ }\\
    {\normalsize \normalfont
    $^1$ Key Lab of Intell. Info. Process., Inst. of Comput. Tech., CAS\\
    $^2$ School of Computer Science and Tech., University of Chinese Academy of Sciences\\
    $^3$ State Key Lab of Info. Security, Inst. of Info. Engineering, CAS\\
    $^4$ School of Cyber Security, University of Chinese Academy of Sciences\\
    $^5$ BDKM, University of Chinese Academy of Sciences\\
    $^6$ Peng Cheng Laboratory\\
    }
    {\tt\small shaohuiyang21@mails.ucas.ac.cn xuqianqian@ict.ac.cn \quad\quad \\ yangzhiyong21@ucas.ac.cn baoshilong@iie.ac.cn qmhuang@ucas.ac.cn
    }
  }
}

\begin{document}

\maketitle
\begin{abstract}
    The Partial Area Under the ROC Curve (PAUC), typically including One-way Partial AUC (OPAUC) and Two-way Partial AUC (TPAUC), measures the average performance of a binary classifier within a specific false positive rate and/or true positive rate interval, which is a widely adopted measure when decision constraints must be considered. Consequently, PAUC optimization has naturally attracted increasing attention in the machine learning community within the last few years. Nonetheless, most of the existing methods could only optimize PAUC approximately, leading to inevitable biases that are not controllable. Fortunately, a recent work presents an unbiased formulation of the PAUC optimization problem via distributional robust optimization. However, it is based on the pair-wise formulation of AUC, which suffers from the limited scalability w.r.t. sample size and a slow convergence rate, especially for TPAUC. To address this issue, we present a simpler reformulation of the problem in an asymptotically unbiased and instance-wise manner. For both OPAUC and TPAUC, we come to a nonconvex strongly concave minimax regularized problem of instance-wise functions. On top of this, we employ an efficient solver enjoys a linear per-iteration computational complexity w.r.t. the sample size  and a time-complexity of $O(\epsilon^{-1/3})$ to reach a $\epsilon$ stationary point. Furthermore, we find that the minimax reformulation also facilitates the theoretical analysis of generalization error as a byproduct. Compared with the existing results, we present new error bounds that are much easier to prove and could deal with hypotheses with real-valued outputs. Finally, extensive experiments on several benchmark datasets demonstrate the effectiveness of our method.
\end{abstract}

\section{Introduction}
AUC refers to the Area Under the Receiver Operating Characteristic (ROC) curve \cite{J1982The}, where the ROC curve is obtained by plotting the True Positive Rate (TPR) against the False Positive Rate (FPR) of a given classifier for all possible thresholds. Since it is insensitive to the class distribution, AUC has become one of the standard metrics for long-tail, and imbalanced datasets \cite{J1982The, mcclish1989analyzing, yang2022auc}. Consequently, AUC optimization has attracted increasing attention in the machine learning community ever since the early 2000s \cite{graepel2000large, cortes2003auc, yan2003optimizing, joachims2005support}. Over the last two decades, research on AUC optimization has evolved from the simplest linear models and decision trees \cite{pepe2000combining, freund2003efficient, rakotomamonjy2004support, ying2016stochastic} to state-of-the-art deep learning architectures \cite{liu2019stochastic, guo2020communication, yang2021deep, yuan2021large, yuan2021compositional, wang2022momentum}. With such remarkable success, one can now easily apply AUC optimization to deal with various real-world problems ranging from financial fraud detection\cite{huang2022auc, choi2017machine, mubalaike2018deep}, spam email detection \cite{narasimhan2013structural}, to medical diagnosis \cite{narasimhan2013structural, yang2022auc, yang2021deep, yuan2021large}, etc.

\begin{figure}[h]
    \centering
    \includegraphics[scale=0.8]{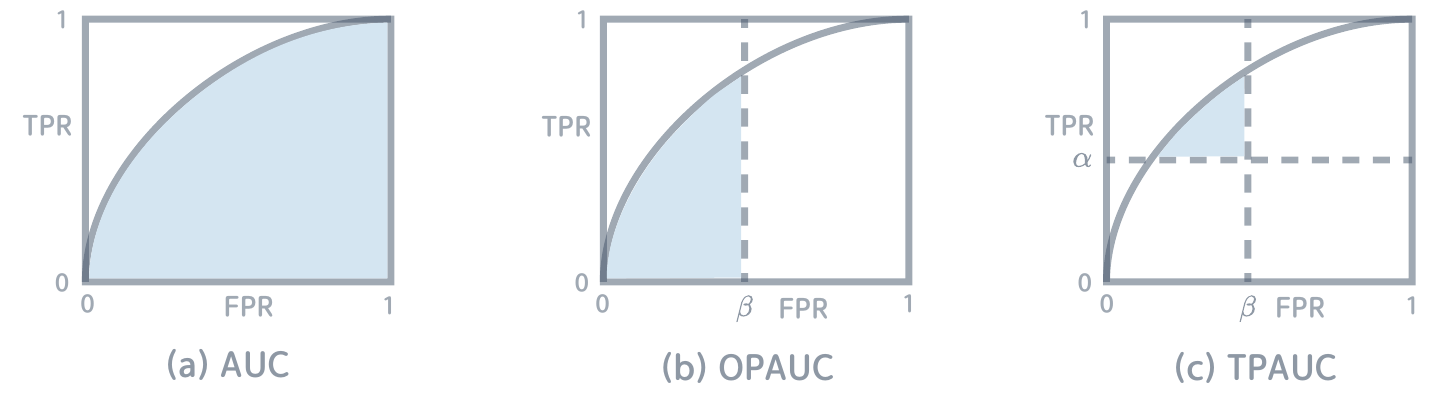}
    \caption{The comparison among different AUC variants. (a) The area under the ROC curve (AUC). (b) The one-way partial AUC (OPAUC). (c) The two-way partial AUC (TPAUC).}
    \label{fig1}
\end{figure}

However, in such long-tailed applications, we are often interested in a specific region in the ROC curve where its area is called Partial AUC (PAUC). As illustrated in Fig.\ref{fig1}, there are two types of PAUC. Here, One-way Partial AUC (OPAUC) measures the area within an FPR interval ( $0 \le \mathrm{FPR} \le \beta$); while Two-way Partial AUC (TPAUC) measures the area with $ \mathrm{FPR} \le \beta$, $\mathrm{TPR} \ge \alpha$. Unlike the full AUC, optimizing PAUC requires selecting top-ranked or/and bottom-ranked instances, leading to a hard combinatorial optimization problem. Many efforts have been made to solve the problem \cite{2003Partial, narasimhan2013structural, narasimhan2017support, kumar2021implicit, yang2021all}. However, the majority of them rely on full-batch optimization and the approximation of the top (bottom)-$k$ ranking process, which suffers from immeasurable biases and undesirable efficiency. Most recently, researchers have started to explore mini-batch PAUC optimization for deep learning models. \cite{yang2021all} proposed a novel end-to-end optimization framework for PAUC. This formulation has a fast convergence rate with the help of a stochastic optimization algorithm, but the estimation of PAUC is still biased. Later, \cite{zhu2022auc} proposed a Distributional Robust Optimization (DRO) framework for PAUC optimization, where the bias can be eliminated by a clever reformulation and the compositional SGD algorithms \cite{qi2021stochastic}. However, they adopt the pair-wise loss function which has limited scalability w.r.t. sample size and $O(\epsilon^{-4})$/$O(\epsilon^{-6})$ time complexity to reach the $\epsilon$-stationary point for $\mathrm{OPAUC}$/$\mathrm{TPAUC}$. Considering the efficiency bottleneck comes from the pair-wise formulation, we will explore the following question in this paper:

\begin{qbox}
  \emph{\textbf{Can we design a simpler, nearly asymptotically unbiased and instance-wise formulation to optimize OPAUC and TPAUC in an efficient way?}}
\end{qbox}

To answer this question, we propose an efficient and nearly unbiased optimization algorithm (the bias vanishes asymptotically when $\kappa \rightarrow 0$) for regularized PAUC maximization with a faster convergence guarantee. The comparison with previous results are listed in Tab.\ref{tab1}. We consider both $\mathrm{OPAUC}$ and $\mathrm{TPAUC}$ maximization, where for $\mathrm{OPAUC}$, we focus on maximizing PAUC in the region: $\mathrm{FPR} \le \beta$ and for $\mathrm{TPAUC}$ we focus on the region: $\mathrm{FPR} \le \beta$ and $\mathrm{TPR} \ge \alpha$. We summarize our contributions below.

\begin{table}[tbp]
    \small
    \centering  
    \caption{
    Comparison with existing partial AUC algorithms. The convergence rate represents the number of iterations after which an algorithm can find an $\epsilon$-stationary point, where $\epsilon$-sp is $\epsilon$-stationary point.  $\triangle$ implies a natural result of non-convex SGD. $n_+^B$ ($n_-^B$ resp.) is the number of positive (negative resp.) instances for each mini-batch $B$. 
    }
    \begin{tabular}{c|ccccc}
        \toprule
        ~ & SOPA \cite{zhu2022auc}&SOPA-S \cite{zhu2022auc}&TPAUC \cite{yang2021all}&Ours \\  
        \midrule
        Convergence Rate ($\mathrm{OPAUC}$)&$O(\epsilon^{-4})$&$O(\epsilon^{-4})$&$O(\epsilon^{-4})^{\triangle}$&$O(\epsilon^{-3})$ \\
        Convergence Rate ($\mathrm{TPAUC}$)&$O(\epsilon^{-6})$&$O(\epsilon^{-4})$&$O(\epsilon^{-4})^{\triangle}$&$O(\epsilon^{-3})$ \\
        Convergence Measure & $\epsilon$-sp (non-smooth) & $\epsilon$-sp & $\epsilon$-sp & $\epsilon$-sp\\
        Smoothness & non-smooth & smooth & smooth & smooth\\
        Unbiasedness&$\surd$&$\times$&$\times$& \makecell[c]{with bias $O(1/\kappa)$ \\ when $\omega = 0 $}  \\
        Per-Iteration Time Complexity&$O(n_+^B n_-^B)$&$O(n_+^B n_-^B)$&$O(n_+^B n_-^B)$&$O(n_+^B + n_-^B)$ \\
        \bottomrule
    \end{tabular}
    \label{tab1}  
\end{table}

\begin{itemize}
    \item With a proper regularization, we propose a nonconvex strongly concave minimax instance-wise formulation for $\mathrm{OPAUC}$ and $\mathrm{TPAUC}$ maximization. On top of our proposed formulation,  we employ an efficient stochastic minimax algorithm that finds a $\epsilon$-stationary point within $O(\epsilon^{-3})$ iterations.
    \item We conduct a generalization analysis of our proposed methods. Our instance-wise reformulation can overcome the interdependent issue of the original pair-wise generalization analysis. The proof is much easier than the existing results. Moreover, compared with \cite{narasimhan2017support, yang2021all}, it can be applied to any real-valued hypothesis functions other than the hard-threshold functions.
    \item We conduct extensive experiments on multiple imbalanced image classification tasks. The experimental results speak to the effectiveness of our proposed methods.
\end{itemize}

\section{Preliminaries}
\textbf{Notations.} In this section, we give some definitions and preliminaries about $\mathrm{OPAUC}$ and
$\mathrm{TPAUC}$. Let $\mathcal{X} \subseteq \mathbb{R}^d$ be the input space, 
$\mathcal{Y}=\{0, 1\}$ be the label space. We denote $\mathcal{D}_\mathcal{P}$ and $\mathcal{D}_\mathcal{N}$ as positive and negative instance distribution, respectively. Let $S=\{\bm{z}=(\bm{x}_i,y_i)\}_{i=1}^n$ be a set of training data drawn from distribution $\mathcal{D}_\mathcal{Z}$, where $n$ is the number of samples. Let $\mathcal{P}$ ($\mathcal{N}$ resp.) be a set of positive (negative resp.)
instances in the dataset. \ul{\textbf{In this paper we only focus on the scoring functions $f:\mathcal{X}\mapsto [0,1]$}}. The output range can be simply implemented by any deep neural network with sigmoid outputs. 

\textbf{Standard AUC.} The standard AUC calculates the entire area under the ROC curve. For mathematical convenience, our arguments start with the pair-wise reformulation of AUC. Specifically, as shown in \cite{J1982The}, $\mathrm{AUC}$ measures the probability of a positive instance having a higher score than a negative instance:
\begin{equation}
    \mathrm{AUC}(f) = \underset{\bm{x}\sim \mathcal{D}_{\mathcal{P}},\bm{x}'\sim \mathcal{D}_{\mathcal{N}}}{\Pr}\left[f(\bm{x})> f(\bm{x}')\right].\\
\end{equation}
\textbf{$\mathrm{\mathbf{OPAUC}}$. } As mentioned in the introduction, instead of considering the entire region of ROC, we focus on two forms of PAUC, namely TPAUC and OPAUC. According to \cite{2003Partial}, $\mathrm{OPAUC}$ is equivalent to the probability of a positive instance $\bm{x}$ being scored higher than a negative instance $\bm{x}'$ within the specific range $f(\bm{x}')\in[\eta_\beta(f), 1]$ \stt $\underset{\bm{x}'\sim\mathcal{\mathcal{D}_\mathcal{N}}}{\Pr}[f(\bm{x}')\geq\eta_{\beta}]=\beta$:
\begin{equation}
\begin{aligned}
    \mathrm{OPAUC}(f) &= \underset{\bm{x}\sim \mathcal{D}_{\mathcal{P}},\bm{x}'\sim \mathcal{D}_{\mathcal{N}}}{\Pr}\left[f(\bm{x})> f(\bm{x}'),  f(\bm{x}')\geq  \eta_\beta(f) \right].
\end{aligned}
\label{OPAUC}
\end{equation}
Practically, we do not know the exact data distributions $\mathcal{D}_{\mathcal{P}}$, $\mathcal{D}_{\mathcal{N}}$ to calculate Eq.\eqref{OPAUC}. Therefore, we turn to the empirical estimation of Eq.\eqref{OPAUC}. Given a finite dataset $S$ with $n$ instances, let $n_+$, $n_-$ be the numbers of positive/negative instances, respectively. For the $\mathrm{OPAUC}$, its empirical estimation could be expressed as \cite{narasimhan2013structural}:
\begin{equation}
    \begin{aligned}
        \hat{\mathrm{AUC}}_{\beta}(f, S) = 1-\sum_{i=1}^{n_+} \sum_{j=1}^{n_-^\beta}\frac{\ell_{0,1}{\left(f(\bm{x}_i)- f(\bm{x}_{[j]}')\right)}}{n_+ n_-^\beta},
        \label{OPAUCM}
    \end{aligned}
\end{equation}
where $n_-^\beta=\lfloor n_-\cdot \beta \rfloor$; $\bm{x}'_{[j]}$ denotes the $j$-th largest score among negative samples; $\ell_{0,1}(t)$ is the $0-1$ loss, which returns $1$ if $t <0$ and $0$ otherwise.

\textbf{$\mathrm{\mathbf{TPAUC}}$.}  More recently, \cite{yang2019two} argued that an efficient classifier should have low $\mathrm{FPR}$ and high $\mathrm{TPR}$ simultaneously. Therefore, we also study a more general variant called Two-way Partial AUC ($\mathrm{TPAUC}$), where the restricted regions satisfy $\mathrm{TPR}\geq \alpha$ and $\mathrm{FPR}\leq \beta$. Similar to $\mathrm{OPAUC}$, $\mathrm{TPAUC}$ measures the probability that a positive instance $\bm{x}$ ranks higher than a negative instance $\bm{x}'$ where $f(\bm{x})\in[0, \eta_{\alpha}(f)]$, \stt $\underset{\bm{x}\sim\mathcal{D_\mathcal{P}}}{\Pr}[f(\bm{x})\leq\eta_{\alpha}]=\alpha, f(\bm{x}')\in[\eta_\beta(f), 1]$ \stt $\underset{\bm{x}'\sim\mathcal{D_\mathcal{N}}}{\Pr}[f(\bm{x}')\geq\eta_{\beta}]=\beta$.
\begin{equation}
\begin{aligned}
    \mathrm{TPAUC}(f) &= \underset{\bm{x}\sim \mathcal{D}_{\mathcal{P}},\bm{x}'\sim \mathcal{D}_{\mathcal{N}}}{\Pr}\left[f(\bm{x})> f(\bm{x}'), f(\bm{x})\leq \eta_{\alpha}(f), f(\bm{x}')\geq  \eta_\beta(f) \right].
    \label{TPAUC}
\end{aligned}
\end{equation}
Similarly to $\op$, for the $\tp$, we adopt its empirical estimation \cite{yang2019two, yang2021all}:
\begin{equation}
    \begin{aligned}
        \hat{\mathrm{AUC}}_{\alpha, \beta}(f, S) = 1- \sum_{i=1}^{n_+^\alpha} \sum_{j=1}^{n_-^\beta}\frac{\ell_{0,1}{\left(f(\bm{x}_{[i]})- f(\bm{x}'_{[j]})\right)}}{n_+^\alpha n_-^\beta},
        \label{TPAUCM}
    \end{aligned}
\end{equation}
where $n_+^\alpha=\lfloor n_+\cdot \alpha \rfloor$ and $\bm{x}_{[i]}$ is $i$-th smallest score among all positive instances.  

\section{Problem Formulation}\label{sec:method}
In this section, we introduce how to optimize $\mathrm{OPAUC}$ and $\mathrm{TPAUC}$ in an asymptotically unbiased instance-wise manner.  
Note that Eq.\eqref{OPAUCM} and Eq.\eqref{TPAUCM} are hard to optimize since it is complicated to determine the positive quantile function $\eta_{\alpha}(f)$ and the negative quantile function $\eta_\beta(f)$. So we can not obtain the bottom-ranked positive instances and top-ranked negative instances directly. In this section, we will elaborate on how to tackle these challenges. 

\subsection{Optimizing the OPAUC}\label{sec:opauc}
According to Eq.\eqref{OPAUCM}, given a surrogate loss $\ell$ and the finite dataset $S$, maximizing $\mathrm{OPAUC}$ and $\hat{\mathrm{AUC}}_{\beta}(f, S)$ is equivalent to solving the following problems, respectively:
\begin{equation}
        \underset{f}{\min}~  \mathcal{R}_{\beta}(f)= \mathbb{E}_{\bm{x} \sim \mathcal{D}_\mathcal{P}, \bm{x}'\sim \mathcal{D}_\mathcal{N}} \left[\mathbb{I}_{f(\bm{x}') \ge \eta_\beta(f)} \cdot \ell(f(\bm{x})- f(\bm{x}'))\right],
        \label{OPAUCO}
    \end{equation}
\begin{equation}
\begin{aligned}
    \underset{f}{\min} \ \hat{\mathcal{R}}_{\beta}(f, S)= \sum_{i=1}^{n_+} \sum_{j=1}^{n_-^\beta}\frac{\ell{\left(f(\bm{x}_i)- f(\bm{x}'_{[j]})\right)}}{n_+ n_-^\beta}.
    \label{OPAUCO}
\end{aligned}
\end{equation}
\underline{\textbf{{\color{blue}{Step 1: Instance-wise Reformulation.}}}} Here, to simplify the reformulation, we will use the most popular \textbf{surrogate squared loss} $\ell(x) =(1-x)^2$. Under this setting,  the following theorem shows an instance-wise reformulation of the $\op$ optimization problem (please see  Appendix.\ref{section:proofs_section4} for the proof):
\begin{theorem}
\label{theorem:1}
Assuming that $f(\bm{x})\in[0,1]$, $\forall \bm{x}\in\mathcal{X}$, $F_{op}(f,a,b,\gamma, t, \bm{z})$ is defined as:
\begin{equation}
    \begin{aligned}
F_{op}(f,a,b,\gamma,t, \bm{z})=&\Large[(f(\bm{x})-a)^2- 
2(1+\gamma)f(\bm{x})\Large]y/p-\gamma^2\\
&\Large[(f(\bm{x})-b)^2+2(1+\gamma)f(\bm{x})\Large]\cdot[(1-y)\mathbb{I}_{f(\bm{x})\geq t}]/[(1-p)\beta],
\end{aligned}
\label{eq:fop}
\end{equation}
where $y=1$ for positive instances, $y=0$ for negative instances and we have the following conclusions:
\begin{enumerate}[leftmargin=20pt]
    \item[(a)] (\textbf{Population Version}.) We have:
    \begin{equation}
        \underset{f}{\min} \ {\mathcal{R}}_{\beta}(f) \Leftrightarrow \underset{\cmin}{\min}\ \underset{\gamma\in[-1,1]}{\max} \ 
    \colblue{\underset{\bm{z}\sim \mathcal{D}_\mathcal{Z}}{\mathbb{E}}}
    \left[F_{op}(f,a,b,\gamma, \colblue{\eta_\beta(f)},\bm{z})\right],
    \label{eq:minmaxopauc1}
    \end{equation}
    where $\colblue{\eta_\beta(f)}=\arg\min_{\colblue{\eta_{\beta}} \in\mathbb{R}}\left[\mathbb{E}_{\bm{x}'\sim \mathcal{D}_{\mathcal{N}}}[\mathbb{I}_{f(\bm{x}')\geq\ \colblue{\eta_{\beta}}}]=\beta\right]$.
    \item[(b)] (\textbf{Empirical Version}.) Moreover, given a training dataset $S$ with sample size $n$, denote:
    \begin{equation*}
    \colbit{\underset{\bm{z}\sim S}{\ehat}}[F_{op}(f,a,b,\gamma,\colbit{\hat{\eta}_\beta(f)}, \bm{z})] = \frac{1}{n}\sum_{i=1}^n F_{op}(f,a,b,\gamma,{\colbit{\hat{\eta}_\beta(f)}}, \bm{z}_i),
    \end{equation*}
    where $\colbit{\hat{\eta}_\beta(f)}$ is the empirical quantile of the negative instances in $S$. We have:
    \begin{equation}
        \underset{f}{\min} \ \hat{\mathcal{R}}_{\beta}(f, S) \Leftrightarrow \underset{\cmin}{\min}\ \underset{\gamma\in[-1,1]}{\max} \ 
    \colbit{\underset{\bm{z}\sim S}{\ehat}}
    \left[F_{op}(f,a,b,\gamma, \colbit{\hat{\eta}_\beta(f)}, \bm{z} ) \right],
    \label{eq:minmaxopauc2}
    \end{equation} 
\end{enumerate}
\end{theorem}

\underline{\textbf{\color{blue}{Step 2: Differentiable Sample Selection.}}} Thm.\ref{theorem:1} provides a support to convert the pair-wise loss into instance-wise loss for $\mathrm{OPAUC}$. However, the minimax problem Eq.\eqref{eq:minmaxopauc2} is still difficult to solve due to the operation $\mathbb{I}_{f(\bm{x}')\geq\eta_{\beta}(f)}$, which requires selecting top-ranked negative instances. To make the sample selection process differentiable, we adopt the following lemma.
\begin{lemma}
\label{lemma:1}
$\sum_{i=1}^k x_{[i]}$ is a convex function of $(x_1,\cdots,x_n)$ where $x_{[i]}$ is the top-i element of a set $\{x_1,x_2,\cdots,x_n\}$. Furthermore, for $x_i,i=1,\cdots,n$, we have $\frac{1}{k}\sum_{i=1}^kx_{[i]}=\min_{s}\{s+\frac{1}{k}\sum_{i=1}^n[x_i-s]_+\}$, where $[a]_+=\max\{0,a\}$. The population version is $\mathbb{E}_{x}[x\cdot\mathbb{I}_{x\geq \eta(\alpha)}]=\min_s \frac{1}{\alpha} \mathbb{E}_{x}[\alpha s+[x-s]_+]$, where $\eta(\alpha)=\arg\min_{\eta\in\mathbb{R}}[\mathbb{E}_{x}[\mathbb{I}_{x\geq \eta}]=\alpha]$ (please see Appendix.\ref{section:proofs_section4} for the proof).
\end{lemma}

Lem.\ref{lemma:1} proposes an Average Top-$k$ (ATk) loss which is the surrogate loss for top-$k$ loss to eliminate the sorting problem. Optimizing the ATk loss is equivalent to selecting top-ranked instances. Actually, for $\mathcal{R}_{\beta}(f)$, we can just reformulate it as an  Average Top-$k$ (ATk) loss. Denote $\ell_-(\bm{x}')=(f(\bm{x}')-b)^2+2(1+\gamma)f(\bm{x}')$. In the proof of the next theorem, we will show that $\ell_-(\bm{x}')$ is an increasing function w.r.t. $f(\bm{x}')$, namely:
\begin{equation}
    \mathbb{E}_{\bm{x}'\sim\mathcal{D}_\mathcal{N}}[\mathbb{I}_{f(\bm{x}')\geq\eta_{\beta}(f)}\cdot\ell_-(\bm{x}')|f(\bm{x}') \ge \eta_\beta(f))]= \min_s \frac{1}{\beta}\cdot  \mathbb{E}_{\bm{x}'\sim\mathcal{D}_\mathcal{N}} [\beta s + [\ell_-(\bm{x}')-s]_+].
\end{equation}
The similar result holds for $\hat{\mathcal{R}}_{\beta}(f, S)$. Then, we can reach to Thm.\ref{thm:step2} (please see Appendix.\ref{section:proofs_section4} for the proof):
\begin{theorem}\label{thm:step2}
Assuming that $f(\bm{x})\in[0,1]$, for all $\bm{x}\in\mathcal{X}$, we have the equivalent optimization for $\mathrm{OPAUC}$: 
\begin{equation}
\begin{aligned}
    \underset{\cmin}{\min}\ \underset{\gamma\in[-1,1]}{\max} \ 
\colblue{\underset{\bm{z}\sim \mathcal{D}_\mathcal{Z}}{\mathbb{E}}}[F_{op}&(f,a,b,\gamma,\colblue{\efb}, \bm{z})]\\
&\Leftrightarrow \underset{\cmin}{\min}\ 
\underset{\cmax }{\max} 
\ \underset{s'\in\Omega_{s'}}{\min}
\ \colblue{\underset{\bm{z}\sim \mathcal{D}_\mathcal{Z}}{\mathbb{E}}}[G_{op}(f,a,b,\gamma,\bm{z},s')],
\label{minmaxmin_op}
\end{aligned}
\end{equation}
\begin{equation}
    \begin{aligned}
        \underset{\cmin}{\min}\ \underset{\gamma\in[-1,1]}{\max} \ 
    \colbit{\underset{\bm{z}\sim S}{\ehat}}[F_{op}&(f,a,b,\gamma, \colbit{\hefb},\bm{z})]\\
    &\Leftrightarrow \underset{\cmin}{\min}\ 
    \underset{\gamma\in\Omega_{\gamma} }{\max} 
    \ \underset{s'\in\Omega_{s'}}{\min}
    \ \colbit{\underset{\bm{z}\sim S}{\ehat}}[G_{op}(f,a,b,\gamma,\bm{z},s')],
    \label{minmaxmin_op}
    \end{aligned}
    \end{equation}

where $\Omega_{\gamma}=[b-1,1]$, $\Omega_{s'}=[0,5]$ and
\begin{equation}
    \begin{aligned}
G_{op}(f,a,b,&\gamma,\bm{z},s')=[(f(\bm{x})-a)^2- 
2(1+\gamma)f(\bm{x})]y/p-\gamma^2\\
& +
\left(\beta s' +\left[(f(\bm{x})-b)^2+2(1+\gamma) f(\bm{x})-s'\right]_+\right)(1-y)/[\beta (1-p)].
\end{aligned}
\label{OPAUC_IB}
\end{equation}
\label{theorem:2}
\end{theorem}

\underline{\textbf{\color{blue}{Step 3: Asymptotically Unbiased Smoothing.}}} Even with Thm.\ref{thm:step2}, it is hard to optimize the min-max-min formulation in Eq.\eqref{minmaxmin_op}. A solution is to swap the order $\max_{\gamma}$ and $\min_{s'}$ to reformulate it as a min-max problem. The key obstacle to this idea is the non-smooth function $[\cdot]_+$. To avoid the $[\cdot]_+$, we apply the \texttt{softplus} function \cite{glorot2011deep}:
\begin{equation}\label{eq:kappa}
    r_{\colblue{\kappa}}(x) = \frac{\log\left(1+\exp({\colblue{\kappa}} \cdot x)\right)}{{\colblue{\kappa}}},
\end{equation}
as a smooth surrogate. It is easy to show that $r_{\colblue{\kappa}}(x) \overset{\kappa \rightarrow \infty}{\rightarrow} [x]_+$. \textbf{Denote $G_{op}^{\colblue{\kappa}}(f,a,b,\gamma,\bm{z},s')$ the surrogate objective where the $[\cdot]_+$ in $G_{op}(f,a,b,\gamma,\bm{z},s')$ is replaced with $r_{\colblue{\kappa}}(\cdot)$}. We then proceed to solve the surrogate problem: 
\begin{equation}
\begin{aligned}
    &\underset{f,(a,b)\in[0, 1]^2}{\min}\ 
\underset{\gamma\in\Omega_{\gamma} }{\max} 
\ \underset{s'\in\Omega_{s'}}{\min}
\ \colblue{\underset{\bm{z}\sim \mathcal{D}_\mathcal{Z}}{\mathbb{E}}}[G^{\colblue{\kappa}}_{op}(f,a,b,\gamma,\bm{z},s')]
\\
    &\underset{f,(a,b)\in[0, 1]^2}{\min}\ 
\underset{\gamma\in\Omega_{\gamma} }{\max} 
\ \underset{s'\in\Omega_{s'}}{\min}
\ \colbit{\underset{\bm{z}\sim S}{\ehat}}[G^{\colblue{\kappa}}_{op}(f,a,b,\gamma,\bm{z},s')],
\end{aligned}
\label{kappa_estiamtor}
\end{equation}
respectively for the population and empirical version. In Appendix.\ref{uniform_convergence}, we will proof that such a approximation has a convergence rate $O(1/\kappa)$.

\underline{\textbf{\color{blue}{Step 4: The Regularized Problem.}}} It is easy to check that $r_\colblue{\kappa}(x)$ has a bounded second-order derivation. In this way, we can regard $G^{\colblue{\kappa}}_{op}(f,a,b,\gamma,\bm{z},s')$ as a weakly-concave function \cite{boyd2004convex} of $\gamma$. By employing an $\ell_2$ regularization, we turn to a regularized form:
\begin{equation*}
    G^{{\colblue{\kappa}},\colbit{\omega}}_{op}(f,a,b,\gamma,\bm{z} ,s') = G^{\colblue{\kappa}}_{op}(f,a,b,\gamma,\bm{z},s')  - \colbit{\omega} \cdot \gamma^2,
\end{equation*}
With a sufficiently large $\colbit{\omega}$,  $G^{{\colblue{\kappa}},\colbit{\omega}}_{op}(f,a,b,\gamma,\bm{z} ,s')$ is strongly-concave w.r.t. $\gamma$ when all the other variables are fixed. Note that the regularization scheme will inevitably bias. As a very general result, regularization will inevitably induce bias. However, it is known to be a necessary building block to stabilize the solutions and improve generalization performance. We then reach a minimax problem in the final step. 

\underline{\textbf{\color{blue}{Step 5: Min-Max Swapping.}}} According to min-max theorem \cite{boyd2004convex}, if we  replace $G^{\colblue{\kappa}}_{op}(f,a,b,\gamma,\bm{z},s')$ with $G^{{\colblue{\kappa}},\colbit{\omega}}_{op}(f,a,b,\gamma,\bm{z} ,s')$, the surrogate optimization problem satisfies:
\begin{equation}
    \underset{\cmin}{\min}\ \underset{\gamma\in\Omega_{\gamma}}{\max} \min_{s' \in \Omega_{s'}}
    \ \colblue{\underset{\bm{z}\sim \mathcal{D}_\mathcal{Z}}{\mathbb{E}}}[G_{op}^{{\colblue{\kappa}},\colbit{\omega}}]\Leftrightarrow \underset{\cminl}{\min}\ \underset{\cmax}{\max} 
\ \colblue{\underset{\bm{z}\sim \mathcal{D}_\mathcal{Z}}{\mathbb{E}}}[G_{op}^{{\colblue{\kappa}},\colbit{\omega}}],
\label{minmax_op}
\end{equation}

\begin{equation}
    \underset{\cmin}{\min}\ \underset{\gamma\in\Omega_{\gamma}}{\max} \min_{s' \in \Omega_{s'}}
    \ \colbit{\underset{\bm{z}\sim S}{\ehat}}[{G}_{op}^{{\colblue{\kappa}},\colbit{\omega}}]\Leftrightarrow \underset{\cminl}{\min}\ \underset{\cmax}{\max} 
\ \colbit{\underset{\bm{z}\sim S}{\ehat}}[{G}_{op}^{{\colblue{\kappa}},\colbit{\omega}}],
\label{minmax_op_em}
\end{equation}
where $G_{op}^{{\colblue{\kappa}},\colbit{\omega}} = G_{op}^{{\colblue{\kappa}},\colbit{\omega}}(f,a,b,\gamma,\bm{z}, s')$. In this sense, we come to a regularized non-convex strongly-concave problem. In Sec.\ref{sec:alg}, we will employ an efficient solver to optimize the parameters.




\subsection{Optimizing the TPAUC}
 According to Eq.\eqref{TPAUCM}, given a surrogate loss $\ell$ and finite dataset $S$, maximizing $\mathrm{TPAUC}$ and $\hat{\mathrm{AUC}}_{\alpha, \beta}(f, S) $ is equivalent to solving the following problems, respectively:
\begin{equation}
    \underset{f}{\min}~  \mathcal{R}_{\alpha, \beta}(f)= \mathbb{E}_{\bm{x} \sim \mathcal{D}_\mathcal{P}, \bm{x}'\sim \mathcal{D}_\mathcal{N}} \left[\mathbb{I}_{f(\bm{x}') \ge \eta_\beta(f)} \cdot \mathbb{I}_{f(\bm{x}) \le \eta_\alpha(f)}  \cdot \ell(f(\bm{x})- f(\bm{x}'))\right] .
    \label{OPAUCO}
\end{equation}
\begin{equation}
\begin{aligned}
\underset{f}{\min} \ \hat{\mathcal{R}}_{\alpha, \beta}(f, S)= \sum_{i=1}^{\npa} \sum_{j=1}^{\nnb}\frac{\ell{\left(f(\bm{x}_{[i]})- f(\bm{x}'_{[j]})\right)}}{n_+^\alpha n_-^\beta}.
\label{OPAUCO}
\end{aligned}
\end{equation}
Due to the limited space, we present the result directly, please refer to Appendix.\ref{sec:tpauc_reformulation} for more details. 
Similar to $\mathrm{OPAUC}$, we apply the function $r_{\colblue{\kappa}}(x)$, regularization $\colbit{\omega} \gamma^2$ and min-max theorem to solve the problem. In this sense, we can use
\begin{equation}
    \min_{f, (a, b) \in [0, 1]^2, s \in \Omega_{s}, s' \in \Omega_{s'}} \max _{\gamma \in \Omega_{\gamma}} \colblue{\underset{\bm{z}\sim \mathcal{D}_\mathcal{Z}}{\mathbb{E}}}\left[G_{t p}^{\colblue{\kappa}, \colbit{\omega}}\left(f, a, b, \gamma, \bm{z}, s, s'\right)\right],
\end{equation}
where $\Omega_{\gamma}=[\max\{-a,b-1\},1]$ and
\begin{equation}
    \min _{f, (a, b) \in [0, 1]^2, s\in \Omega_{s}, s' \in \Omega_{s'}} \max _{\gamma \in \Omega_{\gamma}} \colbit{\underset{\bm{z}\sim S}{\hat{\mathbb{E}}}}\left[G_{t p}^{\colblue{\kappa}, \colbit{\omega}}\left(f, a, b, \gamma, \bm{z}, s, s'\right)\right],
    \label{minmax_tp_em}
\end{equation}
to minimize $\mathcal{R}_{\alpha, \beta}(f)$, and $\hat{\mathcal{R}}_{\alpha, \beta}(f)$, respectively. Here:

\begin{equation}
    \begin{aligned}
G_{tp}^{\colblue{\kappa},\colbit{\omega}}(f,a,b,\gamma,&\bm{z},s,s')=\left(\alpha s + r_{{\colblue{\kappa}}}\left((f(\bm{x})-a)^2- 
2(1+\gamma)f(\bm{x})-s\right)\right)y/(\alpha p) -(\colbit{\omega}+1)\gamma^2\\
& +
\left(\beta s' +r_{{\colblue{\kappa}}}\left((f(\bm{x})-b)^2+2(1+\gamma) f(\bm{x})-s'\right)\right)(1-y)/[\beta (1-p)].
\label{TPAUC_OB}
\end{aligned}
\end{equation}

According to Thm.2 of \cite{tsaknakis2021minimax}, we have the following corollary:
\begin{corollary}
We can reformulate Eq.\eqref{minmax_op_em} and Eq.\eqref{TPAUC_OB} as an off-the-shelf minimax problem where the coupled constraint is replaced with the Lagrange multipliers ($\theta_b$ for OPAUC, $\theta_b, \theta_a$ for TPAUC).
For $\mathrm{OPAUC}$:
\begin{equation}
\begin{aligned}
\min_{f,(a,b)\in[0,1]^2,s\in\Omega_{s}}&\max_{\gamma\in[b-1,1]} \colblue{\underset{\bm{z}\sim \mathcal{D}_\mathcal{Z}}{\mathbb{E}}}[G_{op}^{\colblue{\kappa},\colbit{\omega}}]
\Leftrightarrow
\min_{f,(a,b)\in[0,1]^2,s\in\Omega_{s},\theta_b \in[0,M_1]}\max_{\gamma\in[-1,1]} \colblue{\underset{\bm{z}\sim \mathcal{D}_\mathcal{Z}}{\mathbb{E}}}[G_{op}^{\colblue{\kappa},\colbit{\omega}}].\\
&\qquad \qquad \qquad \qquad \qquad \qquad \qquad \qquad \qquad \qquad  - \theta_{b}(b-1-\gamma)
\end{aligned}
\end{equation}
For $\mathrm{TPAUC}$:
\begin{equation}
\begin{aligned}
    &\min_{f,(a,b)\in[0,1]^2,s\in\Omega_{s},s'\in\Omega_{s'}}\max_{\gamma\in[\max\{-a,b-1\},1]} \colblue{\underset{\bm{z}\sim \mathcal{D}_\mathcal{Z}}{\mathbb{E}}}[G_{tp}^{\colblue{\kappa},\colbit{\omega}}] 
\\ \Leftrightarrow
&\min_{f,(a,b)\in[0,1]^2,s\in\Omega_{s},s'\in\Omega_{s'},\theta_a \in[0,M_2],\theta_b \in[0,M_3]}\max_{\gamma\in[-1,1]} \colblue{\underset{\bm{z}\sim \mathcal{D}_\mathcal{Z}}{\mathbb{E}}}[G_{tp}^{\colblue{\kappa},\colbit{\omega}}] \\
 &\qquad \qquad \qquad \qquad  -\theta_{b}(b-1-\gamma)- \theta_{a}(-a-\gamma).
\end{aligned}
\end{equation}
\end{corollary}

The tight constraint $\theta_b \in [0,M_1]$/$\theta_b \in [0,M_2], \theta_a \in [0,M_3]$ comes from the fact that optimum $\theta_b,\theta_a$ are both finite since the objective function is bounded from above. In the experiments, to make sure that $M_1,M_2,M_3$ are large enough, we set them as $M_1 = M_2 =M_3 =10^9$.

\section{Training Algorithm}\label{sec:alg}
According to the derivations in the previous sections, our goal is then to solve the resulting  empirical minimax optimization problems in Eq.\eqref{minmax_op_em} and Eq.\eqref{minmax_tp_em}. It is easy to check that they are strongly-concave w.r.t $\gamma$ whenever $\kappa \le 2 + 2\omega$, when $(f(\bm{x}),a,b)\in[0,1]^3$, $\gamma\in[-1,1]$, $s\in\Omega_{s},s'\in \Omega_{s'}$. Therefore, we can adopt the nonconvex strongly concave minimax optimization algorithms to solve these problems \cite{huang2022accelerated}. In this section, following the work \cite{huang2022accelerated}, we employ an accelerated stochastic gradient descent ascent (ASGDA) method to solve the minimax optimization problem. We denote $\bm{\theta}\in\mathbb{R}^d$ as the parameters of function $f$, $\bm{\tau}=\{\bm{\theta},a,b,s,s',\theta_a,\theta_b\}\in \Omega_{\tau}$ as the variables for the outer min-problem. Alg.\ref{alg:1} shows the framework of our algorithm 
(we adopt the accelerated algorithm in \cite{huang2022accelerated} to solve our problem). There are two key steps. At \texttt{Line 5-6} of Algorithm \ref{alg:1}, variables $\bm{\tau}_{t+1}$ and $\gamma_{t+1}$ are updated in a momentum way. Moreover, the convex combination ensures that they are always feasible given that the initial solution is feasible. At \texttt{Line 9-10}, using the momentum-based variance reduced technique, we can estimate the stochastic first-order partial gradients $\bm{v}_t$ and $w_t$ in a more stable manner.

\label{Training_Algorithm}
\begin{algorithm}[!h]
    \caption{Accelerated Stochastic Gradient Descent Ascent Algorithm}
    \begin{algorithmic}[1]
    \label{alg:1}
    \STATE \textbf{Input}: {Dataset $\mathcal{X}$, learning parameters $\{\nu, \lambda, k, m, c_1, c_2, T\}$}
    \STATE \textbf{Initialize:} Randomly select $\bm{\tau}_0=\{\bm{\theta}_0$, $a_0$, $b_0$, $s_0$, $s^\prime_0$, $\theta_a$, $\theta_b\}$ from $\Omega_{\bm{\tau}}$, $\bm{v}_0=\bm{0}^{d+6}$, $w_0=0$. \\
    \ \ \ \ \ \ \ \ \ \ \ \ \ \ \ \ \ \  Randomly select $\gamma_0$ from $\Omega_{\gamma}$, $t=0$,  
    \FOR{$t=0,1,\cdots,T$}
        \STATE Compute the learning rate $\eta_t=\frac{k}{(m+t)^{1/3}}$;
        \STATE Update $\bm{\tau}_{t+1}=(1-\eta_t)\bm{\tau}_t + \eta_t \mathcal{P}_{\Omega_{\bm{\tau}}}(\bm{\tau}_t-\nu \bm{v}_t)$;
        \STATE Update $\gamma_{t+1} = (1-\eta_t)\gamma_t+\eta_t \mathcal{P}_{\Omega_{\gamma}}(\gamma_t+\lambda w_t)$;
        \STATE Compute $\rho_{t+1}=c_1\eta_t^2$ and $\xi_{t+1}=c_2 \eta_t^2$;
        \STATE Sampling mini-batch data $\mathcal{B}_{t+1}$ from dataset $\mathcal{X}$;
        \STATE Update $\bm{v}_{t+1}=\nabla_{\bm{\tau}}G_{(\cdot)}^{\colblue{\kappa},\colbit{\omega}}(\bm{\tau}_{t+1},\gamma_{t+1};\mathcal{B}_{t+1})+(1-\rho_{t+1})[\bm{v}_{t}-\nabla_{\bm{\tau}}G_{(\cdot)}^{\colblue{\kappa},\colbit{\omega}}(\bm{\tau}_t,\gamma_t,\mathcal{B}_{t+1})]$;
        \STATE Update
        $w_{t+1}=\nabla_{\gamma}G_{(\cdot)}^{\colblue{\kappa},\colbit{\omega}}(\bm{\tau}_{t+1},\gamma_{t+1};\mathcal{B}_{t+1})+(1-\xi_{t+1})[w_{t}-\nabla_{\gamma}G_{(\cdot)}^{\colblue{\kappa},\colbit{\omega}}(\bm{\tau}_t,\gamma_t,\mathcal{B}_{t+1})]$;
    \ENDFOR
    \STATE \textbf{Return} $\bm{\theta}_{T+1}$
    \end{algorithmic}
\end{algorithm}

With the following smoothness assumption, we can get the convergence rate in Thm.\ref{thm:opt}.
\begin{assumption}\label{asum:lip}
$G_{(\cdot)}^{\colblue{\kappa},\colbit{\omega}}(\bm{\tau},\gamma;\mathcal{B})$ has Lipschitz continuous gradients, \ie, there is a positive scalar $L_G$ such that for any $\bm{\tau},\bm{\tau}'\in\Omega_{\bm{\tau}}$, $\gamma,\gamma'\in\Omega_{\gamma}$,
\begin{equation}
    \begin{aligned}
    &\|\nabla G_{(\cdot)}^{\colblue{\kappa},\colbit{\omega}}(\bm{\tau}, \gamma; \mathcal{B}) - \nabla G_{(\cdot)}^{\colblue{\kappa},\colbit{\omega}}(\bm{\tau}', \gamma'; \mathcal{B})\| \leq L_{G}(\|\bm{\tau}-\bm{\tau}'\| + \|\gamma-\gamma')\|).
    \end{aligned}
\end{equation}
\end{assumption}
\begin{theorem}\label{thm:opt}
 (Theorem 9 \cite{huang2022accelerated}) Supposing that Asm.\ref{asum:lip} holds, let $\{\bm{\tau}_t,\gamma_t\}$ be a sequence generated by our method, if the learning rate satisfies:
\begin{equation}
\begin{aligned}
    &c_1\geq \frac{2}{3k^3}+\frac{9\tau^2}{4}, \qquad c_2\geq \frac{2}{3k^3}+\frac{75L_G^2}{2}\\
    &m\geq \max(2, k^3, (c_1 k)^3, (c_2 k)^3), \lambda \leq \min\left(\frac{1}{6L_{G}},\frac{27b\mu}{16}\right)\\
    &\nu\leq \min(\frac{\lambda \tau}{2L_G}\sqrt{\frac{2b}{8\lambda^2+75(L_G/\mu)^2b}},\frac{m^{1/3}}{2(L_G+\frac{L_G^2}{\mu})k}).
\end{aligned}
\end{equation}
Then we have:
\begin{equation}
     \frac{1}{T} \sum_{t=1}^{T} \mathbb{E}\left[\left\|\frac{1}{\nu}(\bm{\tau}_t-\mathcal{P}_{\Omega_{\bm{\tau}}}(\bm{\tau}_{t}-\nu \bm{v}_t))\right\|\right] \leq \frac{2 \sqrt{3 M^{\prime \prime}} m^{1 / 6}}{T^{1 / 2}}+\frac{2 \sqrt{3 M^{\prime \prime}}}{T^{1 / 3}},
\end{equation}
where $\|\frac{1}{\nu}(\bm{\tau}_t-\mathcal{P}_{\Omega_{\bm{\tau}}}(\bm{\tau}_{t}-\nu \nabla F_{(\cdot)}(\bm{\tau}_t)))\|$ is the $l_2$-norm of gradient mapping metric for the outer problem  \cite{dunn1987convergence, ghadimi2016mini, razaviyayn2020nonconvex} with $F_{(\cdot)}(\bm{\tau}_t)=\max_{\gamma\in\Omega_{\gamma}}G_{(\cdot)}^{\colblue{\kappa},\colbit{\omega}}(\bm{\tau}_t,\gamma)$.  When $b=1$, it is easy to verify that $k=O(1)$, $\lambda=O(\mu)$, $\nu^{-1}=O(L_G/\mu)$, $c_1=O(1)$, $c_2=O(L_G^2)$ and $m=O(L_G^6)$. Then we have $M''=O(L_G^3/\mu^3)$. Thus, the algorithm has a convergence rate of $O(\frac{(L_G/\mu)^{3/2}}{T^{1/3}})$. By $\frac{(L_G/\mu)^{3/2}}{T^{1/3}}\leq \epsilon$, then the iteration number to achieve $\epsilon$-first-order saddle point which satisfies: $T\geq (L_G/\mu)^{4.5}\epsilon^{-3}$.
\label{theorem:algorithm}
\end{theorem}

\section{Generalization Analysis}
In this section, we theoretically analyze the generalization performance of our proposed estimators for $\mathrm{OPAUC}$ (please see the Appendix.\ref{section:proofs_generalization} for the $\mathrm{TPAUC}$). According to Thm.\ref{thm:step2} in Sec.\ref{sec:method}, we know that the generalization error of OPAUC with a surrogate loss $\ell$ can be measured as:
\begin{equation}
    \mathcal{R}_{\beta}(f) \propto \underset{\cmins}{\min}\ 
    \underset{\cmax }{\max}  \underset{s'\in\Omega_{s'}}{\min}
    \ \colblue{\underset{\bm{z}\sim \mathcal{D}_\mathcal{Z}}{\mathbb{E}}}[G_{op}(f,a,b,\gamma,\bm{z},s')],
\end{equation}
and 
\begin{equation}
    \hat{\mathcal{R}}_{\beta}(f) \propto  \underset{\cmins}{\min}\ 
    \underset{\cmax }{\max}  \underset{s'\in\Omega_{s'}}{\min}
    \ \colbit{\underset{\bm{z}\sim S}{\ehat}}[G_{op}(f,a,b,\gamma,\bm{z},s')],
\end{equation}
Following the ERM paradigm, to prove the uniform convergence result over a hypothesis class $\mathcal{F}$ of the scoring function $f$, we need to show that:
\begin{equation*}
    \begin{split}
        \sup_{f \in \mathcal{F}}\left[ \mathcal{R}_{\beta}(f)  - \hat{\mathcal{R}}_{\beta}(f) \right] \le \epsilon. 
    \end{split}
\end{equation*}
According to the aforementioned discussion, we only need to prove that:
\begin{equation*}
    \begin{split}
        \sup_{f \in \mathcal{F}} & \left[ \underset{\cmins}{\min}\ 
        \underset{\cmax }{\max}  \underset{s'\in\Omega_{s'}}{\min}
        \ \colblue{\underset{\bm{z}\sim \mathcal{D}_\mathcal{Z}}{\mathbb{E}}}[G_{op}(f,a,b,\gamma,\bm{z},s')] \right.\\ 
        & - \left. \underset{\cmins}{\min}\ 
        \underset{\cmax }{\max}  \underset{s'\in\Omega_{s'}}{\min}
        \ \colbit{\underset{\bm{z}\sim S}{\ehat}}[G_{op}(f,a,b,\gamma,\bm{z},s')] \right] \le \epsilon.
    \end{split}
\end{equation*}

To prove this, we need to define the measure of the complexity of the class $\mathcal{F}$.  Here we adopt the Radermacher complexity $\Re$  as in \cite{bartlett2002rademacher}. Specifically, we come to the following definition:
\begin{definition}
The empirical Rademacher complexity of positive and negative instances with respect to $S$ is defined as:
\begin{equation}
    \hat{\Re}_{+}(\mathcal{F})=\underset{\bm{\sigma}}{\mathbb{E}}\left[\underset{f\in\mathcal{F}}{\sup} \
\frac{1}{n_+}\sum_{i=1}^{n_+}\sigma_i f(\bm{x}_i)\right],
\end{equation}

\begin{equation}
    \hat{\Re}_{-}(\mathcal{F})=\underset{\bm{\sigma}}{\mathbb{E}}\left[\underset{f\in\mathcal{F}}{\sup} \
    \frac{1}{n_-}\sum_{j=1}^{n_-}\sigma_j f(\bm{x}'_j)\right]
\end{equation}
where $(\sigma_1, \cdots, \sigma_{n_+})$ and $(\sigma_1, \cdots, \sigma_{n_-})$ are independent uniform random variables taking values in $\{-1, +1\}$. 
\end{definition}

Finally, we come to the generalization bound as follows:

\begin{theorem}
For any $\delta>0$, with probability at least $1-\delta$ over the draw of an i.i.d. sample set $S$ of size $n$, for all $f\in\mathcal{F}$ we have:
\begin{equation*}
    \begin{aligned}
\underset{\cmins}{\min}\ 
\underset{\cmax }{\max}  \underset{s'\in\Omega_{s'}}{\min}
\ \colblue{\underset{\bm{z}\sim \mathcal{D}_\mathcal{Z}}{\mathbb{E}}}[G_{op}(f,a,b,\gamma,\bm{z},s')] &\le \underset{\cmins}{\min}\  \underset{\cmax }{\max}  \underset{s'\in\Omega_{s'}}{\min}
\ \colbit{\underset{\bm{z}\sim S}{\ehat}}[G_{op}(f,a,b,\gamma,\bm{z},s')] \\ 
&+ O(\hat{\Re}_{+}(\mathcal{F}) + \hat{\Re}_{-}(\mathcal{F})) + O( \np^{-1/2} + \beta^{-1}\nn^{-1/2})
\end{aligned}
\end{equation*}
\label{theorem:4}
\end{theorem}

\begin{remark}
    Although the results we obtain are similar to some previous studies. \cite{yang2021all, narasimhan2017support}, our generalization analysis is simpler and does not require complex error decomposition. Moreover, our results hold for all real-valued hypothesis class with outputs in $[0,1]$, while the previous results \cite{yang2021all, narasimhan2017support} only hold for hard-threshold functions.
\end{remark}

\section{Experiment}
In this section, we conduct a series of experiments on different datasets for both $\mathrm{OPAUC}$ and $\mathrm{TPAUC}$ optimization. Due to space limitations, please refer to the Appendix.\ref{section:experiment_details} for the details of implementation and competitors. The source
code is available in \url{https://github.com/Shaocr/PAUCI}.
\subsection{Setups}
We adopt three imbalanced binary classification datasets:  CIFAR-10-LT \cite{elson2007asirra}, CIFAR-100-LT \cite{krizhevsky2009learning} and Tiny-ImgaeNet-200-LT following the instructions in \cite{yang2021all}, where the binary datasets are constructed by selecting one super category as positive class and the other categories as negative class. Please see Appendix.\ref{section:experiment_details} for more details. The evaluation metrics in experiments are $\hat{\mathrm{AUC}}_{\beta}$ and $\hat{\mathrm{AUC}}_{\alpha, \beta}$.

\subsection{Overall Results}

In Tab.\ref{tab:3}, Tab.\ref{tab:4}, we record the performance on test sets of all the methods on three subsets of CIFAR-10-LT, CIFAR-100-LT, and Tiny-Imagent-200-LT. Each method is tuned independently for $\mathrm{OPAUC}$ and $\mathrm{TPAUC}$ metrics. From the results, we make the following remarks: (1) Our proposed methods outperform all the competitors  in most cases. Even for failure cases, our methods attain fairly competitive results compared with the competitors for $\mathrm{OPAUC}$ and $\mathrm{TPAUC}$. (2) In addition, we can see that the normal AUC optimization method AUC-M has less reasonable performance under PAUC metric. This demonstrates the necessity of developing the PAUC optimization algorithm. (3) Approximation methods SOPA-S, AUC-poly, and AUC-exp have lower performance than the unbiased algorithm SOPA and our instance-wise algorithm PAUCI in most cases. Above all, the experimental results show the effectiveness of our proposed method.
\begin{table}[h]
    \centering
    \small
    \caption{
        OPAUC ($\mathrm{FPR}\leq 0.3$) on testing data of different imbalanced datasets. The highest and the second best results are highlighted in \textcolor{orgin}{orange} and \textcolor{second}{blue}, respectively.
    }
    \setlength\tabcolsep{2.5pt}
    \label{tab:3}
        \begin{tabular}{c|ccc|ccc|ccc}
    \toprule
        ~ & \multicolumn{3}{c|}{CIFAR-10-LT} & \multicolumn{3}{c|}{CIFAR-100-LT} & \multicolumn{3}{c}{Tiny-Imagenet-LT} \\ \midrule Methods & Subset 1 & Subset 2& Subset 3&Subset 1&Subset 2&Subset 3&Subset 1&Subset 2&Subset 3 \\ \midrule
        SOPA \cite{zhu2022auc} & \textcolor{second}{0.7659} & \textcolor{second}{0.9688} & \textcolor{second}{0.7651} & \textcolor{second}{0.9108} & \textcolor{second}{0.9875} & 0.8483& 0.8157 & 0.9037 & 0.9066 \\ 
        SOPA-S \cite{zhu2022auc} & 0.7548 & 0.9674 & 0.7542 & 0.9033 & 0.9860 & 0.8449& 0.8180 & \textcolor{second}{0.9087} & 0.9095 \\ 
        AGD-SBCD \cite{zhu2022auc} & 0.7526 & 0.9615 & 0.7497 & 0.9105 & 0.9814 & 0.8406& 0.8135 & 0.9081 & 0.9057 \\ 
        AUC-poly \cite{yang2021all} & 0.7542 & 0.9672 & 0.7538 & 0.9027 & 0.9859 & 0.8441& 0.8185 & 0.9084 & \textcolor{second}{0.9100} \\ 
        AUC-exp \cite{yang2021all} & 0.7347 & 0.9620 & 0.7457 & 0.8987 & 0.9850 & 0.8407& 0.8127 & 0.9026 & 0.9049 \\ 
        CE & 0.7417& 0.9431 & 0.7428 & 0.8903 & 0.9695 & 0.8321& 0.8023 & 0.8917 & 0.8878 \\
        MB \cite{kar2014online} & 0.7492 & 0.9648 & 0.7500 & 0.9003 & 0.9804 & \textcolor{orgin}{\textbf{0.8575}}& \textcolor{second}{0.8193} & 0.9072 & 0.9091 \\ 
        AUC-M \cite{ying2016stochastic}& 0.7334 & 0.9609 & 0.7442 & 0.8996 & 0.9845 & 0.8403 & 0.8102 & 0.9011 & 0.9043 \\ \midrule
        PAUCI & \textcolor{orgin}{\textbf{0.7721}} & \textcolor{orgin}{\textbf{0.9716}} & \textcolor{orgin}{\textbf{0.7746}} & \textcolor{orgin}{\textbf{0.9155}} & \textcolor{orgin}{\textbf{0.9889}} & \textcolor{second}{0.8492} & \textcolor{orgin}{\textbf{0.8267}}&
        \textcolor{orgin}{\textbf{0.9214}} & \textcolor{orgin}{\textbf{0.9217}}\\
        \bottomrule
    \end{tabular}
\end{table}

\begin{table}[h]
    \centering
    \small
    \caption{
    TPAUC ($\mathrm{TPR}\geq 0.5$, $\mathrm{FPR}\leq 0.5$) on testing data of different imbalanced datasets.
    }
    \setlength\tabcolsep{2.5pt}
    \label{tab:4}
    \begin{tabular}{c|ccc|ccc|ccc}
    \toprule
        ~ & \multicolumn{3}{c|}{CIFAR-10-LT} & \multicolumn{3}{c|}{CIFAR-100-LT} & \multicolumn{3}{c}{Tiny-Imagenet-LT} \\ \midrule Methods & Subset 1 & Subset 2& Subset 3&Subset 1&Subset 2&Subset 3&Subset 1&Subset 2&Subset 3 \\ \midrule
        SOPA \cite{zhu2022auc} & \textcolor{second}{0.7096} & \textcolor{second}{0.9593} & \textcolor{second}{0.7220} & \textcolor{second}{0.8714} & \textcolor{second}{0.9855} & 0.7485& \textcolor{second}{0.7417} & \textcolor{second}{0.8681} & \textcolor{second}{0.8650} \\ 
        SOPA-S \cite{zhu2022auc} & 0.6603& 0.9456 & 0.6917 & 0.8617 & 0.9812 & 0.7419& 0.7354 & 0.8666 & 0.8628 \\ 
        AUC-poly \cite{yang2021all} & 0.6804 & 0.9543 & 0.6974 & 0.8618 & 0.9835 & 0.7431& 0.7349 & 0.8676 & 0.8627 \\ 
        AUC-exp \cite{yang2021all}& 0.6669 & 0.9493 & 0.6930 & 0.8613 & 0.9827 & 0.7447& 0.7328 & 0.8672 & 0.8626 \\ 
        CE & 0.6420 & 0.9353 & 0.6798 & 0.8467 & 0.9603 & 0.7311& 0.7223 & 0.8517 & 0.8478 \\
        MB \cite{kar2014online}& 0.6437 & 0.9492 & 0.6913 & 0.8665 & 0.9677 & \textcolor{orgin}{\textbf{0.7583}}& 0.7348 & 0.8651 & 0.8624 \\ 
        AUC-M \cite{ying2016stochastic} & 0.6520 & 0.9381 & 0.6821 & 0.8505 & 0.9822 & 0.7324 & 0.7361 & 0.8517 & 0.8598 \\\midrule
        PAUCI & \textcolor{orgin}{\textbf{0.7192}} & \textcolor{orgin}{\textbf{0.9663}} & \textcolor{orgin}{\textbf{0.7305}} & \textcolor{orgin}{\textbf{0.8814}} & \textcolor{orgin}{\textbf{0.9874}} & \textcolor{second}{0.7497} & \textcolor{orgin}{\textbf{0.7618}} &\textcolor{orgin}{\textbf{0.8875}} & \textcolor{orgin}{\textbf{0.8860}}\\
        \bottomrule
    \end{tabular}
\end{table}

%
%

\subsection{Convergence Analysis}
In the convergence experiments, for sake of fairness, we did not use warm-up. All algorithms use hyperparameters in the performance experiments. We show the plots of training convergence in Fig.\ref{fig:opaucconvergence} and Fig.\ref{fig:tpaucconvergence} on CIFAR-10 for both $\mathrm{OPAUC}$ and $\mathrm{TPAUC}$. Due to the space limitation, the other results could be found in Appendix.\ref{section:experiment_details}. According to the figures, we can make the following observations: (1) Our algorithm and SOPA converge faster than other methods for $\mathrm{OPAUC}$. However, for $\mathrm{TPAUC}$ optimization, the SOPA converges very slowly due to its complicated algorithm, while our method still shows the best convergence property in most cases. (2) It's notable that our algorithm converges to stabilize after twenty epochs in most cases. That means our method has better stability in practice.

\begin{figure}[htbp]
	\centering
		\begin{minipage}{0.32\linewidth}
		\centering
		\includegraphics[width=\linewidth]{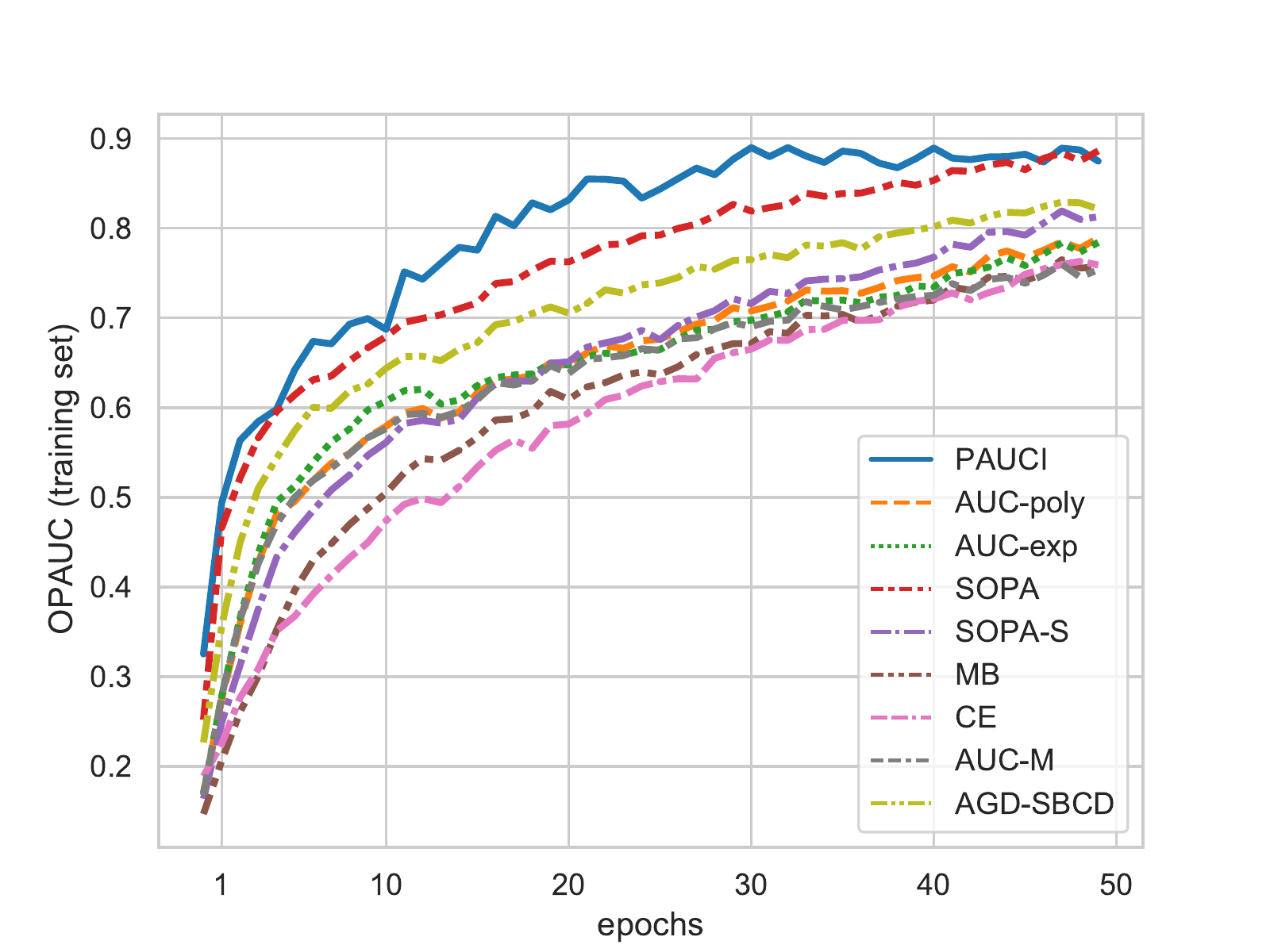}
		\caption*{(a) CIFAR-100-Long-Tail-1}
	\end{minipage}
	\begin{minipage}{0.32\linewidth}
		\centering
		\includegraphics[width=\linewidth]{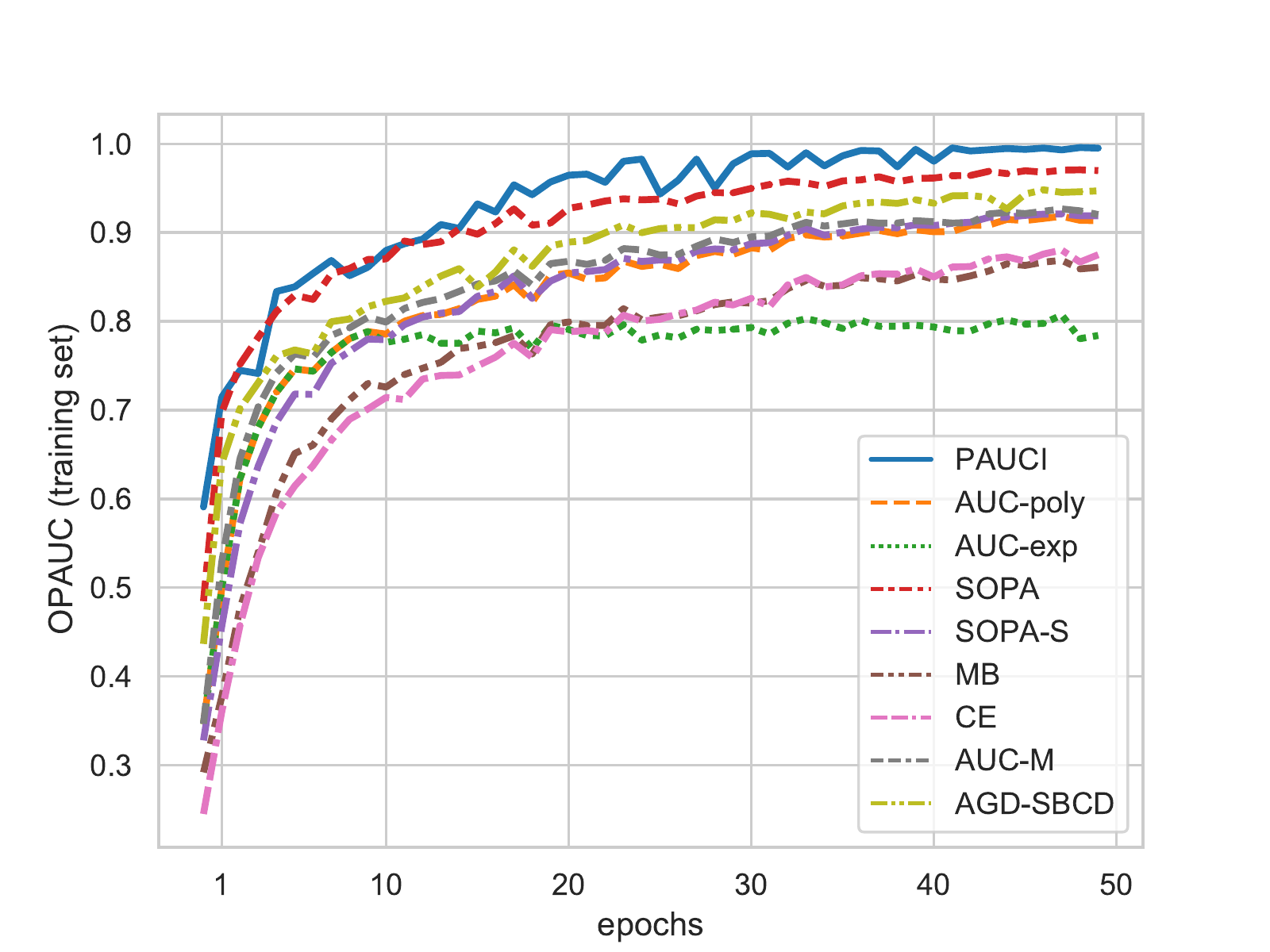}
		\caption*{(b) CIFAR-100-Long-Tail-2}
	\end{minipage}
	\begin{minipage}{0.32\linewidth}
		\centering
		\includegraphics[width=\linewidth]{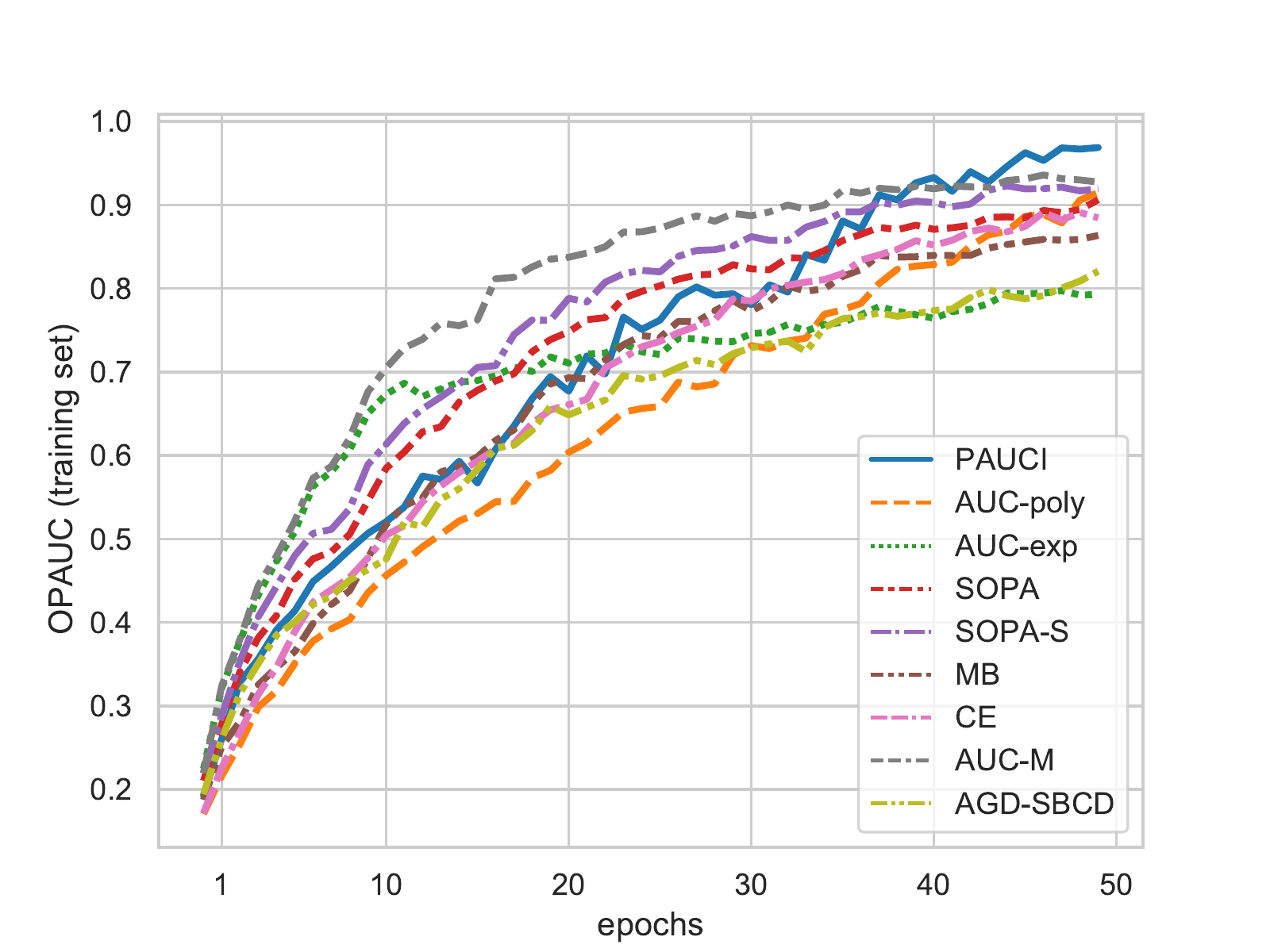}
		\caption*{(c) CIFAR-100-Long-Tail-3}
	\end{minipage}
	\caption{Convergence of OPAUC optimization.}
\end{figure}

\begin{figure}[htbp]
	\centering
		\begin{minipage}{0.32\linewidth}
		\centering
		\includegraphics[width=\linewidth]{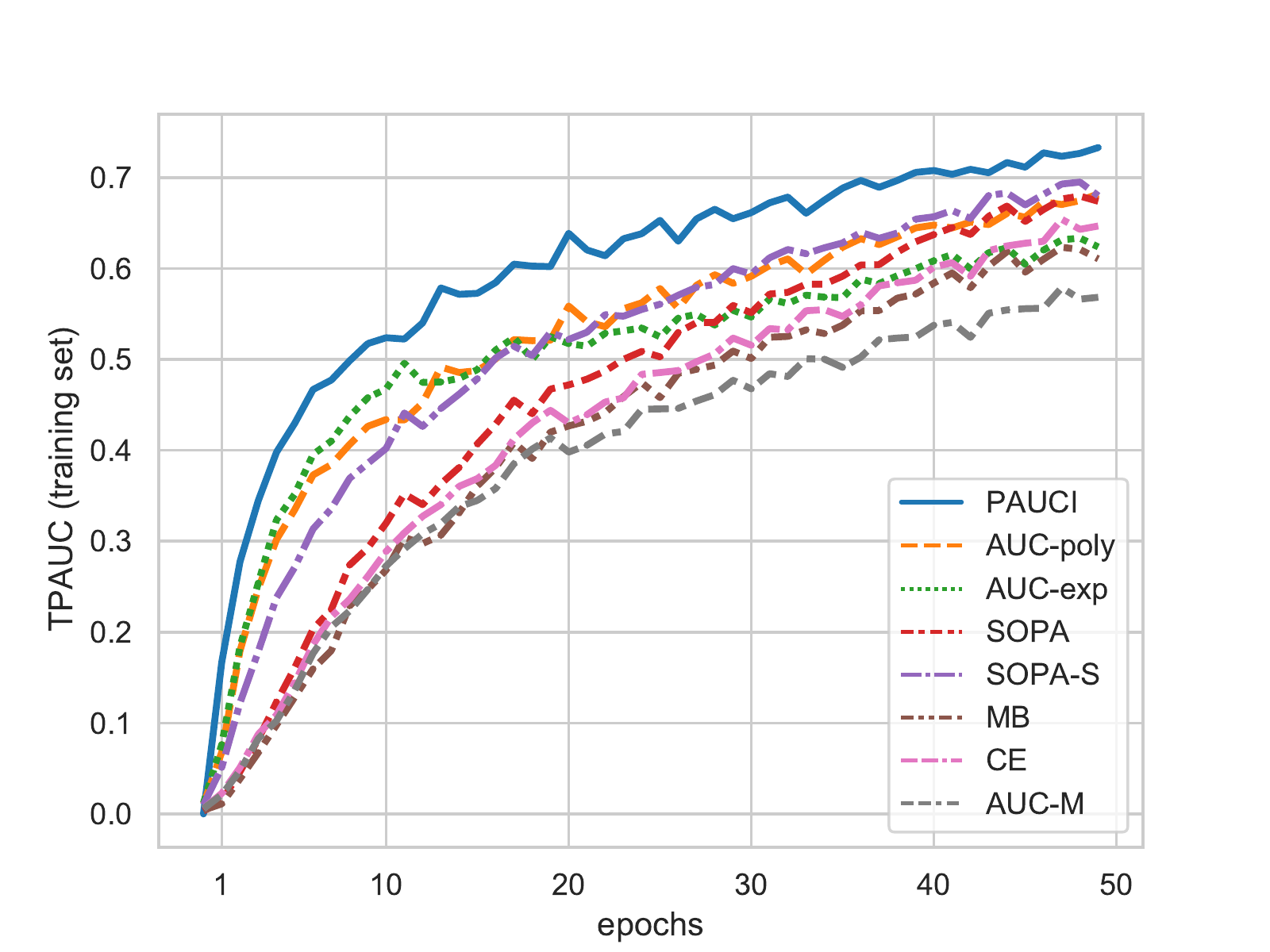}
		\caption*{(d) CIFAR-100-Long-Tail-1}
	\end{minipage}
	\begin{minipage}{0.32\linewidth}
		\centering
		\includegraphics[width=\linewidth]{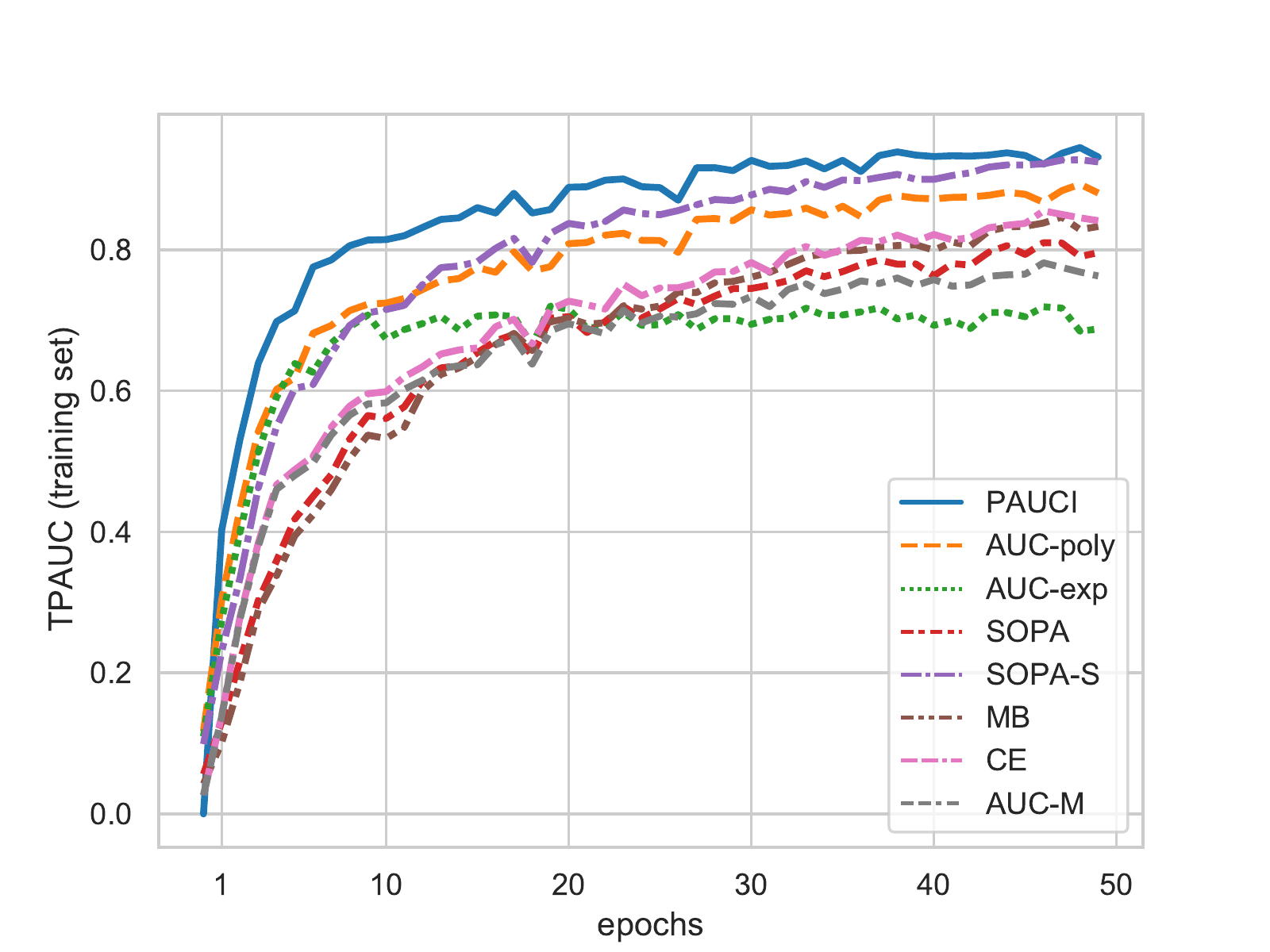}
		\caption*{(e) CIFAR-100-Long-Tail-2}
	\end{minipage}
	\begin{minipage}{0.32\linewidth}
		\centering
		\includegraphics[width=\linewidth]{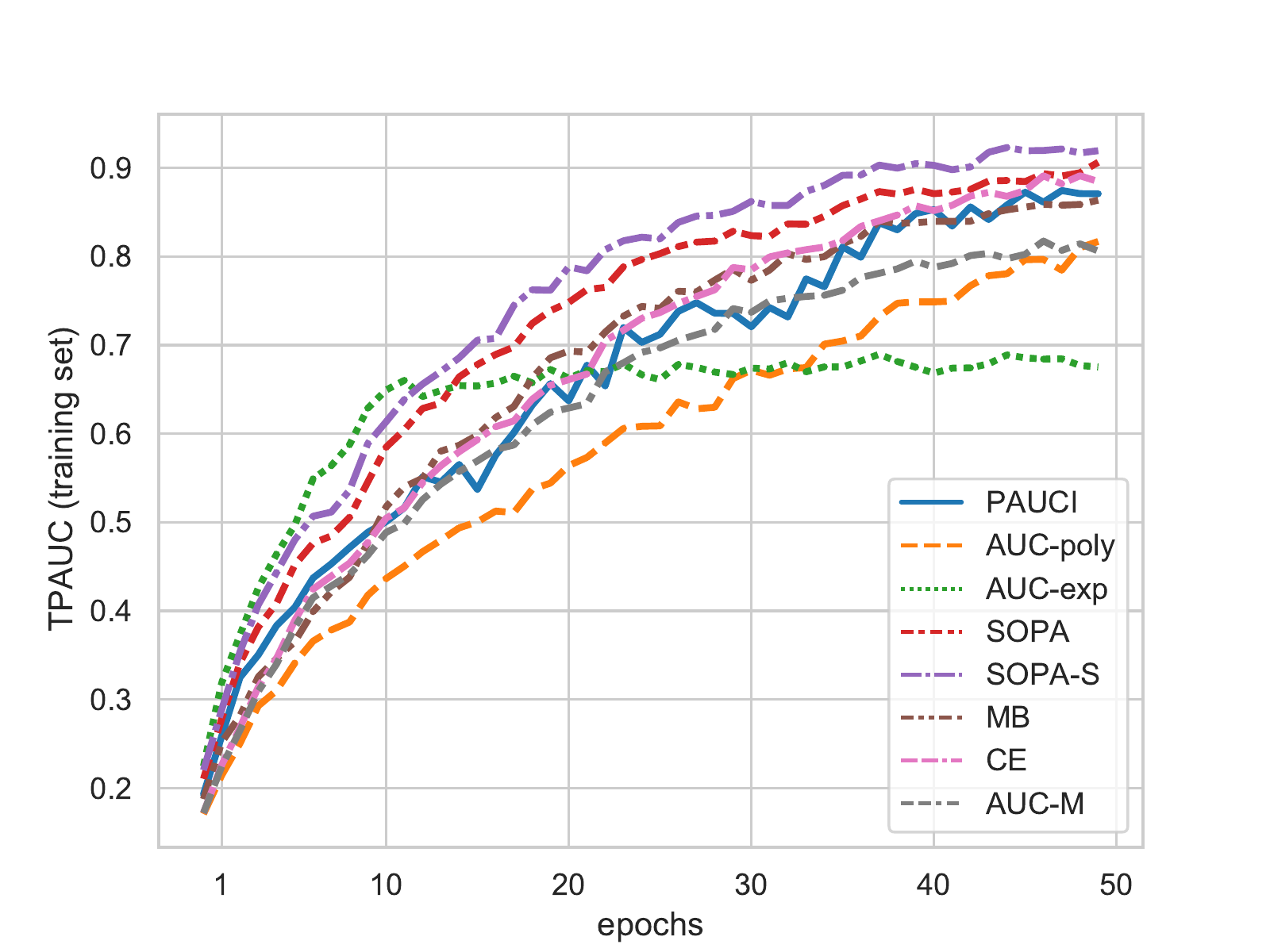}
		\caption*{(f) CIFAR-100-Long-Tail-3}
	\end{minipage}
	\caption{Convergence of TPAUC optimization.}
\end{figure}

\section{Conclusion}
In this paper, we focus on designing an efficient and asymptotically unbiased algorithm for PAUC. We propose a nonconvex strongly concave minimax instance-wise formulation for $\mathrm{OPAUC}$ and $\mathrm{TPAUC}$. In this way, we incorporate the instances selection into the loss calculation to eliminate the score ranking challenge.
For $\mathrm{OPAUC}$ and $\mathrm{TPAUC}$, we employ an efficient stochastic minimax algorithm that ensures we can find a $\epsilon$-first order saddle point after $O(\epsilon^{-3})$ iterations. Moreover, we present a theoretical analysis of the generalization error of our formulation. Our conclusion may contribute to future work about AUC generalization. Finally, empirical studies over a range of long-tailed benchmark datasets speak to the effectiveness of our proposed algorithm. 

\section{Acknowledgements}
This work was supported in part by the National Key R\&D Program of China under Grant 2018AAA0102000, in part by National Natural Science Foundation of China: U21B2038, 61931008, 61836002, 6212200758 and 61976202, in part by the Fundamental Research Funds for the Central Universities, in part by Youth Innovation Promotion Association CAS, in part by the Strategic Priority Research Program of Chinese Academy of Sciences, Grant No. XDB28000000, in part by the National Postdoctoral Program for Innovative Talents under Grant BX2021298, and in part by mindspore, which is a new AI computing framework \footnote{https://www.mindspore.cn/}.

\bibliography{reference}

\appendix

\section*{Checklist}

\begin{enumerate}

\item For all authors...
\begin{enumerate}
  \item Do the main claims made in the abstract and introduction accurately reflect the paper's contributions and scope?
    \answerYes{}
  \item Did you describe the limitations of your work?
   \answerNA{}
  \item Did you discuss any potential negative societal impacts of your work? \answerNo{} We have not discovered any negative societal impacts in our study.
  \item Have you read the ethics review guidelines and ensured that your paper conforms to them?
    \answerYes{}
\end{enumerate}

\item If you are including theoretical results...
\begin{enumerate}
  \item Did you state the full set of assumptions of all theoretical results?
    \answerYes{}
        \item Did you include complete proofs of all theoretical results?
    \answerYes{}
\end{enumerate}

\item If you ran experiments...
\begin{enumerate}
  \item Did you include the code, data, and instructions needed to reproduce the main experimental results (either in the supplemental material or as a URL)?
    \answerYes{}
  \item Did you specify all the training details (e.g., data splits, hyperparameters, how they were chosen)?
    \answerYes{}
        \item Did you report error bars (e.g., with respect to the random seed after running experiments multiple times)?
    \answerNo{} Our experiments are very time consuming.
        \item Did you include the total amount of compute and the type of resources used (e.g., type of GPUs, internal cluster, or cloud provider)?
    \answerYes{}
\end{enumerate}

\item If you are using existing assets (e.g., code, data, models) or curating/releasing new assets...
\begin{enumerate}
  \item If your work uses existing assets, did you cite the creators?
    \answerYes{}
  \item Did you mention the license of the assets?
    \answerYes{}
  \item Did you include any new assets either in the supplemental material or as a URL?
    \answerYes{}
  \item Did you discuss whether and how consent was obtained from people whose data you're using/curating?
    \answerNo{} Because we use the open source software and data.
  \item Did you discuss whether the data you are using/curating contains personally identifiable information or offensive content?
    \answerNA{}
\end{enumerate}

\item If you used crowdsourcing or conducted research with human subjects...
\begin{enumerate}
  \item Did you include the full text of instructions given to participants and screenshots, if applicable?
    \answerNA{}
  \item Did you describe any potential participant risks, with links to Institutional Review Board (IRB) approvals, if applicable?
    \answerNA{}
  \item Did you include the estimated hourly wage paid to participants and the total amount spent on participant compensation?
    \answerNA{}
\end{enumerate}

\end{enumerate}


\newpage

\begin{appendices}
\settocdepth{section}
\addtocontents{ptc}{\setcounter{tocdepth}{3}}

\LARGE \textbf{List of Appendix}
\normalsize
\startcontents[sections]
\printcontents[sections]{}{1}{}

\newpage

\section{Related work}
\textbf{Deep AUC Optimization.} 
In the past few decades, AUC optimization has already achieved remarkable success in the long-tailed/imbalanced learning task \cite{yang2022auc}. A partial list of the related literature includes \cite{graepel2000large, cortes2003auc, yan2003optimizing, joachims2005support, pepe2000combining, freund2003efficient, rakotomamonjy2004support, ying2016stochastic}. In recent age, many studies focused on AUC optimization with stochastic gradient method. For example, on top of the square surrogate loss, \cite{ying2016stochastic} first proposed a minimax reformulation of the AUC. With a strongly convex regularizer, \cite{natole2018stochastic} improved the convergence rate of the stochastic learning algorithm for AUC to $O(1/T)$. In succession, \cite{liu2019stochastic, guo2020communication, yuan2021compositional} proposed some AUC optimization methods that can be applied to nonconvex deep neural networks.

\textbf{Partial AUC (PAUC) Optimization.} \cite{mcclish1989analyzing} first introduced the concept of PAUC. Earlier studies related to PAUC only paid attention to the simplest linear models. In \cite{pepe2000combining}, the PAUC is first optimized by a distribution-free rank-based method. \cite{wang2011marker} developed a non-parametric estimate of the PAUC, and selected features at each step to build the final classifier. \cite{narasimhan2013structural} develops a cutting plane algorithm to find the most violated constraint instance, decomposing PAUC optimization into subproblems and solving them by an efficient structural SVM-based approach. However, most of the above approaches often fall into the non-differentiable property or intractable optimization problems, posing a significant obstacle to the end-to-end implementation. Using the Implicit Function Theorem, \cite{kumar2021implicit} formulated a rate-constrained optimization problem that modeled the quantile threshold as the output of a function of model parameters. As a milestone study, \cite{yang2021all} simplifies the challenging sample-selected problem involved in PAUC optimization in a bi-level manner and thus facilitates the end-to-end optimization for PAUC of deep learning. Concretely, the inner-level optimization achieves instances selection, and the outer-level optimization minimizes the loss. However, their estimation may suffer from an approximation error with true PAUC. \cite{zhu2022auc,yao2022large} proposed a smooth estimator of PAUC and provided a sound theoretical convergence guarantee of their algorithm. Nevertheless, their algorithm is limited by a slow convergence rate, especially for TPAUC.

\textbf{Generalization Analysis for Partial AUC Optimization.}  \cite{narasimhan2017support} presented the first generalization analysis for $\mathrm{OPAUC}$ and derived a uniform convergence generalization bound. Following their work, a recent study \cite{yang2021all} extended this generalization bound to $\mathrm{TPAUC}$. However, limited by the pair-wise form of AUC, all of above studies require complicated decomposition. Moreover, their generalization analysis only hold for hard-threshold functions and VC-dimension. Based on our instance-wise reformulation, we show that the generalization of partial AUC is as simple as other instance-wise algorithm and can deal with real-valued score functions by Rademacher complexity.

\section{Convergence of the Bias without Regularization}
\label{uniform_convergence}

Take $\mathrm{OPAUC}$ as an example, we will prove that the approximation induced by $r_k$ has a finite convergence rate which vanishes when $\kappa\rightarrow \infty$. For the sake of convenience, we denote the bias as:
\begin{align*}
    & \Delta_\kappa = 
    \min _{f,(a, b)\in[0,1]^2} 
    \max _{\gamma\in\Omega_{\gamma}} 
    \min _{s^{\prime}\in\Omega_{s'}} 
    \underset{\bm{z} \sim S}{\hat{\mathbb{E}}}\Big[
    G_{o p}^{\kappa}\left(f, a, b, \gamma, \boldsymbol{z}, s^{\prime}\right)
    \Big]\\
    &~~~~~~~~~~~~-
    \min _{ f,(a, b)\in[0,1]^2} 
    \max _{\gamma\in\Omega_{\gamma}} 
    \min _{s^{\prime}\in\Omega_{s'}} 
    \underset{\bm{z} \sim S}{\hat{\mathbb{E}}}\Big[
    G_{o p}(f,a, b, \gamma, \boldsymbol{z}, s^{\prime})\Big]
\end{align*}

Specifically, we can also prove the following convergence condition 
holds without the regularization term:
\begin{theorem}
With the assumption that $f(\bm{x}) \in [0,1], \forall \bm{x}$, we have the following convergence result:
\begin{equation}
\begin{aligned}
& \lim _{\kappa \rightarrow \infty}
\bigg|
\min _{f,(a, b)\in[0,1]^2} 
\max _{\gamma\in\Omega_{\gamma}} 
\min _{s^{\prime}\in\Omega_{s'}} 
\underset{\bm{z} \sim S}{\hat{\mathbb{E}}}\Big[
G_{o p}^{\kappa}\left(f, a, b, \gamma, \boldsymbol{z}, s^{\prime}\right)
\Big]\\
&~~~~~~~~~~~~-
\min _{ f,(a, b)\in[0,1]^2} 
\max _{\gamma\in\Omega_{\gamma}} 
\min _{s^{\prime}\in\Omega_{s'}} 
\underset{\bm{z} \sim S}{\hat{\mathbb{E}}}\Big[
G_{o p}(f,a, b, \gamma, \boldsymbol{z}, s^{\prime})\Big]
\bigg| \\
&=0. 
\end{aligned}
\end{equation}
Moreover, we can also obtain a convergence rate:
\begin{equation}
    \Delta_\kappa = O(1/\kappa).
\end{equation}
\end{theorem}

\begin{proof}
Denote:

\begin{equation}
    \begin{aligned}
\Delta_\kappa =&\bigg|
\min _{f,(a, b)\in[0,1]^2} 
\max _{\gamma\in\Omega_{\gamma}} 
\min _{s^{\prime}\in\Omega_{s'}} 
\underset{\bm{z} \sim S}{\hat{\mathbb{E}}}\Big[
G_{o p}^{\kappa}\left(f, a, b, \gamma, \boldsymbol{z}, s^{\prime}\right)
\Big]\\
&~-
\min _{ f,(a, b)\in[0,1]^2} 
\max _{\gamma\in\Omega_{\gamma}} 
\min _{s^{\prime}\in\Omega_{s'}} 
\underset{\bm{z} \sim S}{\hat{\mathbb{E}}}\Big[
G_{o p}(f,a, b, \gamma, \boldsymbol{z}, s^{\prime})\Big]
\bigg| .
\end{aligned}
\end{equation}

First, we have:

\begin{equation}
\begin{aligned}
& \limsup_{\kappa \rightarrow +\infty}\Delta_\kappa \leq & \underbrace{
\limsup _{\kappa \rightarrow +\infty} 
\sup _{f, (a, b)\in[0,1]^2, \gamma\in \Omega_{\gamma}, s^{\prime}\in \Omega_{s'}, 
\boldsymbol{z}\sim\mathcal{D}_{\mathcal{Z}}}
\left|
\frac{\log(1+\exp(\kappa\cdot g))}{\kappa} - 
[g]_+
\right|}_{(a)}.
\end{aligned}
\end{equation}

where $g=(f(\bm{x})-b)^2+2(1+\gamma)f(\bm{x})-s'$ and $[x]_+=\max\{x,0\}$.  Since $g \in[-5,5]$ in the feasible set, we have:

\begin{equation}
(a)\le \limsup _{\kappa \rightarrow +\infty} 
\sup _{x\in[-5,5]}
\left|
\frac{\log(1+\exp(\kappa\cdot x))}{\kappa} - 
[x]_+
\right|.
\end{equation}

Next we prove that 

\begin{equation}
\underset{\kappa\to\infty}{\limsup} 
\underset{x \in [-5,5]}{\sup}
\left[\left|
\frac{\log(1+\exp(\kappa\cdot x))}{\kappa} - 
[x]_+
\right|\right] \le 0.
\end{equation}

For the sake of simplicity, we denote:

\begin{equation}
\ell(x) = \left|
\frac{\log(1+\exp(\kappa\cdot x))}{\kappa} - 
[x]_+
\right|.
\end{equation}

It is easy to see that, when $x<0$, we have:

\begin{equation}
\ell(x)^\prime = \left(\frac{\log(1+\exp(\kappa \cdot x))}{\kappa}\right)^\prime \ge 0.
\end{equation}

When $x>0$, we have:

\begin{equation}
    \ell(x)'=\left(\frac{\log(1+\exp(\kappa \cdot x))}{\kappa}-x\right)^\prime \le 0.
\end{equation}

Hence, the supremum must be attained at $x=0$. We thus have:

\begin{equation}
    (a)\le \limsup_{\kappa\rightarrow +\infty} \frac{\log(2)}{\kappa}= 0 .
\end{equation}

Obviously, the absolute value ensures that:

\begin{equation}
    \liminf_{\kappa\rightarrow +\infty}\Delta_\kappa\ge 0.
\end{equation}

The result follows from the fact:

\begin{equation}
    0 \le \liminf_{\kappa \rightarrow +\infty} \Delta_\kappa \le \limsup_{\kappa \rightarrow +\infty} \Delta_\kappa \le 0.
\end{equation}

Moreover, from the proof above, we also obtain a convergence rate:
\begin{equation}
    \Delta_\kappa = O(1/\kappa).
\end{equation}

\end{proof}

\section{The Constrained Reformulation}

\label{val_constraintval}
In this section, we will prove that the constrained reformulation which is used in the proof of Thm.\ref{thm:step2} and  Thm.\ref{tpthm:step2}. Our proof can be established by Lem.\ref{lem:A}, Lem.\ref{lem:B}, and Thm.\ref{thm:A}. Throughout the proof, we will define:

\begin{equation}
 \begin{aligned}
&a^* = \mathbb{E}_{\bm{x}\sim\mathcal{D}_{\mathcal{P}}}[f(\bm{x})] &:= E_+\\
&b^* = \mathbb{E}_{\bm{x}'\sim\mathcal{D}_{\mathcal{N}}}[f(\bm{x}')|f(\bm{x}')\ge \eta_\beta(f)] &:=E_- \\
&b^* - a^* & := \Delta E\\ 
&\tilde{a}^* = \mathbb{E}_{\bm{x}\sim\mathcal{D}_{\mathcal{P}}}[f(\bm{x})|f(\bm{x})\le \eta_\alpha(f)] &:= \tilde{E}_+\\
&b^* - \tilde{a}^* & := \Delta \tilde{E}\\
& \mathbb{E}_{\bm{x}\sim\mathcal{D}_\mathcal{P}}[(f(\bm{x})-a)^2]& :={E}_a\\ 
& \mathbb{E}_{\bm{x}\sim\mathcal{D}_\mathcal{P}}[(f(\bm{x})-a)^2|
f(\bm{x})\leq\eta_\alpha(f)]& :=\tilde{E}_a\\ 
& \mathbb{E}_{\bm{x}'\sim\mathcal{D}_{\mathcal{N}}}[(f(\bm{x}')-b)^2|f(\bm{x}')\geq\eta_{\beta}(f)]  &:= E_b \\
&\mathbb{E}_{\bm{x}\sim\mathcal{D}_\mathcal{P}}[f(\bm{x})^2|
f(\bm{x})\leq\eta_\alpha(f)] &:=E_{+,2}\\
& \mathbb{E}_{\bm{x}'\sim\mathcal{D}_{\mathcal{N}}}[f(\bm{x}')^2|f(\bm{x}')\geq\eta_{\beta}(f)] &:=E_{-,2}
\end{aligned}   
\end{equation}

\begin{lemma}\label{lem:A}
(The Reformulation for OPAUC) For a fixed scoring function $f$, the following two problems shares the same optimum, given that the scoring function satisfies: $f(\bm{x}) \in [0,1],~ \forall \bm{x}$:
\begin{equation}
    \begin{aligned}
    \boldsymbol{(OP1)} \min_{(a,b)\in[0,1]^2}\max_{\gamma \in [-1,1]}
\mathbb{E}_{\bm{x}\sim\mathcal{D}_\mathcal{P}}[(f(\bm{x})-a)^2] +
\mathbb{E}_{\bm{x}'\sim\mathcal{D}_\mathcal{N}}[(f(\bm{x}')-b)^2|
f(\bm{x}')\geq\eta_\beta(f)] \\
+2\Delta E + 2\gamma \Delta E-\gamma^2
\\
\boldsymbol{(OP2)} \min_{(a,b)\in[0,1]^2}\max_{\gamma \in [b-1,1]}
\mathbb{E}_{\bm{x}\sim\mathcal{D}_\mathcal{P}}[(f(\bm{x})-a)^2] +
\mathbb{E}_{\bm{x}'\sim\mathcal{D}_\mathcal{N}}[(f(\bm{x}')-b)^2|
f(\bm{x}')\geq\eta_\beta(f)] \\
+2\Delta E + 2\gamma \Delta E-\gamma^2
    \end{aligned}
\end{equation}
\end{lemma}

\begin{remark}
$\boldsymbol{(OP1)}$ and $\boldsymbol{(OP2)}$ have the equivalent formulation:
\begin{equation}
    \begin{aligned}
\boldsymbol{(OP1)}\Leftrightarrow \min_{(a,b)\in[0,1]^2}&\max_{\gamma \in [-1,1]} \mathbb{E}_{\bm{z}\sim\mathcal{D}_{\mathcal{Z}}}\Big[\Large[(f(\bm{x})-a)^2- 
2(1+\gamma)f(\bm{x})\Large]y/p-\gamma^2\\
&\Large[(f(\bm{x})-b)^2+2(1+\gamma)f(\bm{x})\Large]\cdot[(1-y)\mathbb{I}_{f(\bm{x})\geq \eta_\beta(f)}]/[(1-p)\beta]\big].
\end{aligned}
\end{equation}
\begin{equation}
    \begin{aligned}
\boldsymbol{(OP2)}\Leftrightarrow \min_{(a,b)\in[0,1]^2}&\max_{\gamma \in [b-1,1]} \mathbb{E}_{\bm{z}\sim\mathcal{D}_{\mathcal{Z}}}\Big[\Large[(f(\bm{x})-a)^2- 
2(1+\gamma)f(\bm{x})\Large]y/p-\gamma^2\\
&\Large[(f(\bm{x})-b)^2+2(1+\gamma)f(\bm{x})\Large]\cdot[(1-y)\mathbb{I}_{f(\bm{x})\geq \eta_\beta(f)}]/[(1-p)\beta]\big].
\end{aligned}
\end{equation}
\end{remark}

\begin{proof}
From the proof of our main paper, we know that $(OP1)$ has a closed-form minimum:
\begin{equation}
     E_{a^*} + {E}_{b^*} + (\Delta E)^2 + 2 \Delta E.
\end{equation}
Hence, we only need to prove that $(OP2)$ has the same minimum solution. By expanding $(OP2)$, we have:
\begin{equation}
    \begin{aligned}
\min_{(a,b)\in[0,1]^2}\max_{\gamma\in[b-1,1]} 
\mathbb{E}_{\bm{z}\sim\mathcal{D}_{\mathcal{Z}}}[
F_{op}(f,a,b,\gamma,\eta_\beta(f),\bm{z})] &= \\ 
2\Delta E + \min_{a \in [0,1]}E_a  + \min_{b\in[0,1]}\max_{\gamma \in [b-1,1]} F_0
\end{aligned}
\end{equation}
where 
\begin{equation}
    F_0:= E_b + 2\gamma\Delta E - \gamma^2
\end{equation}
Obviously since $a$ is decoupled with $b,\gamma$, we have:
\begin{equation}
    \min_{a \in [0,1]} E_a  = E_{a^*}
\end{equation}
Now, we solve the minimax problem of $F_0$. For any fixed feasible $b$, the inner max problem is a truncated quadratic programming, which has a unique and closed-form solution. Hence, we first solve the inner maximization problem for fixed $b$, and then represent the minimax problem as a minimization problem for $b$. Specifically, we have:
\begin{equation}
    \left(\max_{\gamma\in[b-1,1]} 2\gamma \Delta E-\gamma^2\right) = 
\begin{cases}
(\Delta E)^2,  & \Delta E \ge b-1 \\
2(b-1)\Delta E - (b-1)^2, & \text{otherwise}
\end{cases} 
\end{equation}

Thus, we have:
\begin{equation}
    \min_{b\in[0,1]} \max_{\gamma\in[b-1,1]} F_0=
\min_{b \in [0,1]} F_1
\end{equation}
where
\begin{equation}
    F_1 = \begin{cases}
&F_{1,0}(b) := E_b + (\Delta E)^2,  b-1 \le \Delta E \\ 
& F_{1,1}(b) := E_{-,2}- 2bE_- + 2b-1 + 2(b-1)\Delta E,  \text{otherwise}
\end{cases}
\end{equation}
It is easy to see that both cases of $F_1$ are convex functions w.r.t $b$. So, we can find the global minimum by comparing the minimum of $F_{1,0}$ and $F_{1,1}$. 
\begin{itemize}
\item CASE 1: $\Delta E \ge b-1$. 
    It is easy to see that $b^* = E_- \in (-\infty, 1+\Delta E]$, by taking the derivative to zero, we have, the optimum value is obtained at $b= E_-$ for $F_{1,0}$.

\item CASE 2: $\Delta E \leq b-1$. 
    Again by taking the derivative, we have:
    \begin{equation}
        F_{1,1}(b)' = -2E_- + 2 + 2 \Delta E = 2-2E_+ \ge 0 
    \end{equation}
    We must have:
    \begin{equation}
        \inf_{b \geq 1 + \Delta E} F_{1,1}(b) ~\ge~ F_{1,1}(1+\Delta E) ~=~ F_{1,0}(1+\Delta E) ~\geq~ F_{1,0}(E_-) ~=~ F_{1,0}(b^*)
    \end{equation}
    
\item Putting all together
Hence the global minimum of $F_1$ is obtained at $b^*$ with:
\begin{equation}
    F_1(b^*) = F_{1,0}(b^*) = E_{b^*}+ (\Delta E)^2
\end{equation}
\end{itemize}

Hence, we have $(OP2)$ has the minimum value:
\begin{equation}
     E_{a^*} + E_{b^*} + (\Delta E)^2 + 2 \Delta E
\end{equation}
\end{proof}

Now, we use a similar trick to prove the result for TPAUC:

\begin{lemma}\label{lem:B}
    (The Reformulation for TPAUC) For a fixed scoring function $f$, the following two problems shares the same optimum, given that the scoring function satisfies: $f(\bm{x}) \in [0,1],~ \forall \bm{x}$:
\begin{equation}
    \begin{aligned}
    \boldsymbol{(OP3)} \min_{f,(a,b)\in[0,1]^2}\max_{\gamma \in [-1,1]}
&\mathbb{E}_{\bm{x}\sim\mathcal{D}_\mathcal{P}}[(f(\bm{x})-a)^2|
f(\bm{x})\leq\eta_\alpha(f)]\\+
&\mathbb{E}_{\bm{x}'\sim\mathcal{D}_\mathcal{N}}[(f(\bm{x}')-b)^2|
f(\bm{x}')\geq\eta_\beta(f)]\\
+&2\Delta \tilde{E} + 2\gamma \Delta \tilde{E}-\gamma^2.
\\
\boldsymbol{(OP4)} \min_{f,(a,b)\in[0,1]^2}\max_{\gamma \in [\max\{-a,b-1\},1]}
&\mathbb{E}_{\bm{x}\sim\mathcal{D}_\mathcal{P}}[(f(\bm{x})-a)^2|
f(\bm{x})\leq\eta_\alpha(f)]\\+
&\mathbb{E}_{\bm{x}'\sim\mathcal{D}_\mathcal{N}}[(f(\bm{x}')-b)^2|
f(\bm{x}')\geq\eta_\beta(f)]\\
 +&2\Delta \tilde{E} + 2\gamma \Delta \tilde{E}-\gamma^2.
    \end{aligned}
\end{equation}
\end{lemma}

\begin{remark}
$\boldsymbol{(OP3)}$ and $\boldsymbol{(OP4)}$ have the equivalent formulation:
\begin{equation}
   \begin{aligned}
\boldsymbol{(OP3)} \Leftrightarrow
\min_{(a,b)\in[0,1]^2}&\max_{\gamma \in [-1, 1]} \mathbb{E}_{\bm{z}\sim \mathcal{D}_\mathcal{Z}}\Big[\Large[(f(\bm{x})-a)^2- 
2(1+\gamma)f(\bm{x})\Large]\cdot[y\mathbb{I}_{f(\bm{x})\leq \eta_\alpha(f)}]/p-\gamma^2\\
&\Large[(f(\bm{x})-b)^2+2(1+\gamma)f(\bm{x})\Large]\cdot[(1-y)\mathbb{I}_{f(\bm{x})\geq \eta_\beta(f)}]/[(1-p)\beta]\Big].
\end{aligned} 
\end{equation}
\begin{equation}
    \begin{aligned}
\boldsymbol{(OP4)} \Leftrightarrow
\min_{(a,b)\in[0,1]^2}&\max_{\gamma \in [\max\{-a,b-1\} 1]} \mathbb{E}_{\bm{z}\sim \mathcal{D}_\mathcal{Z}}\Big[\Large[(f(\bm{x})-a)^2- 
2(1+\gamma)f(\bm{x})\Large]\cdot[y\mathbb{I}_{f(\bm{x})\leq \eta_\alpha(f)}]/p\\
&\Large[(f(\bm{x})-b)^2+2(1+\gamma)f(\bm{x})\Large]\cdot[(1-y)\mathbb{I}_{f(\bm{x})\geq \eta_\beta(f)}]/[(1-p)\beta]-\gamma^2\Big].
\end{aligned}
\end{equation}
\end{remark}

\begin{proof}
Again, $(OP3)$ has the minimum value:

\begin{equation}
    \tilde{E}_{\tilde{a}^*} + E_{b^*} + (\Delta \tilde{E})^2 + 2 \Delta \tilde{E} 
\end{equation}

We proof that $(OP4)$ ends up with the minimum value. By expanding $(OP4)$, we have:
\begin{equation}
    \begin{aligned}
(OP4)= 2\Delta \tilde{E} + \min_{(a,b)\in[0,1]^2}\max_{\gamma \in [\max\{-a,b-1\},1]} F_3
\end{aligned}
\end{equation}
where 
\begin{equation}
    F_3:=  \tilde{E}_a+ E_{b}
 +2\Delta \tilde{E} + 2\gamma \Delta \tilde{E}-\gamma^2
\end{equation}

For any fixed feasible $a,b$, the inner max problem is a truncated quadratic programming, which has a unique and closed-form solution. Specifically, define $c = \max\{-a,b-1\}$, we have:
\begin{equation}
   \left(\max_{\gamma \in [c,1]} 2\gamma \Delta \tilde{E}-\gamma^2\right) = 
\begin{cases}
(\Delta \tilde{E})^2,  & \Delta \tilde{E} \ge c \\
2c\Delta \tilde{E} - c^2, & \text{otherwise}
\end{cases}  
\end{equation}
Thus, we have:
\begin{equation}
    \min_{(a,b)\in[0,1]^2} \max_{\gamma\in[c,1]} F_3=
\min_{(a,b) \in [0,1]} F_4
\end{equation}
where
\begin{equation}
    F_4 = \begin{cases}
&F_{4,0}(a,b) := \tilde{E}_a + E_b + (\Delta \tilde{E})^2,  c \le \Delta \tilde{E} \\ 
& F_{4,1}(a,b) := \tilde{E}_a + E_{-,2} -2b E_- + 2(b-1) \Delta \tilde{E} + 2b -1 ,  b-1 \ge \Delta \tilde{E}, -a \le b-1 \\ 
& F_{4,2}(a,b) := E_b + E_{+,2}  -2a \tilde{E}_+ -2a \Delta \tilde{E},   -a \ge \Delta \tilde{E}, b-1 \le -a 
\end{cases}
\end{equation}
It is easy to see that both cases of $F_1$ are convex functions w.r.t $b$. So, we can find the global minimum by comparing the minimum of $F_{1,0}$ and $F_{1,1}$. 
\begin{itemize}

    \item CASE 1: $\Delta \tilde{E} \ge \max\{-a,b-1\}$. 
 
It is easy to check that when $a = \tilde{E}_+, b = E_- $, we have $-a \le \Delta \tilde{E}$ and $b-1 \le \Delta \tilde{E}$. It is easy to see that $a,b$ are decoupled in the expression of $F_{4,0}(a,b)$. By setting:
\begin{equation}
     \begin{aligned}
\frac{\partial F_{4,0}(a,b)}{\partial a} = 0, \\ 
\frac{\partial F_{4,0}(a,b)}{\partial b} = 0
\end{aligned}
\end{equation}
We know that the minimum solution is attained at $a= \tilde{a}^*$, $b= b^*$. Then the minimum value of $F_{4,0}(a,b)$ at this range becomes:
\begin{equation}
    \tilde{E}_{\tilde{a}^*} + E_{b^*} + (\Delta \tilde{E})^2
\end{equation}
Moreover, we will also use the fact that $E_{\tilde{a}^*}$ and $E_{b^*}$ are also the global minimum for $E_a$ and $E_b$, respectively. 

    \item CASE 2: $b-1 \ge \Delta \tilde{E},~ -a \le b-1$.

It is easy to see that $E_a \ge E_{\tilde{a}^*}$ in this case.  According to the same derivation as in Lem.\ref{lem:A} CASE 2, we have:
\begin{equation}
    E_{-,2} -2b E_- + 2(b-1) \Delta \tilde{E} + 2b -1 \ge E_{b^*} + (\Delta \tilde{E})^2
\end{equation}
holds when $b -1 \geq \Delta \tilde{E}$. Recall that CASE 2 is include in the condition $b -1 \geq\Delta \tilde{E}$. So, under the condition of CASE 2:
\begin{equation}
    F_{4,1}(a,b) \ge \tilde{E}_{\tilde{a}^*} + E_{b^*} + (\Delta \tilde{E})^2
\end{equation}

    \item CASE 3: $-a\geq \Delta \tilde{E}, b-1\leq -a$

In this case, we have $E_b \ge E_{b^*}$. It remains to check:
\begin{equation}
    g(a) =  -2a \tilde{E}_+ -2a \Delta \tilde{E}
\end{equation}
By taking derivative, we have:
\begin{equation}
    g'(a) =  -2\tilde{E}_+ - 2\Delta \tilde{E} = -2\tilde{E}_- \le 0. 
\end{equation}
Similar as the proof of CASE 2, when $-a \ge \Delta \tilde{E}$, we have:
\begin{equation}
    g(a) \ge \tilde{E}_{\tilde{a}^*} + (\Delta \tilde{E})^2
\end{equation}
 and thus
\begin{equation}
    F_{4,2}(a,b) \ge \tilde{E}_{\tilde{a}^*} + E_{b^*} + (\Delta \tilde{E})^2
\end{equation}
holds. Since the condition of CASE 3 is included in the set $-a \ge \Delta \tilde{E}$:
\begin{equation}
    F_{4,2}(a,b) \ge \tilde{E}_{\tilde{a}^*} + E_{b^*} + (\Delta \tilde{E})^2
\end{equation}
holds under the condition of CASE 3.
    \item Putting altogether:
The minimum value of $(OP4)$ reads:
\begin{equation}
    \tilde{E}_{\tilde{a}^*} + E_{b^*} + (\Delta \tilde{E})^2 + 2\Delta \tilde{E}
\end{equation}
which is the same as $(OP3)$.
\end{itemize}
\end{proof}

Finally, since for each fixed $f$ $(OP3) = (OP4)$, and $(OP1) = (OP2)$ . We can then claim the following theorem:

\begin{theorem}\label{thm:A}
\label{Constrainted_Reformulation}
(Constrainted Reformulation) 
\begin{equation}
    \min_f (OP1) = \min_f (OP2), ~~ \min_f (OP3) = \min_f (OP4)
\end{equation}
\end{theorem}

\begin{remark}
Since the calculation is irelevant to the definition of the expectation, the replace the population-level expectation with the empirical expectation over the training data.
\end{remark}

\begin{remark}
By applying Theorem 1, we can get the reformulation result in Theorem 2
\item  for $\mathrm{OPAUC}$
\begin{equation}
    \min_{f,(a,b)\in[0,1]^2} \max_{\gamma\in [b-1,1]}\min_{s'\in\Omega_{s'}}\mathbb{E}_{\bm{z}\sim \mathcal{D}_\mathcal{Z}} [G_{op}(f,a,b,\gamma,\bm{z},s')]
\end{equation}
where
\begin{equation}
    \begin{aligned}
    G_{op}(f,a,b,\gamma,\bm{z},s')&=[(f(\bm{x})-a)^2- 
    2(1+\gamma)f(\bm{x})]y/p-\gamma^2\\
    & +
    \left(\beta s' +\left[(f(\bm{x})-b)^2+2(1+\gamma) f(\bm{x})-s'\right]_+\right)(1-y)/[\beta (1-p)].
    \end{aligned}
\end{equation}
\item for $\mathrm{TPAUC}$
\begin{equation}
    \min_{f,(a,b)\in[0,1]^2 } \max_{\gamma\in [\max\{-a,b-1\}, 1]}\min_{s\in\Omega_{s},s'\in\Omega_{s'}}\mathbb{E}_{\bm{z}\sim \mathcal{D}_\mathcal{Z}} [G_{tp}(f,a,b,\gamma,\bm{z},s,s')]
\end{equation}
where
\begin{equation}
    \begin{aligned}
    G_{tp}(f,a,b,\gamma,\bm{z},s,s')&=\left(\alpha s + r_{\kappa}\left((f(\bm{x})-a)^2- 
    2(1+\gamma)f(\bm{x})-s\right)\right)y/(\alpha p) -\gamma^2\\
    & +
    \left(\beta s' +r_{\kappa}\left((f(\bm{x})-b)^2+2(1+\gamma) f(\bm{x})-s'\right)\right)(1-y)/[\beta (1-p)].
    \end{aligned}
\end{equation}
\end{remark}

\section{Reformulation for TPAUC}
\label{sec:tpauc_reformulation}
According to Eq.\eqref{TPAUCM}, given a surrogate loss $\ell$ and the finite dataset $S$, maximizing $\mathrm{TPAUC}$ and $\hat{\mathrm{AUC}}_{\alpha,\beta}(f, S)$ is equivalent to solving the following problems, respectively:
\begin{equation}
        \underset{f}{\min}~  \mathcal{R}_{\alpha,\beta}(f)= \mathbb{E}_{\bm{x} \sim \mathcal{D}_\mathcal{P}, \bm{x}'\sim \mathcal{D}_\mathcal{N}} \left[\mathbb{I}_{f(\bm{x}) \le \eta_\alpha(f)} \cdot \mathbb{I}_{f(\bm{x}') \ge \eta_\beta(f)} \cdot \ell(f(\bm{x})- f(\bm{x}'))\right],
        \label{TPAUCO}
    \end{equation}
\begin{equation}
\begin{aligned}
    \underset{f}{\min} \ \hat{\mathcal{R}}_{\alpha,\beta}(f, S)= \sum_{i=1}^{n_+^\alpha} \sum_{j=1}^{n_-^\beta}\frac{\ell{\left(f(\bm{x}_{[i]})- f(\bm{x}'_{[j]})\right)}}{n_+^\alpha n_-^\beta}.
    \label{TPAUCO}
\end{aligned}
\end{equation}
Similar to $\mathrm{OPAUC}$, we have the following theorem shows an instance-wise reformulation of the $\tp$ optimization problem:
\begin{theorem}
\label{theorem:7}
Assuming that $f(\bm{x})\in[0,1]$, $\forall \bm{x}\in\mathcal{X}$, $F_{tp}(f,a,b,\gamma, t, t', \bm{z})$ is defined as:
\begin{equation}
    \begin{aligned}
F_{tp}(&f,a,b,\gamma,t, t', \bm{z})=(f(\bm{x})-a)^2y\mathbb{I}_{f(\bm{x})\leq t}/(\alpha p)+
(f(\bm{x})-b)^2(1-y)\mathbb{I}_{f(\bm{x}')\geq t'}/[\beta (1-p)]\\
&+2(1+\gamma)f(\bm{x})(1-y)\mathbb{I}_{f(\bm{x}')\geq t'}/[\beta (1-p)] - 
2(1+\gamma)f(\bm{x})y/p\mathbb{I}_{f(\bm{x})\leq t}/(\alpha p) -\gamma^2,
\end{aligned}
\label{eq:fop}
\end{equation}

 where $y=1$ for positive instances, $y=0$ for negative instances and we have the following conclusions:
\begin{enumerate}[leftmargin=20pt]
    \item[(a)] (\textbf{Population Version}.) We have:
    \begin{equation}
        \underset{f}{\min} \ {\mathcal{R}}_{\alpha,\beta}(f) \Leftrightarrow \underset{\cmin}{\min}\ \underset{\gamma\in[-1,1]}{\max} \ 
    \colblue{\underset{\bm{z}\sim \mathcal{D}_\mathcal{Z}}{\mathbb{E}}}
    \left[F_{tp}(f,a,b,\gamma, \colblue{\eta_\alpha(f)},\colblue{\eta_\beta(f)},\bm{z})\right],
    \label{eq:minmaxtpauc1}
    \end{equation}
    where $\colblue{\eta_\alpha(f)}=\arg\min_{\colblue{\eta_{\alpha}} \in\mathbb{R}}\left[\mathbb{E}_{\bm{x}\sim \mathcal{D}_{\mathcal{P}}}[\mathbb{I}_{f(\bm{x})\leq\ \colblue{\eta_{\alpha}}}]=\alpha\right]$ and $\colblue{\eta_\beta(f)}=\arg\min_{\colblue{\eta_{\beta}} \in\mathbb{R}}\left[\mathbb{E}_{\bm{x}'\sim \mathcal{D}_{\mathcal{N}}}[\mathbb{I}_{f(\bm{x}')\geq\ \colblue{\eta_{\beta}}}]=\beta\right]$.
    \item[(b)] (\textbf{Empirical Version}.) Moreover, given a training dataset $S$ with sample size $n$, denote:
    \begin{equation*}
    \colbit{\ehat_{z \sim S}}[F_{tp}(f,a,b,\gamma,\colbit{\hat{\eta}_{\alpha}(f)}, \colbit{\hat{\eta}_\beta(f)}, \bm{z})] = \frac{1}{n}\sum_{i=1}^n F_{tp}(f,a,b,\gamma,\colbit{\hat{\eta}_{\alpha}(f)}, {\colbit{\hat{\eta}_\beta(f)}}, \bm{z}),
    \end{equation*}
    where $\colbit{\hat{\eta}_{\alpha}(f)}$ and $\colbit{\hat{\eta}_\beta(f)}$ are the empirical quantile of the positive and negative instances in $S$, respectively. We have:
    \begin{equation}
        \underset{f}{\min} \ \hat{\mathcal{R}}_{\alpha,\beta}(f, S) \Leftrightarrow \underset{\cmin}{\min}\ \underset{\gamma\in[-1,1]}{\max} \ 
    \colbit{\underset{\bm{z}\sim S}{\ehat}}
    \left[F_{tp}(f,a,b,\gamma, \colbit{\hat{\eta}_{\alpha}(f)}, \colbit{\hat{\eta}_\beta(f)}, \bm{z} ) \right],
    \label{eq:minmaxtpauc2}
    \end{equation} 
\end{enumerate}
\end{theorem}

Thm.\ref{theorem:7} provides a support to convert the pair-wise loss into instance-wise loss for $\mathrm{TPAUC}$. Actually, for $\mathcal{R}_{\alpha, \beta}(f)$, we can just reformulate it as an  Average Top-$k$ (ATk) loss. Denote $\ell_+(\bm{x})=(f(\bm{x})-a)^2-2(1+\gamma)f(\bm{x})$ and $\ell_-(\bm{x}')=(f(\bm{x}')-b)^2+2(1+\gamma)f(\bm{x}')$. In the proof of the next theorem, we will show that $\ell_+(\bm{x})$ is an decreasing function and $\ell_-(\bm{x}')$ is an increasing function w.r.t. $f(\bm{x})$ and $f(\bm{x}')$, namely:
\begin{equation}
    \mathbb{E}_{\bm{x}\sim\mathcal{D}_\mathcal{P}}[\mathbb{I}_{f(\bm{x})\leq\eta_{\alpha}(f)}\cdot\ell_+(\bm{x})]= \min_s \frac{1}{\alpha} \cdot \mathbb{E}_{\bm{x}\sim\mathcal{D}_\mathcal{P}} [\alpha s + [\ell_+(\bm{x})-s]_+],
\end{equation}
\begin{equation}
    \mathbb{E}_{\bm{x}'\sim\mathcal{D}_\mathcal{N}}[\mathbb{I}_{f(\bm{x}')\geq\eta_{\beta}(f)}\cdot\ell_-(\bm{x}')]= \min_{s'}  \frac{1}{\beta} \cdot \mathbb{E}_{\bm{x}'\sim\mathcal{D}_\mathcal{N}} [\beta s' + [\ell_-(\bm{x}')-s']_+],
\end{equation}
The similar result holds for $\hat{\mathcal{R}}_{\alpha, \beta}(f, S)$. Then, we can reach to Thm.\ref{tpthm:step2}
\begin{theorem}\label{tpthm:step2}
Assuming that $f(\bm{x})\in[0,1]$, for all $\bm{x}\in\mathcal{X}$, we have the equivalent optimization for $\mathrm{TPAUC}$: 
\begin{equation}
\begin{aligned}
    \underset{\cmin}{\min}\ \underset{\gamma\in[-1,1]}{\max} \ 
\colblue{\underset{\bm{z}\sim \mathcal{D}_\mathcal{Z}}{\mathbb{E}}}[F_{tp}(f,a,b,\gamma,\colblue{\efa}, \colblue{\efb}, \bm{z})]
\\ \Leftrightarrow \underset{\cmin}{\min}\ 
\underset{\cmax }{\max} 
\ \underset{s\in\Omega_{s},s'\in\Omega_{s'}}{\min}
\ \colblue{\underset{\bm{z}\sim \mathcal{D}_\mathcal{Z}}{\mathbb{E}}}[G_{tp}(f,a,b,\gamma,\bm{z},s,s')],
\label{minmaxmin_tp}
\end{aligned}
\end{equation}
\begin{equation}
    \begin{aligned}
        \underset{\cmin}{\min}\ \underset{\cmax}{\max} \ 
    \colbit{\underset{\bm{z}\sim S}{\ehat}}[F_{tp}(f,a,b,\gamma, \colbit{\hefa}, \colbit{\hefb},\bm{z})]
    \\ \Leftrightarrow \underset{\cmin}{\min}\ 
    \underset{\gamma\in\Omega_{\gamma} }{\max} 
    \ \underset{s\in\Omega_{s},s'\in\Omega_{s'}}{\min}
    \ \colbit{\underset{\bm{z}\sim S}{\ehat}}[G_{tp}(f,a,b,\gamma,\bm{z},s,s')],
    \label{minmaxmin_tp}
    \end{aligned}
    \end{equation}

where $\Omega_{\gamma}=[\max\{b-1, -a\},1]$, $\Omega_{s}=[-4, 1]$, $\Omega_{s'}=[0,5]$ and

\begin{equation}
    \begin{aligned}
G_{tp}(f,a,b,\gamma,&\bm{z},s, s')=\left(\alpha s + \left[(f(\bm{x})-a)^2- 
2(1+\gamma)f(\bm{x})-s\right]_+\right)y/(\alpha p)\\
& +
\left(\beta s' +\left[(f(\bm{x})-b)^2+2(1+\gamma) f(\bm{x})-s'\right]_+\right)(1-y)/[\beta (1-p)] -\gamma^2.
\end{aligned}
\label{OPAUC_IB}
\end{equation}

\label{theorem:8}
\end{theorem}
Similar to $\mathrm{OPAUC}$, we can get a regularized non-convex strongly-concave $\mathrm{TPAUC}$ optimization problem:
\begin{equation}
    \underset{\cmin}{\min}\ \underset{\gamma\in\Omega_{\gamma}}{\max} \min_{s \in \Omega_{s}, s' \in \Omega_{s'}}
    \ \colblue{\underset{\bm{z}\sim \mathcal{D}_\mathcal{Z}}{\mathbb{E}}}[G_{tp}^{{\colblue{\kappa}},\colbit{\omega}}]\Leftrightarrow \underset{f,(a,b)\in[0,1]^2,s\in\Omega_{s},s'\in\Omega_{s'}}{\min}\ \underset{\cmax}{\max} 
\ \colblue{\underset{\bm{z}\sim \mathcal{D}_\mathcal{Z}}{\mathbb{E}}}[G_{tp}^{{\colblue{\kappa}},\colbit{\omega}}],
\label{minmax_tp}
\end{equation}

\begin{equation}
    \underset{\cmin}{\min}\ \underset{\gamma\in\Omega_{\gamma}}{\max} \min_{s \in \Omega_{s}, s'\in \Omega_{s'}}
    \ \colbit{\underset{\bm{z}\sim S}{\ehat}}[{G}_{tp}^{{\colblue{\kappa}},\colbit{\omega}}]\Leftrightarrow \underset{f,(a,b)\in[0,1]^2,s\in\Omega_{s},s'\in\Omega_{s'}}{\min}\ \underset{\cmax}{\max} 
\ \colbit{\underset{\bm{z}\sim S}{\ehat}}[{G}_{tp}^{{\colblue{\kappa}},\colbit{\omega}}],
\end{equation}
where $G_{tp}^{{\colblue{\kappa}},\colbit{\omega}}=G_{tp}^{{\colblue{\kappa}},\colbit{\omega}}(f,a,b,\gamma,\bm{z},s,s')$. 

\section{Experiment Details}
\label{section:experiment_details}
\subsection{Dataset}
\textbf{Binary CIFAR-10-Long-Tail Dataset.} 
The CIFAR-10 dataset contains 60,000 images, each of 32 * 32 shapes, grouped into 10 classes of 6,000 images. The training and test sets contain 50,000 and 10,000 images, respectively. We construct the binary datasets by selecting one super category as positive class and the other categories as negative class. We generate three binary subsets composed of positive categories, including 1) birds, 2) automobiles, and 3) cats.

\textbf{Binary CIFAR-100-Long-Tail Dataset.} The original CIFAR-100 dataset has 100 classes, with each containing 600 images. In the CIFAR-100, there are 100 classes divided into 20 superclasses. By selecting a superclass as a positive class example each time, we create CIFAR-100-LT by following the same process as CIFAR-10-LT. The positive superclasses consist of 1) fruits and vegetables, 2) insects, and 3) large omnivores and herbivores, respectively.  

\textbf{Binary Tiny-ImageNet-200-Long-Tail Dataset.} There are 100,000 256 * 256 color pictures in the Tiny-ImageNet-200 dataset, divided into 200 categories, with 500 pictures per category. We chose three positive superclasses to create binary subsets: 1) dogs, 2) birds, and 3) vehicles. 

All data are divided into training, validation, and test sets with proportion 0.7 : 0.15 : 0.15. In each class, sample sizes decay exponentially, and the ratio of sample sizes of the least frequent to the most frequent class is set to 0.01.

\begin{table}[h]
\caption{Details of dataset.}
\renewcommand\arraystretch{1.2}
\footnotesize
\begin{tabular}{lllll}
   \toprule
   Dataset & Pos. Class ID & Pos. Class Name & \# Pos& \#Neg \\
   \midrule
   CIFAR-10-LT-1 & 2 & birds & 1,508 & 8,907 \\
   CIFAR-10-LT-2 & 1 & automobiles & 2,517 & 7,898 \\
   CIFAR-10-LT-3 & 3 & birds & 904 & 9,511 \\
   \midrule
   CIFAR-100-LT-1 & 6,7,14,18,24 & insects & 1,928 & 13,218 \\
   CIFAR-100-LT-2 & 0,51,53,57,83 & fruits and vegatables & 885 & 14,261 \\
   CIFAR-100-LT-3 & 15,19,21,32,38 & large omnivores herbivores  & 1,172 & 13,974 \\
   \midrule
   Tiny-ImageNet-200-LT-1 & 24,25,26,27,28,29 & dogs & 2,100 & 67,900 \\
   Tiny-ImageNet-200-LT-2 & 11,20,21,22 & birds & 1,400 & 68,600 \\
   Tiny-ImageNet-200-LT-3 & \tabincell{l}{70,81,94,107,111,116,121,\\133,145,153,164,166} & vehicles & 4,200 & 65,800 \\
   \bottomrule
\end{tabular}
\label{tab:dataset}
\end{table}

\subsection{Implementation Details}
All experiments are conducted on an Ubuntu 16.04.1 server equipped with an Intel(R) Xeon(R) Silver 4110 CPU and four RTX 3090 GPUs, and all codes are developed in \texttt{Python 3.8} and \texttt{pytorch 1.8.2} environment. We use the ResNet-18 as a backbone. With a \texttt{Sigmoid} function, the output is scaled into $[0,1]$. The batch size is set as $1024$. Following  the previous studies \cite{yang2021all, yuan2021large, guo2020communication}, we warm up all algorithms for $10$ epochs with CELoss to avoid overfitting. All models are trained using \texttt{SGD} as the basic optimizer. 

\subsection{Competitors}
We compare our algorithm with 6 baselines: the approximation algorithms of PAUC, which are denoted as AUC-poly \cite{yang2021all} (poly calibrated weighting function) and AUC-exp \cite{yang2021all} (exp calibrated weighting function); the DRO formulation of PAUC, which are denoted as SOPA \cite{zhu2022auc} (exact estimator) and SOPA-S \cite{zhu2022auc} (soft estimator); the large-scale PAUC optimization method, which is denoted AGD-SBCD \cite{yao2022large}; the naive mini-batch version of empirical partial AUC optimization, which is denoted as MB \cite{kar2014online}; the AUC minimax \cite{ying2016stochastic} optimization, which is denoted as AUC-M; the binary CELoss; and our method, which is denoted as PAUCI.

\subsection{Parameter Tuning}
The learning rate of all methods is tuned in $[10^{-2},10^{-5}]$. Weight decay
is tuned in $[10^{-3},10^{-5}]$. Specifically, $E_{k}$ for AUC-poly and AUC-exp is searched in $\{3, 5, 8, 10, 12, 15, 18, 20\}$. For AUC-poly, $\gamma$ is searched in $\{0.03, 0.05, 0.08, 0.1, 1, 3, 5\}$.
For AUC-exp, $\gamma$ is searched in $\{8, 10, 15, 20, 25, 30\}$. For SOPA-S, we tune the KL-regularization parameter $\lambda$ in $\{0.1, 1.0, 10\}$, and we fix $\beta_0 = \beta_1 = 0.9$. For PAUCI, $k$ is tuned in $[1, 10]$, $\nu$, $\lambda$,$c_1$, $c_2$ are tuned in $[0, 1]$, $m$ is tuned in $[10, 100]$, $\kappa$ is tuned in $[2, 6]$ and $\omega$ is tuned in $[0, 4]$.

\subsection{Per-iteration Acceleration}
\label{section:convergence}
\begin{figure}[!t]
	\centering
	
		\begin{minipage}{0.32\linewidth}
		\centering
		\includegraphics[width=\linewidth]{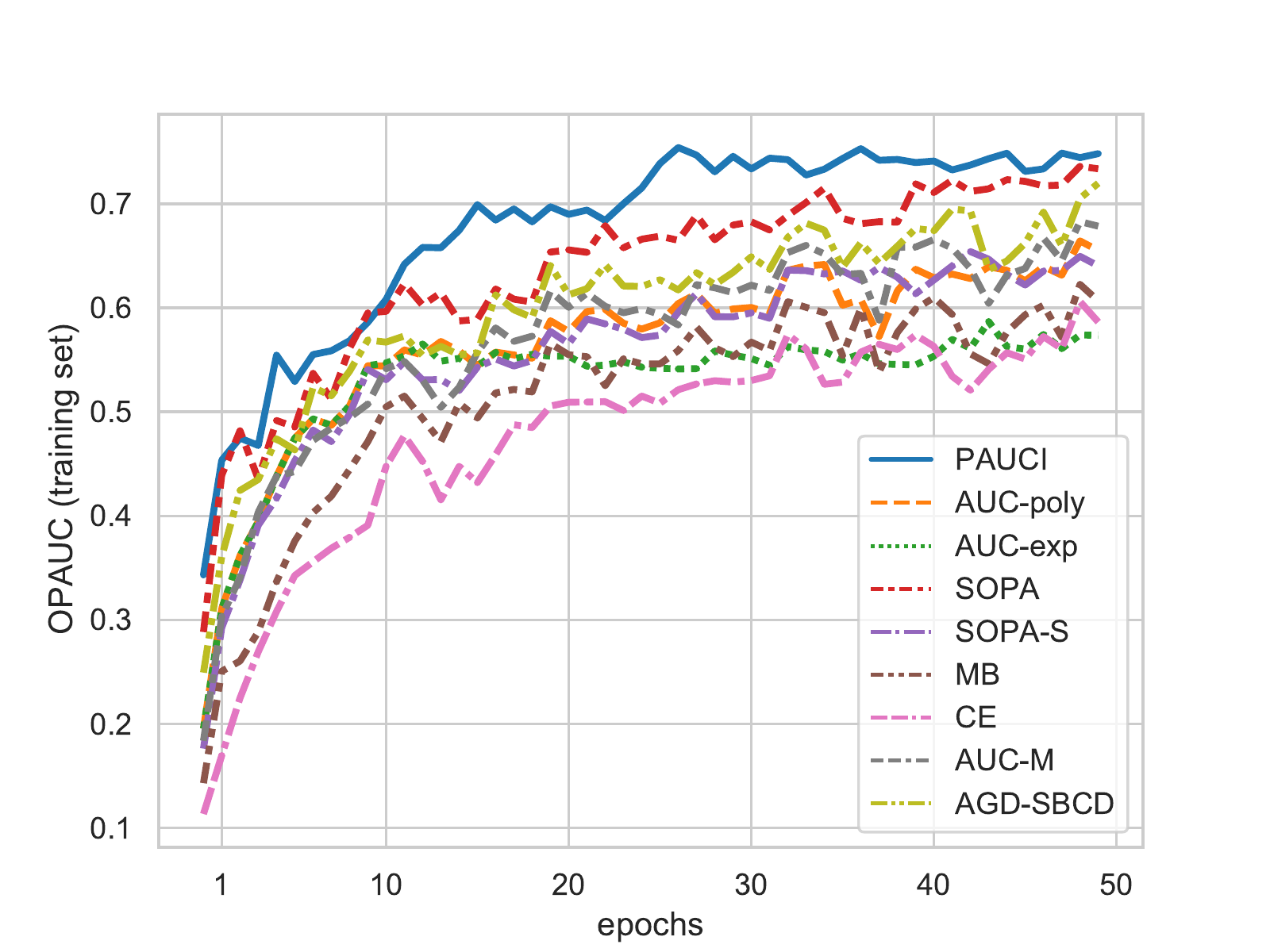}
		\caption*{(a) CIFAR-10-LT-1}
	\end{minipage}
	\begin{minipage}{0.32\linewidth}
		\centering
		\includegraphics[width=\linewidth]{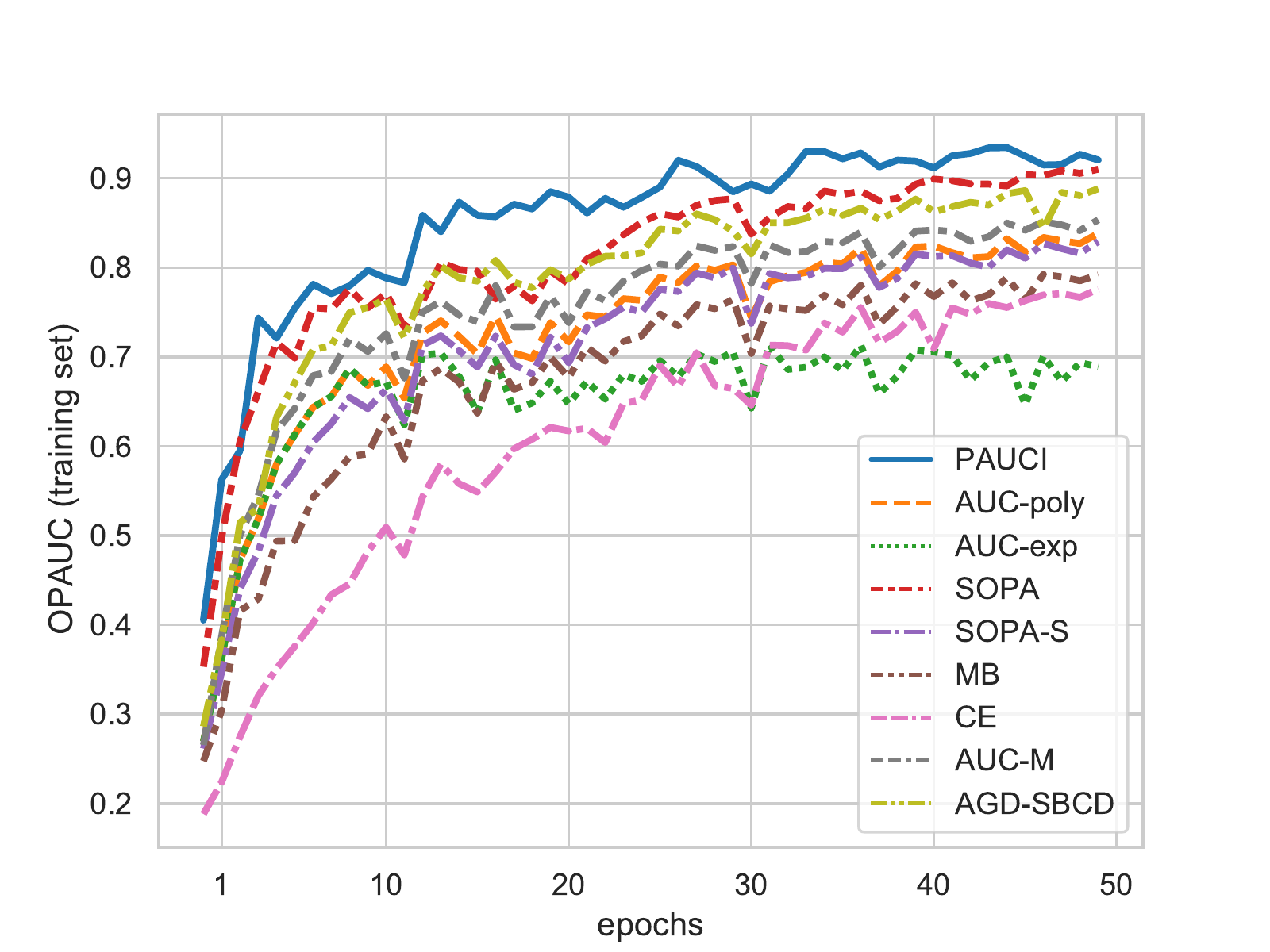}
		\caption*{(b) CIFAR-10-LT-2}
	\end{minipage}
	\begin{minipage}{0.32\linewidth}
		\centering
		\includegraphics[width=\linewidth]{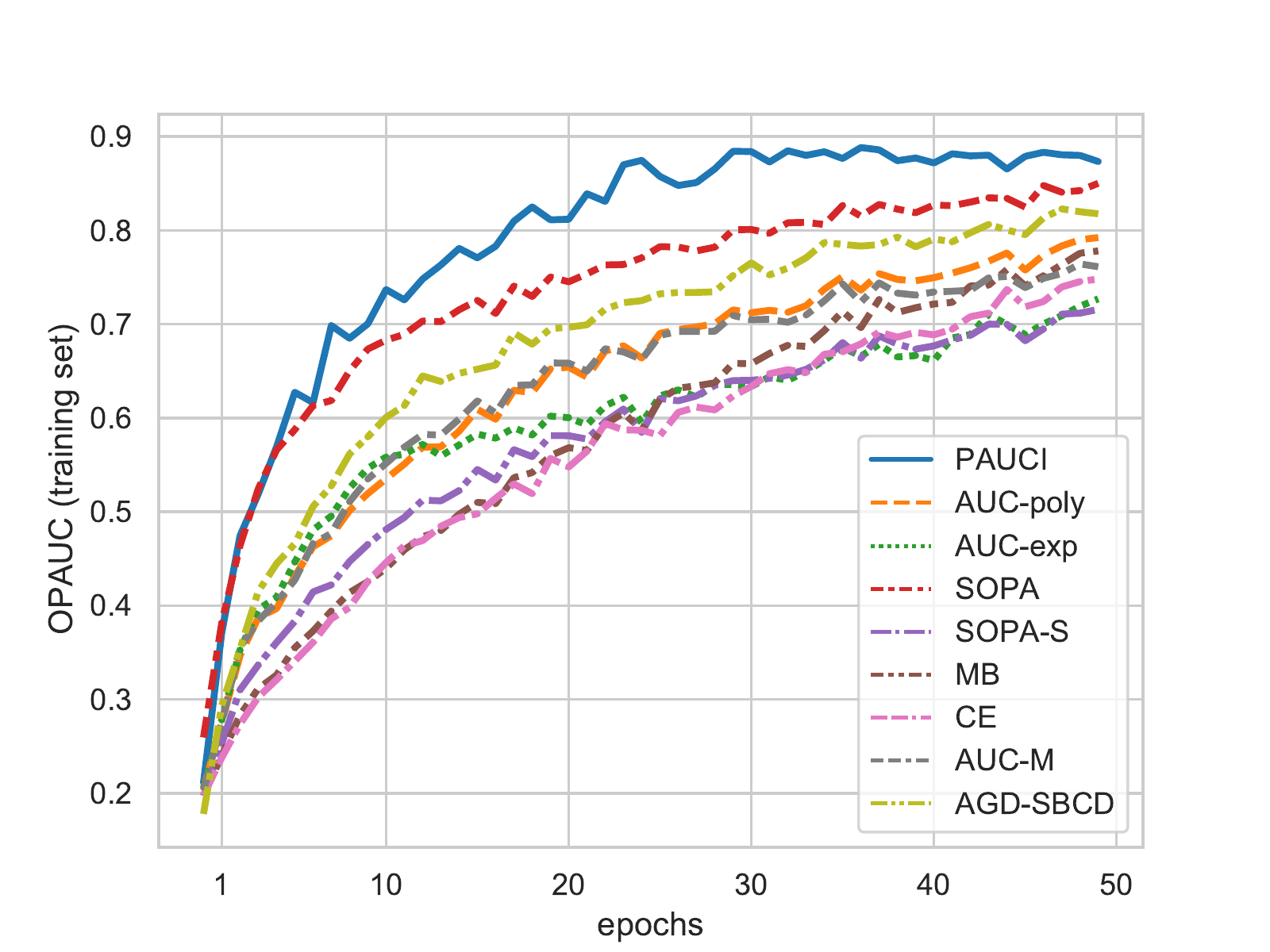}
		\caption*{(c) CIFAR-10-LT-3}
	\end{minipage}
	
	\caption{Convergence of OPAUC optimization.}
	\label{fig:opaucconvergence}
\end{figure}
\begin{figure}[!t]
	\centering
	
		\begin{minipage}{0.32\linewidth}
		\centering
		\includegraphics[width=\linewidth]{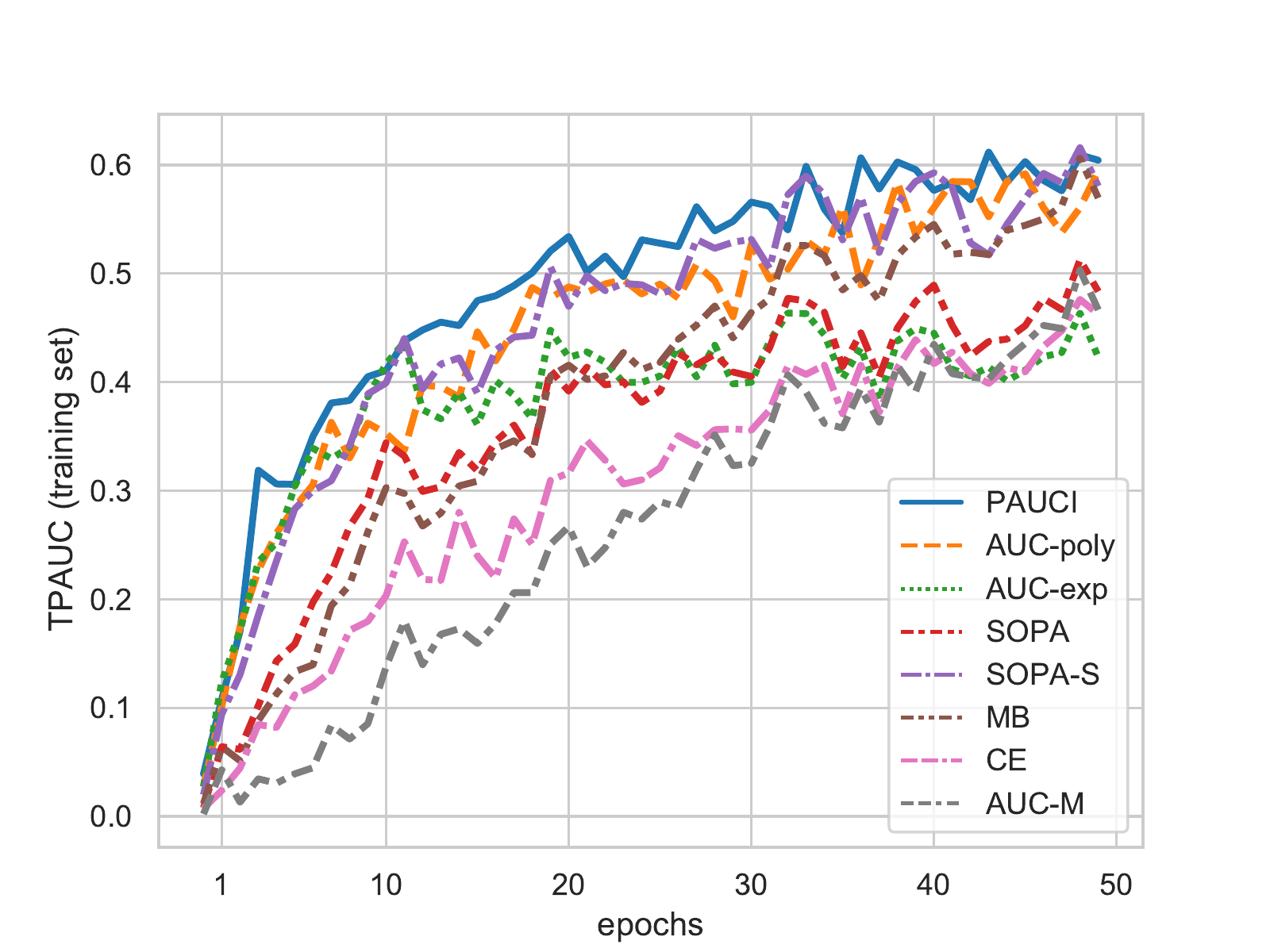}
		\caption*{(a) CIFAR-10-LT-1}
	\end{minipage}
	\begin{minipage}{0.32\linewidth}
		\centering
		\includegraphics[width=\linewidth]{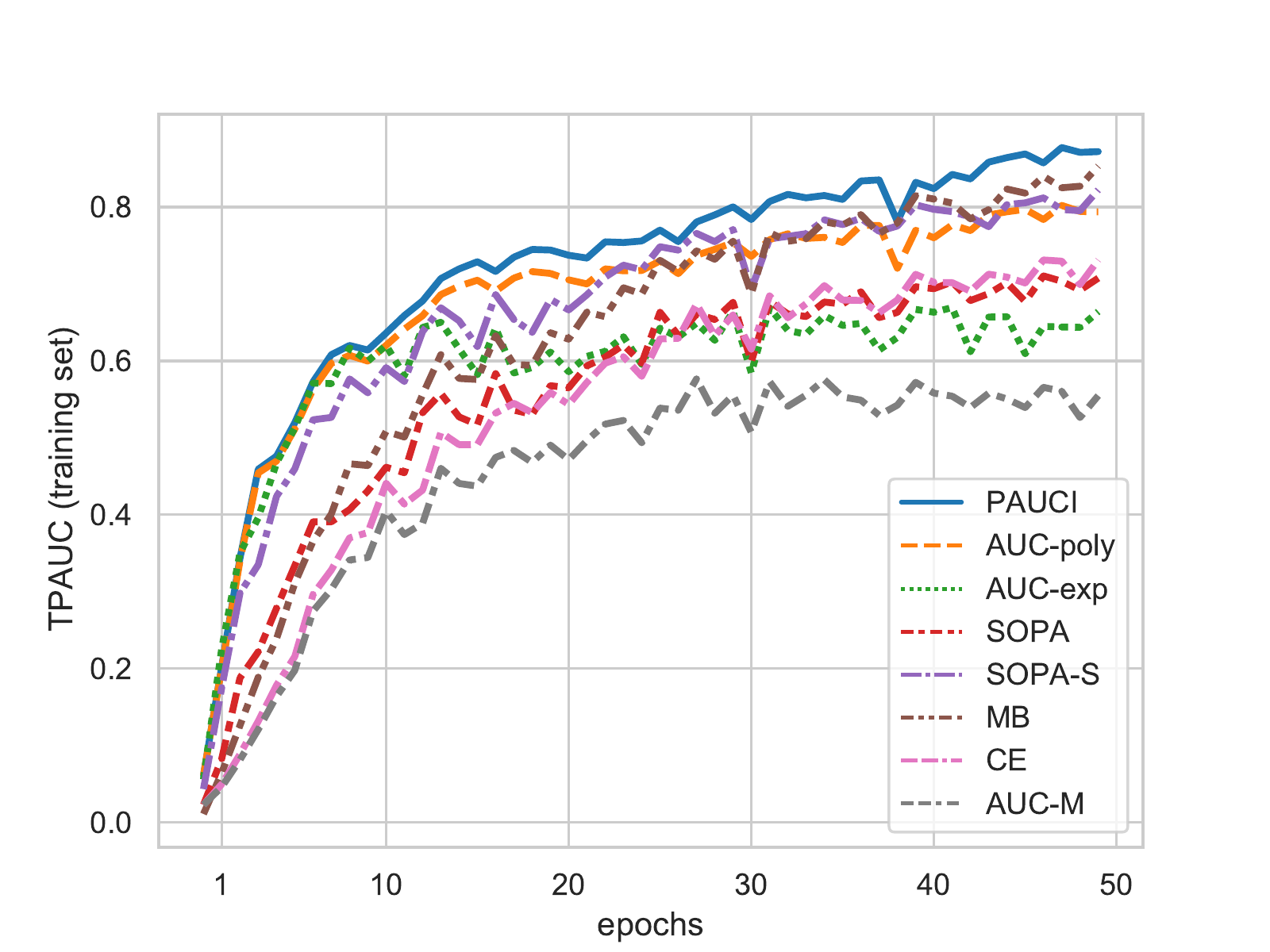}
		\caption*{(b) CIFAR-10-LT-2}
	\end{minipage}
	\begin{minipage}{0.32\linewidth}
		\centering
		\includegraphics[width=\linewidth]{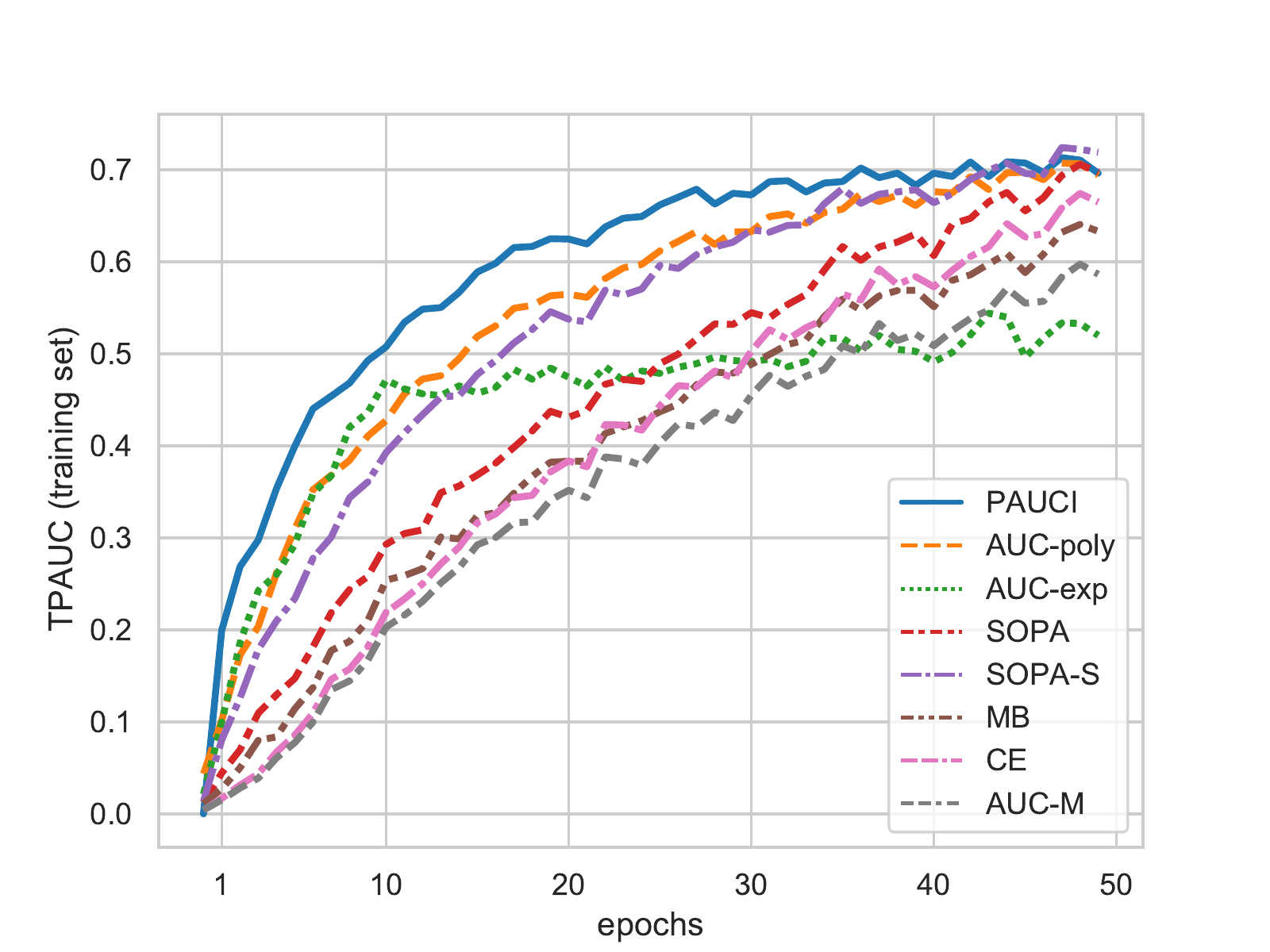}
		\caption*{(c) CIFAR-10-LT-3}
	\end{minipage}
	
	\caption{Convergence of TPAUC optimization.}
	\label{fig:tpaucconvergence}
\end{figure}

We conduct some experiments for per-iteration complexity with a fixed epoch with varying $n_+^B$ and $n_-^B$. All experiments are conducted on an Ubuntu 16.04.1 server with an Intel(R) Xeon(R) Silver 4110 CPU. For every method, we repeat running 10000 times and record the average running time. We only record the loss calculation time and use the python package time.time() to calculate the running time. Methods with * stand for the pair-wise estimator, while methods with ** stand for the instance-wise estimator. Here is the result of the experiment. We see the acceleration is significant when the data is large.

\begin{table}[!ht]
    \centering
    \caption{
    Pre-Iteration time complexity experiments for OPAUC ($\mathrm{FPR}\leq0.3$):
    }
    
        \begin{tabular}{ccccccc}
    \toprule
        unit:ms & $\begin{aligned}n_+^B=64\\n_-^B=64\end{aligned}$ & $\begin{aligned}n_+^B=128\\n_-^B=128\end{aligned}$ & $\begin{aligned}n_+^B=256\\n_-^B=256\end{aligned}$ & $\begin{aligned}n_+^B=512\\n_-^B=512\end{aligned}$ & $\begin{aligned}n_+^B=1024\\n_-^B=1024\end{aligned}$ & $\begin{aligned}n_+^B=2048\\n_-^B=2048\end{aligned}$ \\ \midrule
        SOPA* & 0.075 & 0.205 & 1.427 & 5.053 & 20.132 & 86.779 \\ 
        SOPA-S* & 0.063 & 0.165 & 0.946 & 4.003 & 15.815 & 62.031 \\ 
        AUC-poly* & 0.062 & 0.178 & 1.086 & 3.553 & 14.266 & 56.637 \\
        AUC-exp* & 0.063 & 0.182 & 0.985 & 3.513 & 14.155 & 55.689 \\ 
        AGD-SBCD* & 0.061 & 0.145 & 1.040 & 3.413 & 13.273 & 54.954 \\ 
        MB* & 0.121 & 0.174 & 0.468 & 1.713 & 6.393 & 25.663 \\ 
        PAUCI** & 0.026 & 0.029 & 0.033 & 0.043 & 0.072 & 0.107 \\ 
        AUC-M** & 0.025 & 0.028 & 0.031 & 0.040 & 0.059 & 0.104 \\ 
        CE** & 0.018 & 0.020 & 0.026 & 0.036 & 0.055 & 0.096 \\ 
    \bottomrule
    \end{tabular}
    
\end{table}

\begin{table}[!ht]
    \centering
    \caption{
    Pre-Iteration time complexity experiments for TPAUC ($\mathrm{FPR}\leq0.5,\mathrm{TPR}\geq0.5$):
    }
    
        \begin{tabular}{ccccccc}
    \toprule
        unit:ms & $\begin{aligned}n_+^B=64\\n_-^B=64\end{aligned}$ & $\begin{aligned}n_+^B=128\\n_-^B=128\end{aligned}$ & $\begin{aligned}n_+^B=256\\n_-^B=256\end{aligned}$ & $\begin{aligned}n_+^B=512\\n_-^B=512\end{aligned}$ & $\begin{aligned}n_+^B=1024\\n_-^B=1024\end{aligned}$ & $\begin{aligned}n_+^B=2048\\n_-^B=2048\end{aligned}$ \\ \midrule
        SOPA* & 0.079 & 0.206 & 1.439 & 5.197 & 20.556 & 88.314 \\
        SOPA-S* & 0.065 & 0.153 & 0.947 & 3.940 & 15.388 & 62.541 \\
        AUC-poly* & 0.062 & 0.180 & 1.175 & 3.573 & 14.440 & 56.469 \\ 
        AUC-exp* & 0.059 & 0.206 & 1.154 & 3.558 & 14.080 & 56.566 \\ 
        MB* & 0.173 & 0.198 & 0.491 & 1.955 & 6.554 & 29.369 \\ 
        PAUCI** & 0.030 & 0.030 & 0.038 & 0.045 & 0.071 & 0.109 \\ 
        AUC-M** & 0.025 & 0.027 & 0.033 & 0.043 & 0.059 & 0.104 \\
        CE** & 0.018 & 0.021 & 0.026 & 0.037 & 0.0535 & 0.096 \\ 
    \bottomrule
    \end{tabular}
    
\end{table}

\section{Proofs for Section 3}
\label{section:proofs_section4}

\subsection{Proof for Lemma \ref{lemma:1}}
\textbf{Remainder of Lemma \ref{lemma:1}.} \emph{$\sum_{i=1}^k x_{[i]}$ is a convex function of $(x_1,\cdots,x_n)$ where $x_{[i]}$ is the top-i element of a set $\{x_1,x_2,\cdots,x_n\}$. Furthermore, for $x_i,i=1,\cdots,n$, we have $\frac{1}{k}\sum_{i=1}^kx_{[i]}=\min_{s}\{s+\frac{1}{k}\sum_{i=1}^n[x_i-s]_+\}$, where $[a]_+=\max\{0,a\}$. The population version is $\mathbb{E}_{x}[x\cdot\mathbb{I}_{x\geq \eta(\alpha)}]=\min_s \frac{1}{\alpha}\mathbb{E}_{x}[\alpha s+[x-s]_+]$, where $\eta(\alpha)=\arg\min_{\eta\in\mathbb{R}}[\mathbb{E}_{x}[\mathbb{I}_{x\geq \eta}]=\alpha]$.}
\begin{proof}
For the summation case, please see Lemma 1 in \cite{fan2017learning} for the proof.  We only proof the expectation case here. Specifically, calculating the sub-differential of the term $\mathbb{E}_{x}[\alpha s+[x-s]_+]$ $\textit{w.r.t.}$, $s$, we get:
\begin{equation}
\begin{aligned}
& \alpha-\mathbb{E}_x[\mathbb{I}_{x\geq s}] \in \partial \left(\mathbb{E}_{x}[\alpha s +[x-s]_+]\right)\\
\end{aligned}
\end{equation}
Since $s$ is convex for $\alpha s +\left[x-s\right]_+$, so we can get the optimal $s$ by letting the it be 0:
\begin{equation}
    \mathbb{E}_x[\mathbb{I}_{x\geq s}]=\alpha
\end{equation}
It's' clear that optimal $s$ achieves $\mathrm{top}-\alpha$ quantile.
\end{proof}

\subsection{Proofs for OPAUC}
\subsubsection{Step 1}
\textbf{Remainder of Theorem \ref{theorem:1}.}
\emph{Assuming that $f(\bm{x})\in[0,1]$, $\forall \bm{x}\in\mathcal{X}$, $F_{op}(f,a,b,\gamma, t, \bm{z})$ is defined as:}

\begin{equation}
    \begin{aligned}
F_{op}(f,a,b,\gamma,t, \bm{z})=&\Large[(f(\bm{x})-a)^2- 
2(1+\gamma)f(\bm{x})\Large]y/p-\gamma^2\\
&\Large[(f(\bm{x})-b)^2+2(1+\gamma)f(\bm{x})\Large](1-y)\mathbb{I}_{f(\bm{x})\geq t}/(1-p)/\beta,
\end{aligned}
\label{eq:fop}
\end{equation}

\emph{where $y=1$ for positive instances, $y=0$ for negative instances and we have the following conclusions:}
\begin{enumerate}[leftmargin=20pt]
    \item[(a)] (\textbf{Population Version}.) We have:
    \begin{equation}
        \underset{f}{\min} \ {\mathcal{R}}_{\beta}(f) \Leftrightarrow \underset{\cmin}{\min}\ \underset{\gamma\in[-1,1]}{\max} \ 
    \colblue{\underset{\bm{z}\sim \mathcal{D}_\mathcal{Z}}{\mathbb{E}}}
    \left[F_{op}(f,a,b,\gamma, \colblue{\eta_\beta(f)},\bm{z})\right],
    \label{eq:minmaxopauc1}
    \end{equation}
     where  $\colblue{\eta_\beta(f)}=\arg\min_{\colblue{\eta_{\beta}} \in\mathbb{R}}\mathbb{E}_{\bm{x}'\sim \mathcal{D}_{\mathcal{N}}}[\mathbb{I}_{f(\bm{x}')\geq\ \colblue{\eta_{\beta}}}=\beta]$.
    \item[(b)] (\textbf{Empirical Version}.) Moreover, given a training dataset $S$ with sample size $n$, denote:
    \begin{equation*}
    \colbit{\underset{\bm{z} \sim S}{\ehat}}[F_{op}(f,a,b,\gamma,\colbit{\hat{\eta}_\beta(f)}, \bm{z})] = \frac{1}{n}\sum_{i=1}^n F_{op}(f,a,b,\gamma,{\colbit{\hat{\eta}_\beta(f)}}, \bm{z}_i),
    \end{equation*}
    where $\colbit{\hat{\eta}_\beta(f)}$ is the empirical quantile of the negative instances in $S$. We have:
    \begin{equation}
        \underset{f}{\min} \ \hat{\mathcal{R}}_{\beta}(f, S) \Leftrightarrow \underset{\cmin}{\min}\ \underset{\cmax}{\max} \ 
    \colbit{\underset{\bm{z}\sim S}{\ehat}}
    \left[F_{op}(f,a,b,\gamma, \colbit{\hat{\eta}_\beta(f)}, \bm{z} ) \right],
    \end{equation} 
\end{enumerate}

\begin{proof}
Firstly, we give a reformulation of $\mathrm{OPAUC}$:
\begin{equation}
\begin{aligned}
    \underset{f}{\min} \ {\mathcal{R}}_{\beta}(f) &=\underset{f}{\min} \ \mathbb{E}_{\bm{x} \sim \mathcal{D}_\mathcal{P}, \bm{x}'\sim \mathcal{D}_\mathcal{N}} \left[\mathbb{I}_{f(\bm{x}') \ge \eta_\beta(f)} \cdot \ell(f(\bm{x})- f(\bm{x}'))\right]\\
    &=\underset{f}{\min}\ \mathbb{E}_{\bm{x} \sim \mathcal{D}_\mathcal{P}, \bm{x}'\sim \mathcal{D}_\mathcal{N}} \left[\ell(f(\bm{x})- f(\bm{x}'))|f(\bm{x}') \ge \eta_\beta(f)\right]\cdot \underset{\bm{x}'\sim\mathcal{D}_{\mathcal{N}}}{\mathbb{P}}[f(\bm{x}') \ge \eta_\beta(f)]\\
    &=\underset{f}{\min}\ \mathbb{E}_{\bm{x} \sim \mathcal{D}_\mathcal{P}, \bm{x}'\sim \mathcal{D}_\mathcal{N}} \left[\ell(f(\bm{x})- f(\bm{x}'))|f(\bm{x}') \ge \eta_\beta(f)\right]\cdot\beta\\
    &=\beta\cdot\underset{f}{\min}\ \mathbb{E}_{\bm{x} \sim \mathcal{D}_\mathcal{P}, \bm{x}'\sim \mathcal{D}_\mathcal{N}} \left[\ell(f(\bm{x})- f(\bm{x}'))|f(\bm{x}') \ge \eta_\beta(f)\right].
\end{aligned}
\end{equation}

Applying the surrogate loss $(1-x)^2$ to the estimator of $\mathrm{OPAUC}$, we have:
\begin{equation}
\begin{aligned}
&\underset{\bm{x},\bm{x}'\sim\mathcal{D}_{\mathcal{P}}, \mathcal{D}_{\mathcal{N}}}
{\mathbb{E}}[(1-(f(\bm{x})-f(\bm{x}')))^2|f(\bm{x}')\geq\eta_\beta(f)]\\
&=1+\exdp[f(\bm{x})^2]
+\exdn[f(\bm{x}')^2|f(\bm{x}')\geq 
\eta_{\beta}(f)] -2\exdp[f(\bm{x})]\\
& \qquad +2\exdn[f(\bm{x}')|f(\bm{x}')\geq 
\eta_{\beta}(f)] -2\underset{\bm{x}\sim \mathcal{D}_\mathcal{P}}{\mathbb{E}}[f(\bm{x})]
\underset{\bm{x}'\sim \mathcal{D}_\mathcal{N}}{\mathbb{E}}[f(\bm{x}')|f(\bm{x}')\geq 
\eta_{\beta}(f)]\\
&=1+\underset{\bm{x}\sim \mathcal{D}_\mathcal{P}}{\mathbb{E}}[f(\bm{x})^2]
-\underset{\bm{x}\sim \mathcal{D}_\mathcal{P}}{\mathbb{E}}[f(\bm{x})]^2+\underset{\bm{x}'\sim \mathcal{D}_\mathcal{N}}{\mathbb{E}}[f(\bm{x}')^2|f(\bm{x}') \ge \eta_\beta(f)] \\
& \qquad - \underset{\bm{x}'\sim \mathcal{D}_\mathcal{N}}{\mathbb{E}}[f(\bm{x}')^2|f(\bm{x}') \ge \eta_\beta(f)]^2-2\underset{\bm{x}\sim \mathcal{D}_\mathcal{P}}{\mathbb{E}}[f(\bm{x})]
+2\underset{\bm{x}'\sim \mathcal{D}_\mathcal{N}}{\mathbb{E}}[f(\bm{x}')|f(\bm{x}') \ge \eta_\beta(f)] \\
& \qquad + (\underset{\bm{x}\sim \mathcal{D}_\mathcal{P}}{\mathbb{E}}[f(\bm{x})]
-\underset{\bm{x}'\sim \mathcal{D}_\mathcal{N}}{\mathbb{E}}[f(\bm{x}')|f(\bm{x}') \ge \eta_\beta(f)])^2.
\end{aligned}
\end{equation}
Note that 
\begin{equation}
   \underset{\bm{x}\sim \mathcal{D}_\mathcal{P}}{\mathbb{E}}[f(\bm{x})^2]-\underset{\bm{x}\sim \mathcal{D}_\mathcal{P}}{\mathbb{E}}[f(\bm{x})]^2=\min_{a\in[0,1]}\underset{\bm{x}\sim \mathcal{D}_\mathcal{P}}{\mathbb{E}}[(f(\bm{x})-a)^2],
\end{equation}
where the minimization is achieved by:
\begin{equation}
    a^* = \underset{\bm{x}\sim \mathcal{D}_\mathcal{P}}{\mathbb{E}}[f(\bm{x})],
\end{equation}
where $a^*\in[0, 1]$. Likewise, 
\begin{equation}
\begin{aligned}
    &\underset{\bm{x}'\sim \mathcal{D}_\mathcal{N}}{\mathbb{E}}[f(\bm{x}')^2|f(\bm{x}') \ge \eta_\beta(f)] - \underset{\bm{x}'\sim \mathcal{D}_\mathcal{N}}{\mathbb{E}}[f(\bm{x}')|f(\bm{x}') \ge \eta_\beta(f)]^2=\\
    &\qquad \qquad \qquad \min_{b\in[0,1]}\underset{\bm{x}'\sim \mathcal{D}_\mathcal{N}}{\mathbb{E}}[(f(\bm{x}')-b)^2|f(\bm{x}') \ge \eta_\beta(f)],
\end{aligned}
\end{equation}
where the minimization is get by:
\begin{equation}
    b^* = \underset{\bm{x}'\sim \mathcal{D}_\mathcal{N}}{\mathbb{E}}[f(\bm{x}')|f(\bm{x}') \ge \eta_\beta(f)].
\end{equation}
where $b^*\in[0, 1]$. It's notable that
\begin{equation}
\begin{aligned}
    &\left(\underset{\bm{x}'\sim \mathcal{D}_\mathcal{N}}{\mathbb{E}}[f(\bm{x}')|f(\bm{x}') \ge \eta_\beta(f)]-\underset{\bm{x}\sim \mathcal{D}_\mathcal{P}}{\mathbb{E}}[f(\bm{x})]\right)^2=\\
&\max_{\gamma}\left\{2\gamma\left(\underset{\bm{x}'\sim \mathcal{D}_\mathcal{N}}{\mathbb{E}}[f(\bm{x}')|f(\bm{x}') \ge \eta_\beta(f)]-\underset{\bm{x}\sim \mathcal{D}_\mathcal{P}}{\mathbb{E}}[f(\bm{x})]\right)-\gamma^2\right\}, 
\end{aligned}
\end{equation}
where the maximization can be obtained by:
\begin{equation}
    \gamma^* = \underset{\bm{x}'\sim \mathcal{D}_\mathcal{N}}{\mathbb{E}}[f(\bm{x}')|f(\bm{x}') \ge \eta_\beta(f)]-\underset{\bm{x}\sim \mathcal{D}_\mathcal{P}}{\mathbb{E}}[f(\bm{x})].
\end{equation}
It's clear that $\gamma^*=b^*-a^*$. Then we can constraint $\gamma$ with range $[-1, 1]$ and get the equivalent optimization formulation:
\begin{equation}
\begin{aligned}
    &\underset{\bm{x},\bm{x}'\sim\mathcal{D}_{\mathcal{P}}, \mathcal{D}_{\mathcal{N}}}
{\mathbb{E}}[(1-(f(\bm{x})-f(\bm{x}')))^2|f(\bm{x}')\geq\eta_\beta(f)]\Leftrightarrow \\
&\min_{(a,b)\in[0,1]^2}\max_{\gamma\in[-1,1]} \exdp [(f(\bm{x})-a)^2-2(\gamma+1)f(\bm{x})]  -\gamma^2 \\
&+ \exdn [(f(\bm{x}')-b)^2+2(\gamma+1)f(\bm{x}')|f(\bm{x}')\geq \eta_{\beta}(f)].
\end{aligned}
\end{equation}
Taking expectation $\textit{w.r.t.}$, $\bm{z}$, we have:
\begin{equation}
    \begin{aligned}
\underset{f}{\min} \ \mathcal{R}_{\beta}(f)
\Leftrightarrow \underset{f,a,b}{\min}\ \underset{\gamma\in[-1, 1]}{\max} \ 
\colblue{\underset{\bm{z}\sim \mathcal{D}_\mathcal{Z}}{\mathbb{E}}}
[F_{op}(f,a,b,\gamma,\eta_{\beta}(f),\bm{z})],
\end{aligned}
\end{equation}
and the instance-wise function $F_{op}(f,a,b,\gamma,\eta_{\beta}(f),\bm{z})$ is defined by:

\begin{equation}
    \begin{aligned}
F_{op}(f,a,b,\gamma,t, \bm{z})=&\Large[(f(\bm{x})-a)^2- 
2(1+\gamma)f(\bm{x})\Large]y/p-\gamma^2\\
&\Large[(f(\bm{x})-b)^2+2(1+\gamma)f(\bm{x})\Large](1-y)\mathbb{I}_{f(\bm{x})\geq t}/(1-p)/\beta,
\end{aligned}
\end{equation}

where $p=\Pr[y=1]$. The same result holds for empirical version $\colbit{\underset{\bm{z}\sim S}{\hat{\mathbb{E}}}}
[F_{op}(f,a,b,\gamma,\hat{\eta}_{\beta}(f),\bm{z})]$.
\end{proof}

\subsubsection{Step 2}
First we need the following proposition to complete the proof in this subsection.
\begin{proposition}
If $\gamma \in \Omega_{\gamma} =[b-1, 1]$, $\ell_{-}(\bm{x}')=(f(\bm{x}')-b)^2+2(1+\gamma) f(\bm{x}')$ is an increasing function \wrt $f(\bm{x}')$ when $\bm{x}'\sim\mathcal{D}_\mathcal{N}$ and $f(\bm{x}')\in[0,1]$. 
\label{proposition1}
\begin{proof}
We have:
\begin{equation}
    \frac{\partial \ell_{-}(\bm{x}')}{\partial f(\bm{x}')}=2(f(\bm{x}')-b+1+\gamma).
\end{equation}
Assuming that $f(\bm{x}')\in[0,1]$, then the feasible solution of $b$ is nonnegative. When $\gamma \in [b-1,1]$, the negative loss function's partial derivative $\partial\ell_{-}(\bm{x}')/\partial f(\bm{x}')\geq 0$. Then $\ell_{-}(\bm{x}')$ is an increasing function \wrt $f(\bm{x}')$.
\end{proof}
\end{proposition}
\begin{remark}
For negative instances, if the loss function is an increasing function \wrt the score $f(\bm{x}')$, then the top-ranked losses are equivalent to the losses of top-ranked instances.
\end{remark}
\textbf{Reminder of Theorem \ref{theorem:2}.}
\emph{Assuming that $f(\bm{x})\in[0,1]$, for all $\bm{x}\in\mathcal{X}$, we have the equivalent optimization for $\mathrm{OPAUC}$: }
\begin{equation}
\begin{aligned}
    \underset{\cmin}{\min}\ \underset{\gamma \in [-1,1]}{\max} \ 
\colblue{\underset{\bm{z}\sim \mathcal{D}_\mathcal{Z}}{\mathbb{E}}}[F_{op}(f,a,b,\gamma,\colblue{\efb}, \bm{z})]
&\Leftrightarrow \underset{\cmin}{\min}\ 
\underset{\cmax }{\max} 
\ \underset{s'\in\Omega_{s'}}{\min}
\ \colblue{\underset{\bm{z}\sim \mathcal{D}_\mathcal{Z}}{\mathbb{E}}}[G_{op}(f,a,b,\gamma,\bm{z},s')],
\end{aligned}
\end{equation}
\begin{equation}
    \begin{aligned}
        \underset{\cmin}{\min}\ \underset{\gamma\in [-1,1]}{\max} \ 
    \colbit{\underset{\bm{z}\sim S}{\ehat}}[F_{op}(f,a,b,\gamma, \colbit{\hefb},\bm{z})]
    &\Leftrightarrow \underset{\cmin}{\min}\ 
    \underset{\gamma\in\Omega_{\gamma} }{\max} 
    \ \underset{s'\in\Omega_{s'}}{\min}
    \ \colbit{\underset{\bm{z}\sim S}{\ehat}}[G_{op}(f,a,b,\gamma,\bm{z},s')],
\end{aligned}
\end{equation}

\emph{where $\Omega_{\gamma}=[b-1,1]$, $\Omega_{s'}=[0,5]$ and}

\begin{equation}
    \begin{aligned}
G_{op}(f,a,b,\gamma,\bm{z},s')&=[(f(\bm{x})-a)^2- 
2(1+\gamma)f(\bm{x})]y/p-\gamma^2\\
& +
\left(\beta s' +\left[(f(\bm{x})-b)^2+2(1+\gamma) f(\bm{x})-s'\right]_+\right)(1-y)/[\beta(1-p)].
\end{aligned}
\end{equation}

\begin{proof}


According to the Thm.\ref{Constrainted_Reformulation} in Appendix.\ref{val_constraintval},
when we constraint $\gamma$ in range $\Omega_{\gamma}=[b-1, 1]$, we have:

\begin{equation}
    \min_{f,(a,b)\in[0,1]^2}\max_{\gamma\in [-1,1]}\colblue{\mathbb{E}_{z\sim\mathcal{D}_{\mathcal{Z}}}} [F_{op}]
    \Leftrightarrow
    \min_{f,(a,b)\in[0,1]^2}\max_{\gamma\in [b-1,1]}\colblue{\mathbb{E}_{z\sim\mathcal{D}_{\mathcal{Z}}}} [F_{op}]
\end{equation}

According to Thm.\ref{theorem:1}, we have: 

\begin{equation}
    \begin{aligned}
\underset{\bm{z}\sim \mathcal{D}_\mathcal{Z}}{\mathbb{E}}
[F_{op}(f,a,b,\gamma,\colblue{\eta_{\beta}(f)},&\bm{z})]\Leftrightarrow 
\underset{\bm{x}\sim \mathcal{D}_\mathcal{P}}{\mathbb{E}}[(f(\bm{x})-a)^2- 
2(1+\gamma)f(\bm{x})] -\gamma^2\\
& +\underset{\bm{x}'\sim \mathcal{D}_\mathcal{N}}{\mathbb{E}}
\left([(f(\bm{x}')-b)^2+2(1+\gamma)f(\bm{x}')]\cdot \mathbb{I}_{f(\bm{x})\geq \eta_{\beta}(f)}\right)/\beta.
\end{aligned}
\end{equation}

We denote $\ell_-(\bm{x}')=(f(\bm{x}')-b)^2+2(1+\gamma)f(\bm{x}')$. The Prop.\ref{proposition1} ensures that the negative loss function $\ell_-(\bm{x}')$ is an increasing function when $\gamma\in[b-1,1]$. Then we can get: 

\begin{equation}
    \mathbb{E}_{\bm{x}'\sim\mathcal{D}_\mathcal{N}}[\mathbb{I}_{f(\bm{x}')\geq\eta_{\beta}(f)}\cdot\ell_-(\bm{x}')]= \min_s \frac{1}{\beta} \mathbb{E}_{\bm{x}'\sim\mathcal{D}_\mathcal{N}} [\beta s + [\ell_-(\bm{x}')-s]_+],
\end{equation}

Applying Lem.\ref{lemma:1} to negative loss function, then we have:

\begin{equation}
    \begin{aligned}
\colblue{\underset{\bm{z}\sim \mathcal{D}_\mathcal{Z}}{\mathbb{E}}}
[F_{op}(f,a,b,\gamma,\colblue{\eta_{\beta}(f)},\bm{z})]&= \underset{s'}{\min}
\underset{\bm{x}\sim \mathcal{D}_\mathcal{P}}{\mathbb{E}}[(f(\bm{x})-a)^2- 
2(1+\gamma)f(\bm{x})] -\gamma^2\\
&\quad +\underset{\bm{x}'\sim \mathcal{D}_\mathcal{N}}{\mathbb{E}}
\left(\beta s' +\left[(f(\bm{x}')-b)^2+2(1+\gamma)f(\bm{x}')-s'\right]_+\right)/\beta.
\end{aligned}
\end{equation}

Then, we get:
\begin{equation}
\begin{aligned}
\colblue{\underset{\bm{z}\sim \mathcal{D}_\mathcal{Z}}{\mathbb{E}}}[F_{op}(f,a,b,\gamma,\colblue{\eta_{\beta}(f)},\bm{z})]
&=
\ \underset{s'\in\Omega_{s'}}{\min}
\ \colblue{\underset{\bm{z}\sim \mathcal{D}_\mathcal{Z}}{\mathbb{E}}}[G_{op}(f,a,b,\gamma,\bm{z},s')],
\end{aligned}
\end{equation}
where

\begin{equation}
    \begin{aligned}
G_{op}(f,a,b,\gamma,\bm{z},s')&=[(f(\bm{x})-a)^2- 
2(1+\gamma)f(\bm{x})]y/p -\gamma^2\\
& +
\left(\beta s' +\left[(f(\bm{x})-b)^2+2(1+\gamma) f(\bm{x})-s'\right]_+\right)(1-y)/[\beta(1-p)].
\end{aligned}
\end{equation}

We have the equivalent optimization for $\mathrm{OPAUC}$:
\begin{equation}
\begin{aligned}
    &\underset{\cmin}{\min}\ \underset{\gamma\in[-1,1]}{\max} \ 
\colblue{\underset{\bm{z}\sim \mathcal{D}_\mathcal{Z}}{\mathbb{E}}}[F_{op}(f,a,b,\gamma,\colblue{\efb}, \bm{z})]
\Leftrightarrow \\
&\underset{\cmin}{\min}\ 
\underset{\cmax }{\max} 
\ \underset{s'\in\Omega_{s'}}{\min}
\ \colblue{\underset{\bm{z}\sim \mathcal{D}_\mathcal{Z}}{\mathbb{E}}}[G_{op}(f,a,b,\gamma,\bm{z},s')],
\end{aligned}
\end{equation}
where $\Omega_{\gamma}=[b-1,1]$, $\Omega_{s'}=[0,5]$, $p=\mathbb{P}[y=1]$. The same result holds for empirical version $\colbit{\underset{\bm{z}\sim S}{\hat{\mathbb{E}}}}[G_{op}(f,a,b,\gamma,\bm{z},s')]$.

\end{proof}

\subsection{Proofs for TPAUC}

\subsubsection{Step 1}
\textbf{Reminder of Theorem \ref{theorem:7}.}
\emph{Assuming that $f(\bm{x})\in[0,1]$, $\forall \bm{x}\in\mathcal{X}$, $F_{tp}(f,a,b,\gamma, t, t', \bm{z})$ is defined as:}

\begin{equation}
    \begin{aligned}
F_{tp}(f,a,b,&\gamma,t, t', \bm{z})=(f(\bm{x})-a)^2y\mathbb{I}_{f(\bm{x})\leq t}/(\alpha p)+
(f(\bm{x})-b)^2(1-y)\mathbb{I}_{f(\bm{x}')\geq t'}/[\beta(1-p)]\\
&+2(1+\gamma)f(\bm{x})(1-y)\mathbb{I}_{f(\bm{x}')\geq t'}/[\beta(1-p)] - 
2(1+\gamma)f(\bm{x})y\mathbb{I}_{f(\bm{x})\leq t}/(\alpha p) -\gamma^2,
\end{aligned}
\label{eq:fop}
\end{equation}

\emph{where $y=1$ for positive instances, $y=0$ for negative instances and we have the following conclusions:
\begin{enumerate}[leftmargin=20pt]
    \item[(a)] (\textbf{Population Version}.) We have:
    \begin{equation}
        \underset{f}{\min} \ {\mathcal{R}}_{\alpha,\beta}(f) \Leftrightarrow \underset{\cmin}{\min}\ \underset{\gamma\in[-1,1]}{\max} \ 
    \colblue{\underset{\bm{z}\sim \mathcal{D}_\mathcal{Z}}{\mathbb{E}}}
    \left[F_{tp}(f,a,b,\gamma, \colblue{\eta_\alpha(f)},\colblue{\eta_\beta(f)},\bm{z})\right],
    \label{eq:minmaxtpauc1}
    \end{equation}
    where $\colblue{\eta_\alpha(f)}=\arg\min_{\colblue{\eta_{\alpha}} \in\mathbb{R}}\mathbb{E}_{\bm{x}\sim \mathcal{D}_{\mathcal{P}}}[\mathbb{I}_{f(\bm{x})\leq\ \colblue{\eta_{\alpha}}}=\alpha]$ and $\colblue{\eta_\beta(f)}=\arg\min_{\colblue{\eta_{\beta}} \in\mathbb{R}}\mathbb{E}_{\bm{x}'\sim \mathcal{D}_{\mathcal{N}}}[\mathbb{I}_{f(\bm{x}')\geq\ \colblue{\eta_{\beta}}}=\beta]$.
    \item[(b)] (\textbf{Empirical Version}.) Moreover, given a training dataset $S$ with sample size $n$, denote:
    \begin{equation*}
    \colbit{\underset{\bm{z} \sim S}{\ehat}}[F_{tp}(f,a,b,\gamma,\colbit{\hat{\eta}_{\alpha}(f)}, \colbit{\hat{\eta}_\beta(f)}, \bm{z})] = \frac{1}{n}\sum_{i=1}^n F_{tp}(f,a,b,\gamma,\colbit{\hat{\eta}_{\alpha}(f)}, {\colbit{\hat{\eta}_\beta(f)}}, \bm{z})
    \end{equation*}
    where $\colbit{\hat{\eta}_{\alpha}(f)}$ and $\colbit{\hat{\eta}_\beta(f)}$ are the empirical quantile of the positive and negative instances in $S$, respectively. We have:
    \begin{equation}
        \underset{f}{\min} \ \hat{\mathcal{R}}_{\alpha,\beta}(f, S) \Leftrightarrow \underset{\cmin}{\min}\ \underset{\cmax}{\max} \ 
    \colbit{\underset{\bm{z}\sim S}{\ehat}}
    \left[F_{tp}(f,a,b,\gamma, \colbit{\hat{\eta}_{\alpha}(f)}, \colbit{\hat{\eta}_\beta(f)}, \bm{z} ) \right],
    \end{equation} 
\end{enumerate}
}
\begin{proof}
Firstly, we give a reformulation of $\mathrm{TPAUC}$:

\begin{equation}
\begin{aligned}
    \underset{f}{\min} \ {\mathcal{R}}_{\alpha,\beta}(f) &=\underset{f}{\min} \ \mathbb{E}_{\bm{x} \sim \mathcal{D}_\mathcal{P}, \bm{x}'\sim \mathcal{D}_\mathcal{N}} \left[\mathbb{I}_{f(\bm{x}) \le \eta_\alpha(f)} \cdot\mathbb{I}_{f(\bm{x}') \ge \eta_\beta(f)} \cdot \ell(f(\bm{x})- f(\bm{x}'))\right]\\
    &=\underset{f}{\min}\ \mathbb{E}_{\bm{x} \sim \mathcal{D}_\mathcal{P}, \bm{x}'\sim \mathcal{D}_\mathcal{N}} \left[\ell(f(\bm{x})- f(\bm{x}'))|f(\bm{x}') \ge \eta_\beta(f), f(\bm{x}) \le \eta_\alpha(f)\right]\\
    & \quad \cdot \underset{\bm{x}'\sim\mathcal{D}_{\mathcal{N}}}{\mathbb{P}}[f(\bm{x}') \ge \eta_\beta(f)]\cdot \underset{\bm{x}\sim\mathcal{D}_{\mathcal{P}}}{\mathbb{P}}[f(\bm{x}) \le \eta_\alpha(f)]\\
    &=\underset{f}{\min}\ \mathbb{E}_{\bm{x} \sim \mathcal{D}_\mathcal{P}, \bm{x}'\sim \mathcal{D}_\mathcal{N}} \left[\ell(f(\bm{x})- f(\bm{x}'))|f(\bm{x}') \ge \eta_\beta(f),f(\bm{x}) \le \eta_\alpha(f) \right]\cdot\alpha\beta\\
    &=\alpha\beta\cdot\underset{f}{\min}\ \mathbb{E}_{\bm{x} \sim \mathcal{D}_\mathcal{P}, \bm{x}'\sim \mathcal{D}_\mathcal{N}} \left[\ell(f(\bm{x})- f(\bm{x}'))|f(\bm{x}') \ge \eta_\beta(f),f(\bm{x}) \le \eta_\alpha(f) \right].
\end{aligned}
\end{equation}

Similar to the proof of Thm.\ref{theorem:1}, using the square surrogate loss, we can get the equivalent optimization formulation:
\begin{equation}
    \begin{aligned}
&\underset{f}{\min} \ \mathcal{R}_{\alpha,\beta}(f)
\Leftrightarrow  \underset{f,(a,b)\in[0,1]^2}{\min}\ \underset{\gamma\in[-1,1]}{\max} \ 
\colblue{\underset{\bm{z}\sim \mathcal{D}_\mathcal{Z}}{\mathbb{E}}}
[F_{tp}(f,a,b,\gamma,\colblue{\eta_{\alpha}(f)}, \colblue{\eta_{\beta}(f)},\bm{z})],
\end{aligned}
\end{equation}
and the instance-wise function $F_{tp}(f,a,b,\gamma,\colblue{\eta_{\alpha}(f)},\colblue{\eta_{\beta}(f)},\bm{z})$ is defined by:

\begin{equation}
    \begin{aligned}
F_{tp}&(f,a,b,\gamma,\colblue{\eta_{\alpha}(f)},\colblue{\eta_{\beta}(f)},\bm{z})\\
&=(f(\bm{x})-a)^2y\mathbb{I}_{f(\bm{x})\leq\eta_{\alpha}(f)}/(\alpha p)+
(f(\bm{x})-b)^2(1-y)\mathbb{I}_{f(\bm{x})\geq\eta_{\beta}(f)}/[\beta (1-p)]\\
&+2(1+\gamma)f(\bm{x})(1-y)\mathbb{I}_{f(\bm{x})\geq\eta_{\beta}(f)}/[\beta (1-p)] - 
2(1+\gamma)f(\bm{x})y\mathbb{I}_{f(\bm{x})\leq\eta_{\alpha}(f)}/(\alpha p) -\gamma^2.
\end{aligned}
\end{equation}

The same result holds for empirical version $\colbit{\underset{\bm{z}\sim S}{\hat{\mathbb{E}}}}
[F_{tp}(f,a,b,\gamma,\colbit{\hat{\eta}_{\alpha}(f)},\colbit{\hat{\eta}_{\beta}(f)},\bm{z})]$.
\end{proof}

\subsubsection{Step 2}
First we need the following proposition to complete the proof in this subsection.


\begin{proposition}
If $\gamma \in \Omega_{\gamma} =[\max\{b-1, -a\}, 1]$, $\ell_+(\bm{x})=(f(\bm{x})-a)^2-2(1+\gamma) f(\bm{x})$ is a decreasing function \wrt $f(\bm{x})$ when $\bm{x}\sim\mathcal{D}_\mathcal{P}$ and $f(\bm{x})\in[0,1]$.
\label{proposition2}
\begin{proof}
We have:
\begin{equation}
    \frac{\partial \ell_{+}(\bm{x})}{\partial f(\bm{x})}=2(f(\bm{x})-a-1-\gamma).
\end{equation}
Assuming that $f(\bm{x})\in [0,1]$, then the feasible solution of $a$ is nonnegative. When $\gamma \in [\max\{b-1, -a\}, 1]$, the positive loss function's partial derivative $\partial\ell_{+}(\bm{x})/\partial f(\bm{x})\leq 0$. Then $\ell_{+}(\bm{x})$ is an decreasing function \wrt $f(\bm{x})$.
\end{proof}
\end{proposition}
\begin{remark}
For positive instances, if the loss function is an decreasing function \wrt the score $f(\bm{x})$, then the top-ranked losses are equivalent to the losses of bottom-ranked instances.
\end{remark}

\textbf{Reminder of Theorem \ref{theorem:8}.}
\emph{Assuming that $f(\bm{x})\in[0,1]$ for all $\bm{x} \in\mathcal{X}$, we have the equivalent optimization for $\mathrm{TPAUC}$:}
\begin{equation}
\begin{aligned}
    \underset{\cmin}{\min}\ \underset{\gamma\in[-1,1]}{\max} \ 
\colblue{\underset{\bm{z}\sim \mathcal{D}_\mathcal{Z}}{\mathbb{E}}}[F_{tp}(f,a,b,\gamma,\colblue{\efa}, \colblue{\efb}, \bm{z})]
\\ \Leftrightarrow \underset{\cmin}{\min}\ 
\underset{\cmax }{\max} 
\ \underset{s\in\Omega_{s},s'\in\Omega_{s'}}{\min}
\ \colblue{\underset{\bm{z}\sim \mathcal{D}_\mathcal{Z}}{\mathbb{E}}}[G_{tp}(f,a,b,\gamma,\bm{z},s,s')],
\end{aligned}
\end{equation}
\begin{equation}
    \begin{aligned}
        \underset{\cmin}{\min}\ \underset{\gamma\in[-1,1]}{\max} \ 
    \colbit{\underset{\bm{z}\sim S}{\ehat}}[F_{tp}(f,a,b,\gamma, \colbit{\hefa}, \colbit{\hefb},\bm{z})]
    \\ \Leftrightarrow \underset{\cmin}{\min}\ 
    \underset{\gamma\in\Omega_{\gamma} }{\max} 
    \ \underset{s\in\Omega_{s}, s'\Omega_{s'}}{\min}
    \ \colbit{\underset{\bm{z}\sim S}{\ehat}}[G_{tp}(f,a,b,\gamma,\bm{z},s,s')],
    \end{aligned}
\end{equation}

\emph{where $\Omega_{\gamma}=[\max\{b-1, -a\},1]$, $\Omega_{s}=[-4, 1]$, $\Omega_{s'}=[0,5]$ and}

\begin{equation}
    \begin{aligned}
G_{tp}(f,a,b,\gamma,\bm{z},s, s')&=\left(\alpha s + \left[(f(\bm{x})-a)^2- 
2(1+\gamma)f(\bm{x})-s\right]_+\right)y/(\alpha p)\\
& +
\left(\beta s' +\left[(f(\bm{x})-b)^2+2(1+\gamma) f(\bm{x})-s'\right]_+\right)(1-y)/[\beta(1-p)] -\gamma^2.
\end{aligned}
\label{OPAUC_IB}
\end{equation}

\begin{proof}

According to the Thm.\ref{Constrainted_Reformulation} in Appendix.\ref{val_constraintval},
when we constraint $\gamma$ in range $\Omega_{\gamma}=[\max\{-a,b-1\}, 1]$, we have:
\begin{equation}
    \min_{f,(a,b)\in[0,1]^2}\max_{\gamma\in [-1,1]}\colblue{\underset{\bm{z}\sim\mathcal{D}_{\mathcal{Z}}}{\mathbb{E}}} [F_{tp}]
    \Leftrightarrow
    \min_{f,(a,b)\in[0,1]^2}\max_{\gamma\in [\max\{-a,b-1\},1]}\colblue{\underset{\bm{z}\sim\mathcal{D}_{\mathcal{Z}}}{\mathbb{E}}} [F_{tp}]
\end{equation}

According to the Thm.\ref{theorem:8}, we have:

\begin{equation}
    \begin{aligned}
\colblue{\underset{\bm{z}\sim \mathcal{D}_\mathcal{Z}}{\mathbb{E}}}
[F_{tp}(f,a,b,\gamma,\colblue{\eta_{\alpha}(f)},&\colblue{\eta_{\beta}(f)},\bm{z})]\Leftrightarrow 
\underset{\bm{x}\sim \mathcal{D}_\mathcal{P}}{\mathbb{E}}\left([(f(\bm{x})-a)^2- 
2(1+\gamma)f(\bm{x})]\cdot\mathbb{I}_{f(\bm{x})\leq \eta_{\alpha}(f)}\right)/\alpha\\
&+\underset{\bm{x}'\sim \mathcal{D}_\mathcal{N}}{\mathbb{E}}
\left([(f(\bm{x}')-b)^2+2(1+\gamma)f(\bm{x}')]\cdot \mathbb{I}_{f(\bm{x}')\geq  \eta_{\beta}(f)}\right)/\beta -\gamma^2.
\end{aligned}
\end{equation}

When we constraint $\gamma$ in range $\Omega_{\gamma}=[\max\{b-1, -a\},1]$, Prop.\ref{proposition1} and Prop.\ref{proposition2} ensure that the positive and negative loss functions are monotonous. Then we can get:
\begin{equation}
    \mathbb{E}_{\bm{x}\sim\mathcal{D}_\mathcal{P}}[\mathbb{I}_{f(\bm{x})\leq\eta_{\alpha}(f)}\cdot\ell_+(\bm{x})]= \min_s \frac{1}{\alpha} \cdot \mathbb{E}_{\bm{x}\sim\mathcal{D}_\mathcal{P}} [\alpha s + [\ell_+(\bm{x})-s]_+],
\end{equation}
\begin{equation}
    \mathbb{E}_{\bm{x}'\sim\mathcal{D}_\mathcal{N}}[\mathbb{I}_{f(\bm{x}')\geq\eta_{\beta}(f)}\cdot\ell_-(\bm{x}')]= \min_{s'} \frac{1}{\beta} \cdot \mathbb{E}_{\bm{x}'\sim\mathcal{D}_\mathcal{N}} [\beta s' + [\ell_-(\bm{x}')-s']_+].
\end{equation}
Applying the Lem.\ref{lemma:1} to positive and negative loss, we have:
\begin{equation}
    \begin{aligned}
\colblue{\underset{\bm{z}\sim \mathcal{D}_\mathcal{Z}}{\mathbb{E}}}
[F_{tp}(f,a,b,\gamma,\colblue{\eta_{\alpha}(f)},\colblue{\eta_{\beta}(f)},\bm{z})]&\Leftrightarrow \underset{s, s'}{\min}
\underset{\bm{x}\sim \mathcal{D}_\mathcal{P}}{\mathbb{E}}\left(\alpha s + \left[(f(\bm{x})-a)^2- 
2(1+\gamma)f(\bm{x}) - s\right]_+\right)/\alpha -\gamma^2\\
&+\underset{\bm{x}'\sim \mathcal{D}_\mathcal{N}}{\mathbb{E}}
\left(\beta s' +\left[(f(\bm{x}')-b)^2+2(1+\gamma)f(\bm{x}')-s'\right]_+\right)/\beta,
\end{aligned}
\end{equation}
Then, we get:
\begin{equation}
\begin{aligned}
\colblue{\underset{\bm{z}\sim \mathcal{D}_\mathcal{Z}}{\mathbb{E}}}[F_{tp}(f,a,b,\gamma, \colblue{\eta_{\alpha}(f)}, \colblue{\eta_{\beta}(f)}, \bm{z})]
&\Leftrightarrow 
\ \underset{s\in\Omega_{s},s'\in\Omega_{s'}}{\min}
\ \colblue{\underset{\bm{z}\sim \mathcal{D}_\mathcal{Z}}{\mathbb{E}}}[G_{tp}(f,a,b,\gamma,\bm{z},s,s')].
\end{aligned}
\end{equation}
where $\Omega_{\gamma}=[\max\{b-1, -a\},1]$, $\Omega_{s}=[-4,1]$, $\Omega_{s'}=[0,5]$, $p=\mathbb{P}[y=1]$ and

\begin{equation}
    \begin{aligned}
G_{tp}(f,a,b,\gamma,&\bm{z},s,s')=\left(\alpha s + \left[(f(\bm{x})-a)^2- 
2(1+\gamma)f(\bm{x})-s\right]_+\right)y/(\alpha p)\\
& +
\left(\beta s' +\left[(f(\bm{x})-b)^2+2(1+\gamma) f(\bm{x})-s'\right]_+\right)(1-y)/[\beta(1-p)] -\gamma^2.
\end{aligned}
\end{equation}

we have the equivalent optimization for $\mathrm{TPAUC}$:
\begin{equation}
\begin{aligned}
    \underset{\cmin}{\min}\ \underset{\gamma\in[-1,1]}{\max} \ 
\colblue{\underset{\bm{z}\sim \mathcal{D}_\mathcal{Z}}{\mathbb{E}}}[F_{tp}(f,a,b,\gamma,\colblue{\efa}, \colblue{\efb}, \bm{z})]
\\ \Leftrightarrow \underset{\cmin}{\min}\ 
\underset{\cmax }{\max} 
\ \underset{s\in\Omega_{s},s'\in\Omega_{s'}}{\min}
\ \colblue{\underset{\bm{z}\sim \mathcal{D}_\mathcal{Z}}{\mathbb{E}}}[G_{tp}(f,a,b,\gamma,\bm{z},s,s')],
\end{aligned}
\end{equation}
The same result is hold for empirical version $\colbit{\underset{\bm{z}\sim S}{\ehat}}[G_{tp}(f,a,b,\gamma,\bm{z},s,s')]$.
\end{proof}

\section{Proof of Generalization Bound}
First we need the following lemma  to complete the proof in this subsection.
\label{section:proofs_generalization}
\begin{lemma}
\begin{equation}
\begin{aligned}
    \underset{x}{\max} \ f(x) - \underset{x'}{\max} \ g(x')\leq \underset{x,x'=x}{\max}\ f(x)-g(x)\\
    \underset{x}{\min} \ f(x) - \underset{x'}{\min} \ g(x')\leq \underset{x,x'=x}{\max}\ f(x)-g(x).
\end{aligned}
\end{equation}
\begin{proof}
Since the difference of suprema does not exceed the supremum of the
difference, we have:
\begin{equation}
    \underset{x}{\max} \ f(x) - \underset{x'}{\max} \ g(x')\leq \underset{x}{\max} \ \underset{x'}{\min} f(x)-g(x')\leq \underset{x,x'=x}{\max}\ f(x)-g(x).
\end{equation}
For $\underset{x}{\min} \ f(x) - \underset{x'}{\min} \ g(x')\leq \underset{x,x'=x}{\max}\ f(x)-g(x)$, we have:
\begin{equation}
    \begin{aligned}
&\underset{x}{\min} \ f(x) - \underset{x'}{\min} \ g(x')
\leq\underset{x}{\min}\ \underset{x'}{\max}\ f(x) - g(x')\\
&=\underset{x'}{\max}\ \underset{x}{\min}\ f(x) - g(x')
\leq \underset{x,x'=x}{\max}\ f(x)-g(x).
\end{aligned}
\end{equation}
\end{proof}
\label{lemma:4}
\end{lemma}
\begin{lemma} (Talagrand’s lemma \cite{schapire2012foundations})
Let $\bm{\phi}_1,\cdots,\bm{\phi}_m$ be $l$-Lipschitz functions from $\mathbb{R}$ to $\mathbb{R}$ and $\sigma_1,\cdots,\sigma_m$ be Rademacher random variables. Then, for any hypothesis set $\mathcal{H}$ of real-valued functions, the following inequality holds:
\begin{equation}
    \left.\frac{1}{m} \underset{\boldsymbol{\sigma}}{\mathbb{E}}\left[\sup _{h \in \mathcal{H}} \sum_{i=1}^{m} \sigma_{i}\left(\Phi_{i} \circ h\right)\left(x_{i}\right)\right)\right] \leq \frac{l}{m} \underset{\boldsymbol{\sigma}}{\mathbb{E}}\left[\sup _{h \in \mathcal{H}} \sum_{i=1}^{m} \sigma_{i} h\left(x_{i}\right)\right]=l \widehat{\Re}_{S}(\mathcal{H}).
\end{equation}
In particular, if $\bm{\phi}_i=\bm{\phi}$ for all $i\in[m]$, then the following holds:
\begin{equation}
    \widehat{\mathfrak{R}}_{S}(\Phi \circ \mathcal{H}) \leq l \widehat{\mathfrak{R}}_{S}(\mathcal{H}).
\end{equation}
\label{lemma:5}
\end{lemma}

\begin{lemma}
Let $\bm{\sigma}$ be Rademacher random variables. Then, for any hypothesis set $\mathcal{F}$ of real-valued functions, the following inequality holds:
\begin{equation}
\begin{aligned}
    \underset{\bm{\sigma}}{\mathbb{E}}\left[\underset{f\in\mathcal{F}}{\sup} \left|\frac{1}{n_+}\sum_{i=1}^{n_+}\sigma_i f(\bm{x}_i)\right|\right]\leq 2\hat{\Re}_+(\mathcal{F})
\end{aligned}
\end{equation}
\begin{equation}
\begin{aligned}
    \underset{\bm{\sigma}}{\mathbb{E}}\left[\underset{f\in\mathcal{F}}{\sup}\left|\frac{1}{n_-}\sum_{j=1}^{n_-}\sigma_j f(\bm{x}'_j)\right|\right]\leq 2\hat{\Re}_-(\mathcal{F})
\end{aligned}
\end{equation}
\begin{proof}
Assuming that $0 \in \mathcal{F}$, then for any $\bm{\sigma}$ we have:
\begin{equation}
    \underset{\colorg{f\in\mathcal{F}}}{\sup} \frac{1}{n_+}\sum_{i=1}^{n_+}\sigma_i f(\bm{x}_i) \geq \frac{1}{n_+}\sum_{i=1}^{n_+}\sigma_i \cdot 0=0
\end{equation}
Similarly, for any $\bm{\sigma}$ we have:
\begin{equation}
    \underset{\colsec{f\in-\mathcal{F}}}{\sup} \frac{1}{n_+}\sum_{i=1}^{n_+}\sigma_i f(\bm{x}_i) \geq \frac{1}{n_+}\sum_{i=1}^{n_+}\sigma_i \cdot 0=0
\end{equation}
where $-\mathcal{F}=\{-f_i(\cdot)\}_{i=1}^{|\mathcal{F}|}$ and $f_i(\cdot)\in\mathcal{F}$. Then we have the following inequality:
\begin{equation}
\begin{aligned}
    \underset{\bm{\sigma}}{\mathbb{E}}\left[\underset{\colorg{f\in\mathcal{F}}}{\sup}\left|\frac{1}{n_+}\sum_{i=1}^{n_+}\sigma_i f(\bm{x}_i)\right|\right]
    &=\underset{\bm{\sigma}}{\mathbb{E}}\left[\max\left\{\underset{\colorg{f\in\mathcal{F}}}{\sup}\frac{1}{n_+}\sum_{i=1}^{n_+}\sigma_i f(\bm{x}_i), \underset{\colsec{f\in-\mathcal{F}}}{\sup}\frac{1}{n_+}\sum_{i=1}^{n_+}\sigma_i f(\bm{x}_i)\right\}\right]\\
    &\overset{(*)}{\leq} \hat{\Re}_+(\colorg{\mathcal{F}})+\hat{\Re}_+(\colsec{-\mathcal{F}})\\
    &= 2\hat{\Re}_+(\colorg{\mathcal{F}})
\end{aligned}
\end{equation}
$(*)$ is due to the fact that $\max\{a,b\}\leq a+b$ when $a\geq 0$, $b\geq 0$. The same result holds for negative instances.
\end{proof}
\label{lemma:6}
\end{lemma}

\begin{lemma}
Let $\bm{\sigma}$ be Rademacher random variables. Then, for any hypothesis set $\mathcal{F}$ of real-valued functions, the following inequality holds:
\begin{equation}
    \underset{\bm{\sigma}}{\mathbb{E}}\left[\underset{a\in[0,1]}{\sup}\frac{1}{n_+}\sum_{i=1}^{n_+}\sigma_i 
 a^2 \right]= O\left(\frac{1}{\sqrt{n_+}}\right),
\end{equation}
\begin{equation}
    \underset{\bm{\sigma}}{\mathbb{E}}\left[\underset{b\in[0,1]}{\sup}\frac{1}{n_-}\sum_{j=1}^{n_-}\sigma_j 
 b^2 \right]= O\left(\frac{1}{\sqrt{n_-}}\right),
\end{equation}
\begin{proof}
Using the Cauchy inequality, we have:
\begin{equation}
\begin{aligned}
    \underset{\bm{\sigma}}{\mathbb{E}}\left[\underset{a\in[0,1]}{\sup}\frac{1}{n_+}\sum_{i=1}^{n_+}\sigma_i 
 a^2 \right]&\leq \left(\underset{a\in[0,1]}{\sup} |a^2|\right)\cdot \underset{\bm{\sigma}}{\mathbb{E}}\left[\left|\frac{1}{n_+}\sum_{i=1}^{n_+}\sigma_i\right|\right]\\
 &\overset{(r)}{\leq} 1 \cdot \sqrt{\underset{\bm{\sigma}}{\mathbb{E}}\left[\left(\frac{1}{n_+}\sum_{i=1}^{n_+}\sigma_i\right)^2\right]}\\
 &= \sqrt{\frac{1}{n_+^2}\underset{\bm{\sigma}}{\mathbb{E}}\left[\sum_{i=1}^{n_+}\sigma_i^2\right]}\\
 &= \sqrt{\frac{1}{n_+}}=O\left(\frac{1}{\sqrt{n_+}}\right)
\end{aligned}
\end{equation}
$(r)$ is due to the fact $\sqrt{(\cdot)}$ is concave and the Jensen's inequality. The same result holds for negative instances.
\end{proof}
\label{lemma:7}
\end{lemma}
\subsection{OPAUC}
\textbf{Reminder of Theorem \ref{theorem:4}.}
\emph{For any $\delta>0$, with probability at least $1-\delta$ over the draw of an i.i.d. sample set $S$ of size $n$, for all $f\in\mathcal{F}$ we have:}
\begin{equation*}
\begin{aligned}
\underset{(a,b)\in[0,1]^2}{\min}\ 
\underset{\cmax }{\max} \underset{s'\in\Omega_{s'}}{\min}
\ \colblue{\underset{\bm{z}\sim \mathcal{D}_\mathcal{Z}}{\mathbb{E}}}[G_{op}(f,a,b,\gamma,\bm{z},s')] \le& \underset{\cmins}{\min}\underset{\cmax }{\max}  \underset{s'\in\Omega_{s'}}{\min}
\ \colbit{\underset{\bm{z}\sim S}{\ehat}}[G_{op}(f,a,b,\gamma,\bm{z},s')] \\ 
&~~+ O(\hat{\Re}_{+}(\mathcal{F}) + \hat{\Re}_{-}(\mathcal{F})) + O( \np^{-1/2} + \beta^{-1}\nn^{-1/2}).
\end{aligned}
\end{equation*}

\begin{proof}
According to the Lem.\ref{lemma:4}, we have:
\begin{equation}
    \begin{aligned}
&\quad \ \underset{f\in\mathcal{F}}{\sup} \left(\underset{(a,b)\in[0,1]^2}
{\min}\ \underset{\gamma\in\Omega_{\gamma}}{\max} \ \underset{s'\in\Omega_{s'}}{\min}
\underset{\bm{z}\sim \mathcal{D}_\mathcal{Z}}{\mathbb{E}}[G_{op}(f,a,b,\gamma,\bm{z},s')]\right.
\\
& \left. \qquad \qquad \qquad \qquad \qquad \qquad \qquad \qquad -\underset{(a,b)\in[0,1]^2}{\min}\ 
\underset{\gamma\in\Omega_{\gamma}}{\max} \  \underset{s'\in\Omega_{s'}}{\min}
\underset{\bm{z}\sim S}{\hat{\mathbb{E}}}[G_{op}(f,a,b,\gamma,\bm{z},s')]\right)\\
&\leq\underset{f\in\mathcal{F},(a,b)\in[0,1]^2}{\sup} 
\left(\underset{\gamma\in\Omega_{\gamma}}{\max} \underset{s'\in\Omega_{s'}}{\min} \underset{\bm{z}\sim 
\mathcal{D}_\mathcal{Z}}{\mathbb{E}}[G_{op}(f,a,b,\gamma,\bm{z},s')]-\underset{\gamma\in\Omega_{\gamma}}{\max} \underset{s'\in\Omega_{s'}}{\min}
\underset{\bm{z}\sim S}{\hat{\mathbb{E}}}[G_{op}(f,a,b,\gamma,\bm{z},s')]\right)\\
&\leq\underset{f\in\mathcal{F},(a,b)\in[0,1]^2,
\gamma\in\Omega_{\gamma}}{\sup}
\left(\underset{s'\in\Omega_{s'}}{\min}\underset{\bm{z}\sim \mathcal{D}_\mathcal{Z}}{\mathbb{E}}[G_{op}(f,a,b,\gamma,\bm{z},s')]-\underset{s'\in\Omega_{s'}}{\min}
\underset{\bm{z}\sim S}{\hat{\mathbb{E}}}[G_{op}(f,a,b,\gamma,\bm{z},s')]\right)\\
&\leq\underset{f\in\mathcal{F},(a,b)\in[0,1]^2,s'\in\Omega_{s'}
\gamma\in\Omega_{\gamma}}{\sup}
\left(\underset{\bm{z}\sim \mathcal{D}_\mathcal{Z}}{\mathbb{E}}[G_{op}(f,a,b,\gamma,\bm{z},s')]-
\underset{\bm{z}\sim S}{\hat{\mathbb{E}}}[G_{op}(f,a,b,\gamma,\bm{z},s')]\right)\\
&\leq\underset{f\in\mathcal{F},a\in[0,1],
\gamma\in\Omega_{\gamma}}{\sup} \left(\underset{\bm{x}\sim\mathcal{D}_\mathcal{P}}
{\mathbb{E}}P(f,a,\gamma,\bm{x}) - \underset{\bm{x}_i\sim \mathcal{P}}
{\hat{\mathbb{E}}}P(f,a,\gamma,\bm{x}_i)\right)\\
&+\underset{f\in\mathcal{F},b\in[0,1],s'\in\Omega_{s'},
\gamma\in\Omega_{\gamma}}{\sup} \left(\underset{\bm{x}'\sim\mathcal{D}_\mathcal{N}}
{\mathbb{E}}N(f,b,\gamma,\bm{x}',s') - \underset{\bm{x}'_j\sim \mathcal{N}}
{\hat{\mathbb{E}}}N(f,b,\gamma,\bm{x}'_j,s')\right).
\end{aligned}
\end{equation}
where $P(f,a,\gamma,\bm{x})=(f(\bm{x})-a)^2-2(1+\gamma)f(\bm{x})$ and $N(f,b,\gamma,\bm{x}',s')=(\beta s' + \left[(f(\bm{x}')-b)^2+2(1+\gamma) f(\bm{x}')-s'\right]_+)/\beta$. According to the Thm 3.3 in \cite{schapire2012foundations}, with probability at least $1-\delta  (\delta>0)$ we have:

\begin{equation}
\begin{aligned}
&\underset{f\in\mathcal{F},a\in[0,1],
\gamma\in\Omega_{\gamma}}{\sup} \left(\underset{\bm{x}\sim\mathcal{D}_\mathcal{P}}
{\mathbb{E}}P(f,a,\gamma,\bm{x}) - \underset{\bm{x}_i\sim \mathcal{P}}
{\hat{\mathbb{E}}}P(f,a,\gamma,\bm{x}_i)\right)\\
&+\underset{f\in\mathcal{F},b\in[0,1],s'\in\Omega_{s'},
\gamma\in\Omega_{\gamma}}{\sup} \left(\underset{\bm{x}'\sim\mathcal{D}_\mathcal{N}}
{\mathbb{E}}N(f,b,\gamma,\bm{x}',s') - \underset{\bm{x}'_j\sim \mathcal{N}}
{\hat{\mathbb{E}}}N(f,b,\gamma,\bm{x}'_j,s')\right)\\
&\leq 2\underbrace{\underset{\bm{\sigma}}{\mathbb{E}}\left[\underset{f\in\mathcal{F},a\in[0,1],
\gamma\in\Omega_{\gamma}}{\sup}\frac{1}{n_+}\sum_{i=1}^{n_+}\sigma_i 
P(f,a,\gamma, \bm{x}_i)\right]}_{(1)}+12\sqrt{\frac{\log\frac{4}{\delta}}{2n_+}}\\
&+2\underbrace{\underset{\bm{\sigma}}{\mathbb{E}}
\left[\underset{f\in\mathcal{F},b\in[0,1],s'\in\Omega_{s'},
\gamma\in\Omega_{\gamma}}{\sup}\frac{1}{n_-}\sum_{j=1}^{n_-}\sigma_j 
N(f,b,\gamma, \bm{x}'_j,s')\right]}_{(2)}
+\frac{15}{\beta} \sqrt{\frac{\log\frac{4}{\delta}}
{2n_-}}.
\end{aligned}
\end{equation}

For term $(1)$, we have:
\begin{equation}
\begin{aligned}
&\underset{\bm{\sigma}}{\mathbb{E}}\left[\underset{f\in\mathcal{F},a\in[0,1],
\gamma\in\Omega_{\gamma}}{\sup}\frac{1}{n_+}\sum_{i=1}^{n_+}\sigma_i 
P(f,a,\gamma, \bm{x}_i)\right]\\
&=\underset{\bm{\sigma}}{\mathbb{E}}\left[\underset{f\in\mathcal{F},a\in[0,1],
\gamma\in\Omega_{\gamma}}{\sup}\frac{1}{n_+}\sum_{i=1}^{n_+}\sigma_i 
\left(f^2(\bm{x}_i)-2(1+\gamma+a)f(\bm{x}_i)+a^2\right)\right]\\
&\overset{(s)}{\leq} \underbrace{\underset{\bm{\sigma}}{\mathbb{E}}\left[\underset{f\in\mathcal{F}}{\sup}\frac{1}{n_+}\sum_{i=1}^{n_+}\sigma_i 
f^2(\bm{x}_i)\right]}_{(a)} + \underbrace{\underset{\bm{\sigma}}{\mathbb{E}}\left[\underset{f\in\mathcal{F}}{\sup}\frac{1}{n_+}\sum_{i=1}^{n_+}\sigma_i 
\left(-2f(\bm{x}_i)\right)\right]}_{(b)} + \underbrace{\underset{\bm{\sigma}}{\mathbb{E}}\left[\underset{a\in[0,1]}{\sup}\frac{1}{n_+}\sum_{i=1}^{n_+}\sigma_i 
\left( a^2 \right)\right]}_{(c)}\\
&+ \underbrace{\underset{\bm{\sigma}}{\mathbb{E}}\left[\underset{f\in\mathcal{F},
\gamma\in\Omega_{\gamma}}{\sup}\frac{1}{n_+}\sum_{i=1}^{n_+}\sigma_i 
\left(-2\gamma f(\bm{x}_i)\right)\right]}_{(d)} + \underbrace{\underset{\bm{\sigma}}{\mathbb{E}}\left[\underset{f\in\mathcal{F},a\in [0, 1]}{\sup}\frac{1}{n_+}\sum_{i=1}^{n_+}\sigma_i 
\left(-2af(\bm{x}_i)\right)\right]}_{(e)}. \\
\end{aligned}
\end{equation}
$(s)$ is due to the fact that $\sup_{a,b} a+b\leq \sup a+ \sup b$. Assuming that $f(\bm{x})\in[0,1]$, $0\in\mathcal{F}$. For term (a), according to the Lem.\ref{lemma:5} and the fact that $x^2$ is $2$-Lipschitz continuous within $[0, 1]$, we get:
\begin{equation}
    \underset{\bm{\sigma}}{\mathbb{E}}\left[\underset{f\in\mathcal{F}}{\sup}\frac{1}{n_+}\sum_{i=1}^{n_+}\sigma_i 
     f^2(\bm{x}_i)\right]\leq 2\underset{\bm{\sigma}}{\mathbb{E}}\left[\underset{f\in\mathcal{F}}{\sup}\frac{1}{n_+}\sum_{i=1}^{n_+}\sigma_i 
    f(\bm{x}_i)\right]=2\hat{\Re}_+(\mathcal{F}).
\end{equation}
Using the fact that $\sigma_i$ and $-\sigma_i$ are distributed in the same way, we can write the term $(b)$ as:
\begin{equation}
    \underset{\bm{\sigma}}{\mathbb{E}}\left[\underset{f\in\mathcal{F}}{\sup}\frac{1}{n_+}\sum_{i=1}^{n_+}\sigma_i 
\left(-2f(\bm{x}_i)\right)\right]=2\underset{\bm{\sigma}}{\mathbb{E}}\left[\underset{f\in\mathcal{F}}{\sup}\frac{1}{n_+}\sum_{i=1}^{n_+}\sigma_i 
    f(\bm{x}_i)\right]=2\hat{\Re}_+(\mathcal{F}).
\end{equation}
For term $(c)$, according to the Lem.\ref{lemma:7}, we have:
\begin{equation}
    \underset{\bm{\sigma}}{\mathbb{E}}\left[\underset{a\in[0,1]}{\sup}\frac{1}{n_+}\sum_{i=1}^{n_+}\sigma_i 
\left( a^2 \right)\right]= O\left(\frac{1}{\sqrt{n_+}}\right).
\end{equation}
For term $(d)$, we have:
\begin{equation}
\begin{aligned}
    \underset{\bm{\sigma}}{\mathbb{E}}\left[\underset{f\in\mathcal{F},
    \gamma\in\Omega_{\gamma}}{\sup}\frac{1}{n_+}\sum_{i=1}^{n_+}\sigma_i 
    \left(-2\gamma f(\bm{x}_i)\right)\right]&\leq \underset{\bm{\sigma}}{\mathbb{E}}\left[\underset{f\in\mathcal{F},
    \gamma\in\Omega_{\gamma}}{\sup}\left|\frac{1}{n_+}\sum_{i=1}^{n_+}\sigma_i 
    \left(-2\gamma f(\bm{x}_i)\right)\right|\right]\\
    &\leq \underset{\bm{\sigma}}{\mathbb{E}}\left[\underset{f\in\mathcal{F},\gamma\in\Omega_{\gamma}}{\sup}|-2\gamma|\cdot \left|\frac{1}{n_+}\sum_{i=1}^{n_+}\sigma_i f(\bm{x}_i)\right|\right]\\
    &\leq 2 \underset{\bm{\sigma}}{\mathbb{E}}\left[\underset{f\in\mathcal{F}}{\sup} \left|\frac{1}{n_+}\sum_{i=1}^{n_+}\sigma_i f(\bm{x}_i)\right|\right]\\
    &\overset{(*)}{\leq} 4\hat{\Re}_+(\mathcal{F}),
\end{aligned}
\end{equation}
where $(*)$ follows from the Lem.\ref{lemma:6}. Similarly, for term $(e)$, we have:
\begin{equation}
\begin{aligned}
    \underset{\bm{\sigma}}{\mathbb{E}}\left[\underset{f\in\mathcal{F},
    a\in[0,1]}{\sup}\frac{1}{n_+}\sum_{i=1}^{n_+}\sigma_i 
    \left(-2a f(\bm{x}_i)\right)\right]\leq 4\hat{\Re}_+(\mathcal{F})
\end{aligned}
\end{equation}
Combining terms (a), (b), (c), (d), (e), then we get:
\begin{equation}
    \underset{\bm{\sigma}}{\mathbb{E}}\left[\underset{f\in\mathcal{F},a\in[0,1],
\gamma\in\Omega_{\gamma}}{\sup}\frac{1}{n_+}\sum_{i=1}^{n_+}\sigma_i 
P(f,a,\gamma, \bm{x}_i)\right]\leq 12\hat{\Re}_+(\mathcal{F})+O\left(\frac{1}{\sqrt{n_+}}\right)
\end{equation}
For term $(2)$, we have:
\begin{equation}
\begin{aligned}
&\underset{\bm{\sigma}}{\mathbb{E}}\left[\underset{f\in\mathcal{F},b\in[0,1],s'\in\Omega_{s'},
\gamma\in\Omega_{\gamma}}{\sup}\frac{1}{n_-}\sum_{j=1}^{n_-}\sigma_j 
N(f,a,\gamma, \bm{x}'_j, s')\right]\\
&=\underset{\bm{\sigma}}{\mathbb{E}}\left[\underset{f\in\mathcal{F},b\in[0,1],s'\in\Omega_{s'},
\gamma\in\Omega_{\gamma}}{\sup}\frac{1}{n_-}\sum_{j=1}^{n_-}\sigma_j \left(
\left[f^2(\bm{x}'_j)+2(1+\gamma-b)f(\bm{x}'_j)+b^2-s'\right]_+ +\beta s'\right)\right]\\
&\overset{(o')}{\leq} \underset{\bm{\sigma}}{\mathbb{E}}\left[\underset{f\in\mathcal{F},b\in[0,1],s'\in\Omega_{s'},
\gamma\in\Omega_{\gamma}}{\sup}\frac{1}{n_-}\sum_{j=1}^{n_-}\sigma_j 
\left(\left[f^2(\bm{x}'_j)+2(1+\gamma-b)f(\bm{x}'_j)+b^2-s'\right]_+\right)\right]+O\left(\frac{1}{\sqrt{n_-}}\right)\\
&\overset{(l')}{\leq} \underset{\bm{\sigma}}{\mathbb{E}}\left[\underset{f\in\mathcal{F},b\in[0,1],s'\in\Omega_{s'},
\gamma\in\Omega_{\gamma}}{\sup}\frac{1}{n_-}\sum_{j=1}^{n_-}\sigma_j 
\left(f^2(\bm{x}'_j)+2(1+\gamma-b)f(\bm{x}'_j)+b^2-s'\right)\right]+O\left(\frac{1}{\sqrt{n_-}}\right)\\
&\overset{(s)}{\leq} \underbrace{\underset{\bm{\sigma}}{\mathbb{E}}\left[\underset{f\in\mathcal{F}}{\sup}\frac{1}{n_-}\sum_{j=1}^{n_-}\sigma_j 
 f^2(\bm{x}'_j)\right]}_{(a')} + \underbrace{\underset{\bm{\sigma}}{\mathbb{E}}\left[\underset{f\in\mathcal{F}}{\sup}\frac{1}{n_-}\sum_{j=1}^{n_-}\sigma_j 
2f(\bm{x}'_j)\right]}_{(b')}+\underbrace{\underset{\bm{\sigma}}{\mathbb{E}}\left[\underset{s'\in\Omega_{s'},b\in[0,1]}{\sup}\frac{1}{n_-}\sum_{j=1}^{n_-}\sigma_j (b^2-s')\right]}_{(c')} \\
&+\underbrace{\underset{\bm{\sigma}}{\mathbb{E}}\left[\underset{f\in\mathcal{F},\gamma\in\Omega_{\gamma}}{\sup}\frac{1}{n_-}\sum_{j=1}^{n_-}\sigma_j
2\gamma f(\bm{x}'_j)\right]}_{(d')}+\underbrace{\underset{\bm{\sigma}}{\mathbb{E}}\left[\underset{f\in\mathcal{F},b\in[0,1]}{\sup}\frac{1}{n_-}\sum_{j=1}^{n_-}\sigma_j
\left(-2b f(\bm{x}'_j)\right)\right]}_{(e')} +O\left(\frac{1}{\sqrt{n_-}}\right).
\end{aligned}
\end{equation}
$(o')$ follows from the Lem.\ref{lemma:7} and the fact that $\sup_{a,b}\leq \sup a +\sup b$.
$(l')$ follows from the Lem.\ref{lemma:5} and the fact that $[\cdot]_+$ is $1$-Lipschitz continuous. For terms $(a')$, $(b')$, $(c')$, $(d')$, $(e')$, we have the similar results as terms $(a)$, $(b)$, $(c)$, $(d)$, $(e)$. So we can get:
\begin{equation}
    \underset{\bm{\sigma}}{\mathbb{E}}\left[\underset{f\in\mathcal{F},b\in[0,1],s'\in\Omega_{s'},
\gamma\in\Omega_{\gamma}}{\sup}\frac{1}{n_-}\sum_{j=1}^{n_-}\sigma_j 
N(f,b,\gamma, \bm{x}'_j,s')\right]\leq 12\hat{\Re}_-(\mathcal{F})+O\left(\frac{1}{\sqrt{n_-}}\right)
\end{equation}
and 

\begin{equation}
    \begin{aligned}
&\underset{f\in\mathcal{F},a\in[0,1],
\gamma\in\Omega_{\gamma}}{\sup} \left(\underset{\bm{x}\sim\mathcal{D}_\mathcal{P}}
{\mathbb{E}}P(f,a,\gamma,\bm{x}) - \underset{\bm{x}_i\sim \mathcal{P}}
{\hat{\mathbb{E}}}P(f,a,\gamma,\bm{x}_i)\right)\\
&+\underset{f\in\mathcal{F},b\in[0,1],s'\in\Omega_{s'}
\gamma\in\Omega_{\gamma}}{\sup} \left(\underset{\bm{x}'\sim\mathcal{D}_\mathcal{N}}
{\mathbb{E}}N(f,b,\gamma,\bm{x}',s') - \underset{\bm{x}'_j\sim \mathcal{N}}
{\hat{\mathbb{E}}}N(f,b,\gamma,\bm{x}'_j,s')\right)\\
&\leq 2\left(12\hat{\Re}_+(\mathcal{F})+O\left(\frac{1}{\sqrt{n_+}}\right)+12\hat{\Re}_-(\mathcal{F})+O\left(\frac{1}{\sqrt{n_-}}\right)\right)+12\sqrt{\frac{\log\frac{4}{\delta}}{2n_+}}+\frac{15}{\beta}\sqrt{\frac{\log\frac{4}{\delta}}
{2n_-}}\\
&= O(\hat{\Re}_{+}(\mathcal{F}) + \hat{\Re}_{-}(\mathcal{F})) + O( \np^{-1/2} + \beta^{-1}\nn^{-1/2})
\end{aligned}
\end{equation}

For any $\delta > 0$, with probability at least $1-\delta$ over the draw of an i.i.d. sample $S$ of positive instances size $n_+$ (negative $n_-$ resp.), each of the following holds for all $f\in\mathcal{F}$:

\begin{equation*}
    \begin{aligned}
&\underset{\cmins}{\min}\ 
\underset{\cmax }{\max}  \underset{s'\in\Omega_{s'}}{\min}
\ \colblue{\underset{\bm{z}\sim \mathcal{D}_\mathcal{Z}}{\mathbb{E}}}[G_{op}(f,a,b,\gamma,\bm{z},s')] \\
&\leq\underset{\cmins}{\min}\underset{\cmax }{\max}  \underset{s'\in\Omega_{s'}}{\min}
\ \colbit{\underset{\bm{z}\sim S}{\ehat}}[G_{op}(f,a,b,\gamma,\bm{z},s')] 
+ O(\hat{\Re}_{+}(\mathcal{F}) + \beta^{-1}\hat{\Re}_{-}(\mathcal{F})) 
\\&+ O( \np^{-1/2} + \beta^{-1}\nn^{-1/2}).
\end{aligned}
\end{equation*}

\end{proof}

\subsection{TPAUC}
\begin{theorem}
For any $\delta>0$, with probability at least $1-\delta$ over the draw of an i.i.d. sample set $S$ of size $n$, for all $f\in\mathcal{F}$ we have:
\begin{equation*}
    \begin{aligned}
&\underset{\cmins}{\min}\ 
\underset{\cmax }{\max}  \underset{s\in\Omega_{s},s'\in\Omega_{s'}}{\min}
\ \colblue{\underset{\bm{z}\sim \mathcal{D}_\mathcal{Z}}{\mathbb{E}}}[G_{tp}(f,a,b,\gamma,\bm{z},s,s')] \le \\
& \underset{\cmins}{\min}\underset{\cmax }{\max}  \underset{s\in\Omega_{s},s'\in\Omega_{s'}}{\min}
\ \colbit{\underset{\bm{z}\sim S}{\ehat}}[G_{tp}(f,a,b,\gamma,\bm{z},s,s')] \\ 
&~~+ O(\hat{\Re}_{+}(\mathcal{F}) + \hat{\Re}_{-}(\mathcal{F})) + O( \alpha^{-1}\np^{-1/2} + \beta^{-1} \nn^{-1/2}).
\end{aligned}
\end{equation*}
\begin{proof}
According to Lem.\ref{lemma:4}, we have:
\begin{equation}
    \begin{aligned}
&\quad \ \underset{f\in\mathcal{F}}{\sup} \left(\underset{(a,b)\in[0,1]^2}
{\min}\ \underset{\gamma\in\Omega_{\gamma}}{\max} \ \underset{s\in\Omega_{s},s'\in\Omega_{s'}}{\min}
\underset{\bm{z}\sim \mathcal{D}_\mathcal{Z}}{\mathbb{E}}[G_{tp}(f,a,b,\gamma,\bm{z},s,s')]\right.
\\
&\spaces \left.-\underset{(a,b)\in[0,1]^2}{\min}\ 
\underset{\gamma\in\Omega_{\gamma}}{\max} \ \underset{s\in\Omega_{s},s'\in\Omega_{s'}}{\min}
\underset{\bm{z}\sim S}{\hat{\mathbb{E}}}[G_{tp}(f,a,b,\gamma,\bm{z},s,s')]\right)\\
&\leq\underset{f\in\mathcal{F},(a,b)\in[0,1]^2}{\sup} 
\left(\underset{\gamma\in\Omega_{\gamma}}{\max}\underset{s\in\Omega_{s},s'\in\Omega_{s'}}{\min}\underset{\bm{z}\sim \mathcal{D}_\mathcal{Z}}{\mathbb{E}}[G_{tp}(f,a,b,\gamma,\bm{z},s,s')]\right.\\
&\spaces \qquad \left.-\underset{\gamma\in\Omega_{\gamma}}{\max}\underset{s\in\Omega_{s},s'\in\Omega_{s'}}{\min}\underset{\bm{z}\sim S}{\hat{\mathbb{E}}}[G_{tp}(f,a,b,\gamma,\bm{z},s,s')]\right)\\
&\leq\underset{f\in\mathcal{F},(a,b)\in[0,1]^2,
\gamma\in\Omega_{\gamma}}{\sup} 
\left(\underset{s\in\Omega_{s},s'\in\Omega_{s'}}{\min}\underset{\bm{z}\sim \mathcal{D}_\mathcal{Z}}{\mathbb{E}}[G_{tp}(f,a,b,\gamma,\bm{z},s,s')]\right.\\
&\spaces \qquad \qquad  \left.-\underset{s\in\Omega_{s},s'\in\Omega_{s'}}{\min} \underset{\bm{z}\sim S}{\hat{\mathbb{E}}}[G_{tp}(f,a,b,\gamma,\bm{z},s,s')]\right)\\
&\leq\underset{f\in\mathcal{F},(a,b)\in[0,1]^2,s\in\Omega_{s},s'\in\Omega_{s'}
\gamma\in\Omega_{\gamma}}{\sup} 
\left(\underset{\bm{z}\sim \mathcal{D}_\mathcal{Z}}{\mathbb{E}}[G_{tp}(f,a,b,\gamma,\bm{z},s,s')]\right.\\
& \spaces \qquad \qquad \left.-\underset{\bm{z}\sim S}{\hat{\mathbb{E}}}[G_{tp}(f,a,b,\gamma,\bm{z},s,s')]\right)\\
&\leq\underset{f\in\mathcal{F},a\in[0,1], s\in\Omega_{s}
\gamma\in\Omega_{\gamma}}{\sup} \left(\underset{\bm{x}\sim\mathcal{D}_\mathcal{P}}
{\mathbb{E}}P(f,a,\gamma,\bm{x},s) - \underset{\bm{x}_i\sim \mathcal{P}}
{\hat{\mathbb{E}}}P(f,a,\gamma,\bm{x}_i,s)\right)\\
&+\underset{f\in\mathcal{F},b\in[0,1],s'\in\Omega_{s'},
\gamma\in\Omega_{\gamma}}{\sup}  \left(\underset{\bm{x}'\sim\mathcal{D}_\mathcal{N}}
{\mathbb{E}}N(f,b,\gamma,\bm{x}',s') - \underset{\bm{x}'_j\sim \mathcal{N}}
{\hat{\mathbb{E}}}N(f,b,\gamma,\bm{x}'_j,s')\right).
\end{aligned}
\end{equation}
where $P(f,a,\gamma,\bm{x},s)=(\alpha s + \left[(f(\bm{x})-a)^2- 
2(1+\gamma) f(\bm{x})-s\right]_+)/\alpha$ and $N(f,b,\gamma,\bm{x}',s')=(\beta s' + \left[(f(\bm{x}')-b)^2+2(1+\gamma) f(\bm{x}')-s'\right]_+)/\beta$. According to the Thm 3.3 in \cite{schapire2012foundations}, with probability at least $1-\delta (\delta>0)$ we have:

\begin{equation}
\begin{aligned}
&\underset{f\in\mathcal{F},a\in[0,1],s\in\Omega_{s}
\gamma\in\Omega_{\gamma}}{\sup} \left(\underset{\bm{x}\sim\mathcal{D}_\mathcal{P}}
{\mathbb{E}}P(f,a,\gamma,\bm{x},s) - \underset{\bm{x}_i\sim \mathcal{P}}
{\hat{\mathbb{E}}}P(f,a,\gamma,\bm{x}_i,s)\right)\\
&+\underset{f\in\mathcal{F},b\in[0,1],s'\in\Omega_{s'},
\gamma\in\Omega_{\gamma}}{\sup} \left(\underset{\bm{x}'\sim\mathcal{D}_\mathcal{N}}
{\mathbb{E}}N(f,b,\gamma,\bm{x}',s') - \underset{\bm{x}'_j\sim \mathcal{N}}
{\hat{\mathbb{E}}}N(f,b,\gamma,\bm{x}'_j,s')\right)\\
&\leq 2\underbrace{\underset{\bm{\sigma}}{\mathbb{E}}\left[\underset{f\in\mathcal{F},a\in[0,1],
s\in\Omega_{s},\gamma\in\Omega_{\gamma}}{\sup}
\frac{1}{n_+}\sum_{i=1}^{n_+}\sigma_i 
P(f,a,\gamma, \bm{x}_i,s)\right]}_{(3)}+\frac{12}{\alpha}\sqrt{\frac{\log\frac{4}{\delta}}{2n_+}}\\
&+2\underbrace{\underset{\bm{\sigma}}{\mathbb{E}}
\left[\underset{f\in\mathcal{F},b\in[0,1],s'\in\Omega_{s'},
\gamma\in\Omega_{\gamma}}{\sup}\frac{1}{n_-}\sum_{j=1}^{n_-}\sigma_j
N(f,b,\gamma, \bm{x}'_j,s')\right]}_{(4)}+
\frac{15}{\beta}\sqrt{\frac{\log\frac{4}{\delta}}{2n_-}}.
\end{aligned}
\end{equation}

For term $(3)$, we have:
\begin{equation}
\begin{aligned}
&\underset{\bm{\sigma}}{\mathbb{E}}\left[\underset{f\in\mathcal{F},a\in[0,1],s\in\Omega_{s},
\gamma\in\Omega_{\gamma}}{\sup}\frac{1}{n_+}\sum_{i=1}^{n_+}\sigma_i 
P(f,a,\gamma, \bm{x}_i, s)\right]\\
&=\underset{\bm{\sigma}}{\mathbb{E}}\left[\underset{f\in\mathcal{F},a\in[0,1],s\in\Omega_{s},
\gamma\in\Omega_{\gamma}}{\sup}\frac{1}{n_+}\sum_{i=1}^{n_+}\sigma_i 
\left(\left[f^2(\bm{x}_i)-2(1+\gamma+a)f(\bm{x}_i)+a^2-s\right]_++\alpha s\right)\right]\\
&\overset{(o^*)}{\leq} \underset{\bm{\sigma}}{\mathbb{E}}\left[\underset{f\in\mathcal{F},a\in[0,1],s\in\Omega_{s},
\gamma\in\Omega_{\gamma}}{\sup}\frac{1}{n_+}\sum_{i=1}^{n_+}\sigma_i 
\left(\left[f^2(\bm{x}_i)-2(1+\gamma+a)f(\bm{x}_i)+a^2- s\right]_+\right)\right]+O\left(\frac{1}{\sqrt{n_+}}\right)\\
&\overset{(l^*)}{\leq} \underset{\bm{\sigma}}{\mathbb{E}}\left[\underset{f\in\mathcal{F},a\in[0,1],s\in\Omega_{s},
\gamma\in\Omega_{\gamma}}{\sup}\frac{1}{n_+}\sum_{i=1}^{n_+}\sigma_i 
\left(f^2(\bm{x}_i)-2(1+\gamma+a)f(\bm{x}_i)+a^2- s\right)\right]+O\left(\frac{1}{\sqrt{n_+}}\right)\\
&\overset{(s)}{\leq} \underbrace{\underset{\bm{\sigma}}{\mathbb{E}}\left[\underset{f\in\mathcal{F}}{\sup}\frac{1}{n_+}\sum_{i=1}^{n_+}\sigma_i 
 f^2(\bm{x}_i)\right]}_{(a^*)} + \underbrace{\underset{\bm{\sigma}}{\mathbb{E}}\left[\underset{f\in\mathcal{F}}{\sup}\frac{1}{n_+}\sum_{i=1}^{n_+}\sigma_i
(-2f(\bm{x}_i))\right]}_{(b^*)}+\underbrace{\underset{\bm{\sigma}}{\mathbb{E}}\left[\underset{a\in[0,1],s\in\Omega_{s}}{\sup}\frac{1}{n_+}\sum_{i=1}^{n_+}\sigma_i (a^2- s)\right]}_{(c^*)} \\
&+\underbrace{\underset{\bm{\sigma}}{\mathbb{E}}\left[\underset{f\in\mathcal{F},\gamma\in\Omega_{\gamma}}{\sup}\frac{1}{n_+}\sum_{i=1}^{n_+}\sigma_i
(-2\gamma f(\bm{x}_i))\right]}_{(d^*)}+\underbrace{\underset{\bm{\sigma}}{\mathbb{E}}\left[\underset{f\in\mathcal{F},a\in[0,1]}{\sup}\frac{1}{n_+}\sum_{i=1}^{n_+}\sigma_i
\left(-2a f(\bm{x}_i)\right)\right]}_{(e^*)} +O\left(\frac{1}{\sqrt{n_+}}\right).
\end{aligned}
\end{equation}
$(o^*)$ is similar to $(o')$. $(l^*)$ follows from the Lem.\ref{lemma:5} and the fact that $[\cdot]_+$ is $1$-Lipschitz continuous. For terms $(a^*)$, $(b^*)$, $(c^*)$, $(d^*)$, $(e^*)$, we have the similar results as terms $(a')$, $(b')$, $(c')$, $(d')$, $(e')$. So we can get:
\begin{equation}
    \underset{\bm{\sigma}}{\mathbb{E}}\left[\underset{f\in\mathcal{F},a\in[0,1],s\in\Omega_{s},
\gamma\in\Omega_{\gamma}}{\sup}\frac{1}{n_+}\sum_{i=1}^{n_+}\sigma_i
P(f,a,\gamma, \bm{x}_i,s)\right]\leq 12\hat{\Re}_+(\mathcal{F})+O\left(\frac{1}{\sqrt{n_+}}\right)
\end{equation}
and

\begin{equation}
\begin{aligned}
&\underset{f\in\mathcal{F},a\in[0,1],s\in\Omega_{s}
\gamma\in\Omega_{\gamma}}{\sup} \left(\underset{\bm{x}\sim\mathcal{D}_\mathcal{P}}
{\mathbb{E}}P(f,a,\gamma,\bm{x},s) - \underset{\bm{x}_i\sim \mathcal{P}}
{\hat{\mathbb{E}}}P(f,a,\gamma,\bm{x}_i,s)\right)\\
&+\underset{f\in\mathcal{F},b\in[0,1],s'\in\Omega_{s'},
\gamma\in\Omega_{\gamma}}{\sup} \left(\underset{\bm{x}'\sim\mathcal{D}_\mathcal{N}}
{\mathbb{E}}N(f,b,\gamma,\bm{x}',s') - \underset{\bm{x}'_j\sim \mathcal{N}}
{\hat{\mathbb{E}}}N(f,b,\gamma,\bm{x}'_j,s')\right)\\
&\leq 2\left(12\hat{\Re}_+(\mathcal{F})+O\left(\frac{1}{\sqrt{n_+}}\right)+12\hat{\Re}_-(\mathcal{F})+O\left(\frac{1}{\sqrt{n_-}}\right)\right)+\frac{12}{\alpha}\sqrt{\frac{\log\frac{4}{\delta}}{2n_+}}+
\frac{15}{\beta}\sqrt{\frac{\log\frac{4}{\delta}}{2n_-}}\\
&= O(\hat{\Re}_{+}(\mathcal{F}) + \hat{\Re}_{-}(\mathcal{F})) + O( \alpha^{-1}\np^{-1/2} + \beta^{-1}\nn^{-1/2}).
\end{aligned}
\end{equation}

For term $(4)$, the same result holds as term $(2)$. For any $\delta > 0$, with probability at least $1-\delta$ over the draw of an i.i.d. sample $S$ of positive instances size $n_+$ (negative $n_-$ resp.), each of the following holds for all $f\in\mathcal{F}$:

\begin{equation*}
    \begin{aligned}
&\underset{\cmins}{\min}\ 
\underset{\cmax }{\max}  \underset{s\in\Omega_{s},s'\in\Omega_{s'}}{\min}
\ \colblue{\underset{\bm{z}\sim \mathcal{D}_\mathcal{Z}}{\mathbb{E}}}[G_{tp}(f,a,b,\gamma,\bm{z},s,s')] \\
&\leq\underset{\cmins}{\min}\underset{\cmax }{\max}  \underset{s\in\Omega_{s},s'\in\Omega_{s'}}{\min}
\ \colbit{\underset{\bm{z}\sim S}{\ehat}}[G_{tp}(f,a,b,\gamma,\bm{z},s,s')] 
+ O(\hat{\Re}_{+}(\mathcal{F}) + \hat{\Re}_{-}(\mathcal{F})) \\
&+ O( \alpha^{-1}\np^{-1/2} + \beta^{-1}\nn^{-1/2}).
\end{aligned}
\end{equation*}

\end{proof}
\end{theorem}
\end{appendices}

\end{document}